\documentclass{article}

\PassOptionsToPackage{numbers, compress}{natbib}
\usepackage[final]{neurips_2023}

\usepackage[utf8]{inputenc} 
\usepackage[T1]{fontenc}    
\usepackage{hyperref}       
\hypersetup{colorlinks=true}
\usepackage{url}            
\usepackage{booktabs}       
\usepackage{amsfonts}       
\usepackage{nicefrac}       
\usepackage{microtype}      
\usepackage{xcolor}         
\usepackage{amsmath}
\usepackage{wrapfig}
\usepackage{amsthm}
\usepackage{thmtools}
\usepackage{thm-restate}

\newtheorem{remark}{Remark}
\usepackage{graphicx}

\usepackage{algorithm,algorithmic}
\usepackage{multirow}
\usepackage{multicol}
\usepackage{arydshln}
\usepackage{subcaption}

\usepackage{adjustbox}

\DeclareMathOperator*{\argmin}{arg\,min}
\usepackage{color, colortbl}
\definecolor{Gray}{gray}{0.9}
\usepackage[nameinlink]{cleveref}

\usepackage[title,titletoc]{appendix}

\newcommand{\gc}{\cellcolor[gray]{0.85}}
\usepackage{makecell}

\usepackage{enumitem}

\makeatletter
\def\adl@drawiv#1#2#3{%
        \hskip.5\tabcolsep
        \xleaders#3{#2.5\@tempdimb #1{1}#2.5\@tempdimb}%
                #2\z@ plus1fil minus1fil\relax
        \hskip.5\tabcolsep}
\newcommand{\cdashlinelr}[1]{%
  \noalign{\vskip\aboverulesep
           \global\let\@dashdrawstore\adl@draw
           \global\let\adl@draw\adl@drawiv}
  \cdashline{#1}
  \noalign{\global\let\adl@draw\@dashdrawstore
           \vskip\belowrulesep}}
\makeatother

\title{Frequency Domain-based Dataset Distillation}

\author{%
  Donghyeok ~Shin\thanks{~Equal contribution} \\
  KAIST \\
  \texttt{tlsehdgur0@kaist.ac.kr} \\
  \And
  Seungjae ~Shin\footnotemark[1] \\
  KAIST \\
  \texttt{tmdwo0910@kaist.ac.kr} \\
  \And
  Il-Chul Moon \\
  KAIST, Summary.AI \\
  \texttt{icmoon@kaist.ac.kr} \\
}

\begin{document}
\maketitle

\begin{abstract}
This paper presents FreD, a novel parameterization method for dataset distillation, which utilizes the frequency domain to distill a small-sized synthetic dataset from a large-sized original dataset. Unlike conventional approaches that focus on the spatial domain, FreD employs frequency-based transforms to optimize the frequency representations of each data instance. By leveraging the concentration of spatial domain information on specific frequency components, FreD intelligently selects a subset of frequency dimensions for optimization, leading to a significant reduction in the required budget for synthesizing an instance. Through the selection of frequency dimensions based on the explained variance, FreD demonstrates both theoretical and empirical evidence of its ability to operate efficiently within a limited budget, while better preserving the information of the original dataset compared to conventional parameterization methods. Furthermore, based on the orthogonal compatibility of FreD with existing methods, we confirm that FreD consistently improves the performances of existing distillation methods over the evaluation scenarios with different benchmark datasets. We release the code at \url{https://github.com/sdh0818/FreD}.
\end{abstract}

\section{Introduction} 
The era of big data presents challenges in data processing, analysis, and storage; and researchers have studied the concept of \textit{dataset distillation} \citep{wang2018dataset,zhao2020dataset,nguyen2021dataset} to resolve these challenges. Specifically, the objective of dataset distillation is to synthesize a dataset with a smaller cardinality that can preserve the performance of original large-sized datasets in machine learning tasks. By distilling the key features from the original dataset into a condensed dataset, less computational resources and storage space are required while maintaining the performances from the original dataset. Dataset distillation optimizes a small-sized variable to represent the input, not the model parameters. This optimization leads the variable to store a synthetic dataset, and the variable is defined in the memory space with the constraint of limited capacity. Since dataset distillation involves the optimization of \textit{data variables} with limited capacity, distillation parameter designs, which we refer to as parameterization, could significantly improve corresponding optimization while minimizing memory usage.

Some of the existing distillation methods, i.e. 2D image dataset distillations \citep{zhao2021dataset,nguyen2021dataset,cazenavette2022dataset,zhao2023dataset,zhou2022dataset}, naively optimize data variables embedded on the input space without any transformation or encoding. We will refer to distillation on the provided input space as spatial domain distillation, as an opposite concept of frequency domain distillation that is the focus of this paper. The main drawback of spatial domain distillation would be the difficulty in specifying the importance of each pixel dimension on the spatial domain, so it is necessary to utilize the same budget as the original dimension for representing a single instance. From the perspective of a data variable, which needs to capture the key information of the original dataset with a limited budget, the variable modeling with whole dimensions becomes a significant bottleneck that limits the number of distilled data instances. Various spatial domain-based parameterization methods \cite{kim2022dataset,liu2022dataset} have been proposed to overcome this problem, but they are either highly vulnerable to instance-specific information loss \cite{kim2022dataset} or require additional training with an auxiliary network \cite{liu2022dataset}.

This paper argues that the spatial domain distillation has limitations in terms of 1) memory efficiency and 2) representation constraints of the entire dataset. Accordingly, we propose a novel \underline{Fre}quency domain-based dataset \underline{D}istillation method, coined FreD, which contributes to maintaining the task performances of the original dataset based on the limited budget. FreD employs a frequency-based transform to learn the data variable in the transformed frequency domain. Particularly, our proposed method selects and utilizes a subset of frequency dimensions that are crucial for the formation of an instance and the corresponding dataset. By doing so, we are able to achieve a better condensed representation of the original dataset with even fewer dimensions, which corresponds to significant efficiency in memory. Throughout the empirical evaluations with different benchmarks, FreD shows consistent improvements over the existing methods regardless of the distillation objective utilized.
\vspace{-0.1in}
\section{Preliminary}
\subsection{Basic Notations}
This paper primarily focuses on the dataset distillation for classification tasks, which is a widely studied scenario in the dataset distillation community \citep{wang2018dataset,zhao2020dataset}. Given $C$ classes, let $\mathcal{X} \in \mathbb{R}^{d}$ denote the input variable space, and $\mathcal{Y} = \{1,2,...,C\}$ represent the set of candidate labels. Our dataset is $D = \{(x_{i},y_{i})\}^{N}_{i=1} \subseteq \mathcal{X} \times \mathcal{Y}$. We assume that each instance $(x,y)$ is drawn i.i.d from the data population distribution $\mathcal{P}$. Let a deep neural network $\phi_{\theta} : \mathcal{X} \rightarrow \mathcal{Y}$ be parameterized by $\theta \in \Theta$. This paper employs cross-entropy for the generic loss function, $\ell(x,y;\theta)$.
\subsection{Previous Researches: Optimization and Parameterization of Dataset Distillation}
\paragraph{Dataset Distillation Formulation.} The goal of dataset distillation is to produce a cardinality-reduced dataset $S$ from the given dataset $D$, while maximally retaining the task-relevant information of $D$. The objective of dataset distillation is formulated as follows:
\begin{align}
    \min_{S}\mathbb{E}_{(x,y) \in \mathcal{P}}[\ell(x,y;\theta_{S})] \,\,\, \text{where} \,\,\, \theta_{S} = \argmin_{\theta} \frac{1}{|S|}\sum_{(x_i,y_i) \in S}\ell(x_{i},y_{i};\theta)
\label{preliminary:eq:DD objective}
\end{align}
\begin{wrapfigure}{r}{0.45\textwidth}
\vspace{-0.2in}
    \begin{minipage}{\linewidth}
    \centering
    \captionsetup[subfigure]{justification=centering}
        \includegraphics[width=0.95\linewidth]{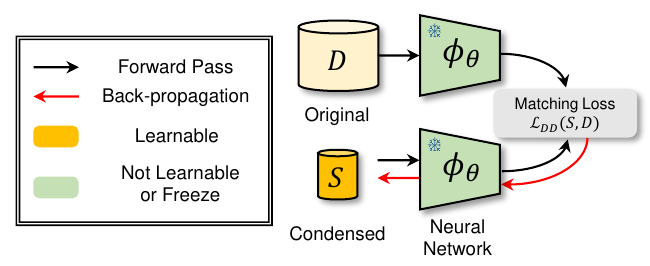}
        \vspace{-0.1in}
        \subcaption{Dataset distillation on input-sized variable.} \label{preliminary:fig:distill_overview_a}
    
        \includegraphics[width=0.95\linewidth]{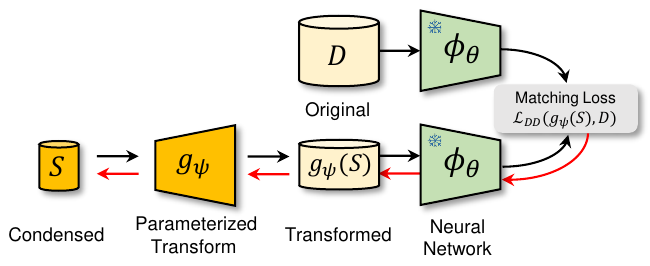}
        \vspace{-0.1in}
        \subcaption{Dataset distillation on variable with\\parameterized transform.} \label{preliminary:fig:distill_overview_b}
        
        \includegraphics[width=0.95\linewidth]{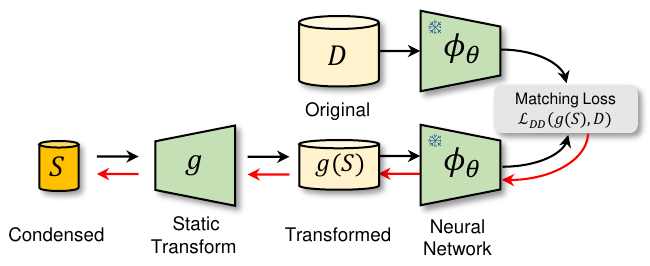}
        \vspace{-0.1in}
        \subcaption{Dataset distillation on variable with\\static transform.} \label{preliminary:fig:distill_overview_c}
    \end{minipage}
\caption{Comparison of different parameterization strategies for optimization of $S$.} \label{preliminary:fig:distill_overview}
\vspace{-0.5in}
\end{wrapfigure}
However, the optimization of Eq. \eqref{preliminary:eq:DD objective} is costly and not scalable since it is a bi-level problem \citep{zhao2020dataset,sachdeva2023data}. To avoid this problem, various proxy objectives are proposed to match the task-relevant information of $D$ such as gradient \cite{zhao2020dataset}, features \cite{zhao2023dataset}, trajectory \cite{cazenavette2022dataset}, etc. Throughout this paper, we generally express these objectives as $\mathcal{L}_{DD}(S,D)$.
\vspace{-0.1in}
\paragraph{Parameterization of $S$.} Orthogonal to the investigation on objectives, other researchers have also examined the parameterization and corresponding dimensionality of the data variable, $S$. By parameterizing $S$ in a more efficient manner, rather than directly storing input-sized instances, it is possible to distill more instances and enhance the representation capability of $S$ for $D$ \citep{kim2022dataset,liu2022dataset}. Figure \ref{preliminary:fig:distill_overview} divides the existing methods into three categories.

As a method of Figure \ref{preliminary:fig:distill_overview_b}, HaBa \citep{liu2022dataset} proposed a technique for dataset factorization, which involves breaking the dataset into bases and hallucination networks for diverse samples. However, incorporating an additional network in distillation requires a separate budget, which is distinct from the data instances.

As a method of Figure \ref{preliminary:fig:distill_overview_c}, IDC \citep{kim2022dataset} proposed the utilization of an upsampler module in dataset distillation to increase the number of data instances. This parameterization enables the condensation of a single instance with reduced spatial dimensions, thereby increasing the available number of instances. However, the compression still operates on the spatial domain, which results in a significant information loss per instance. Please refer to Appendix \ref{appendix:literature reviews} for the detailed literature reviews.
\begin{figure*}[h!]
    \centering
    \begin{subfigure}[t]{0.61\textwidth}
        \centering
        \includegraphics[height=4.2cm]{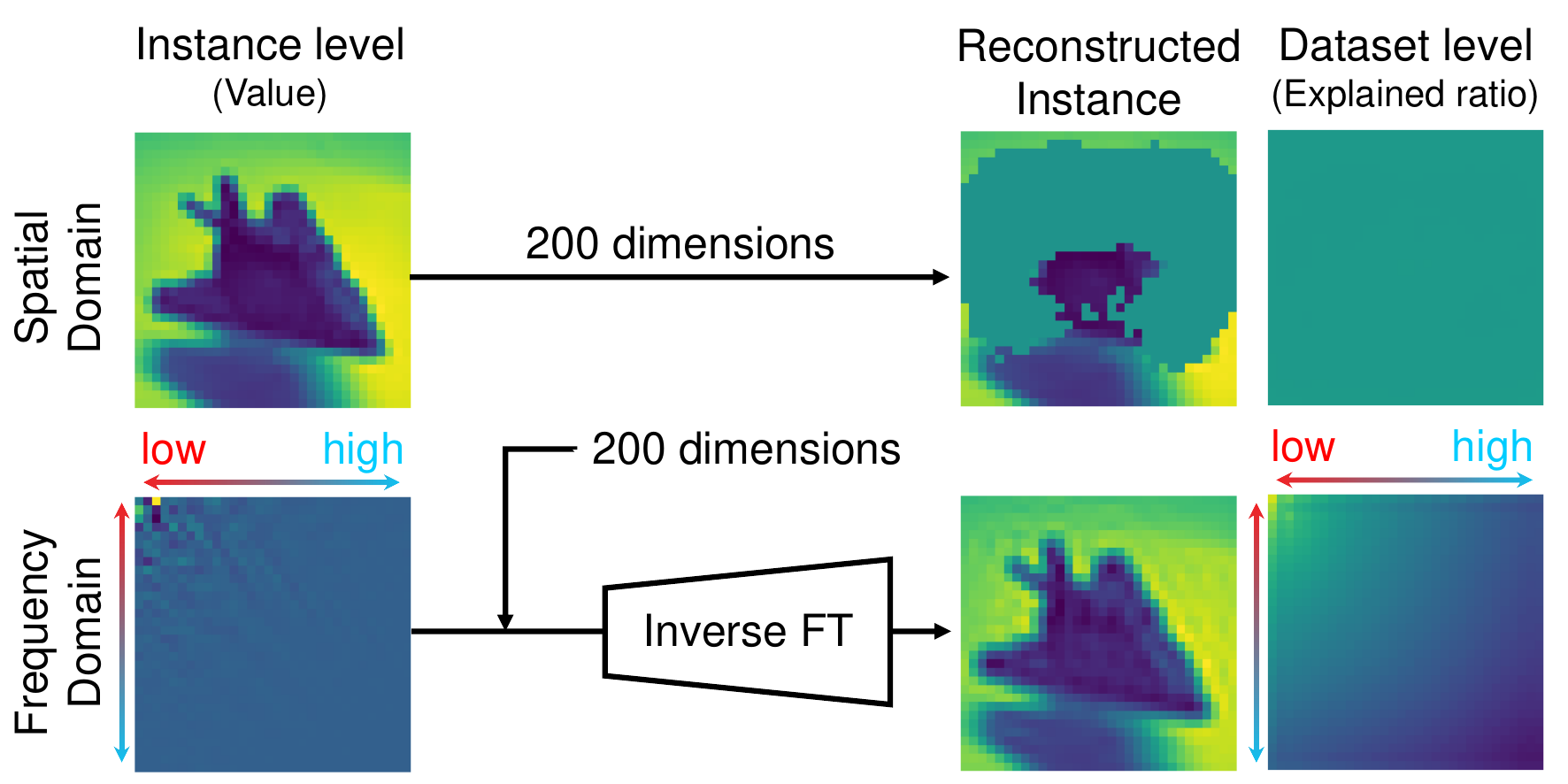}
        \caption{Statistical properties of dataset on spatial and frequency domain.} \label{methodology:fig:domain_comparison}
    \end{subfigure}
    \hfill
    \begin{subfigure}[t]{0.37\textwidth}
        \centering
        \includegraphics[height=4.2cm]{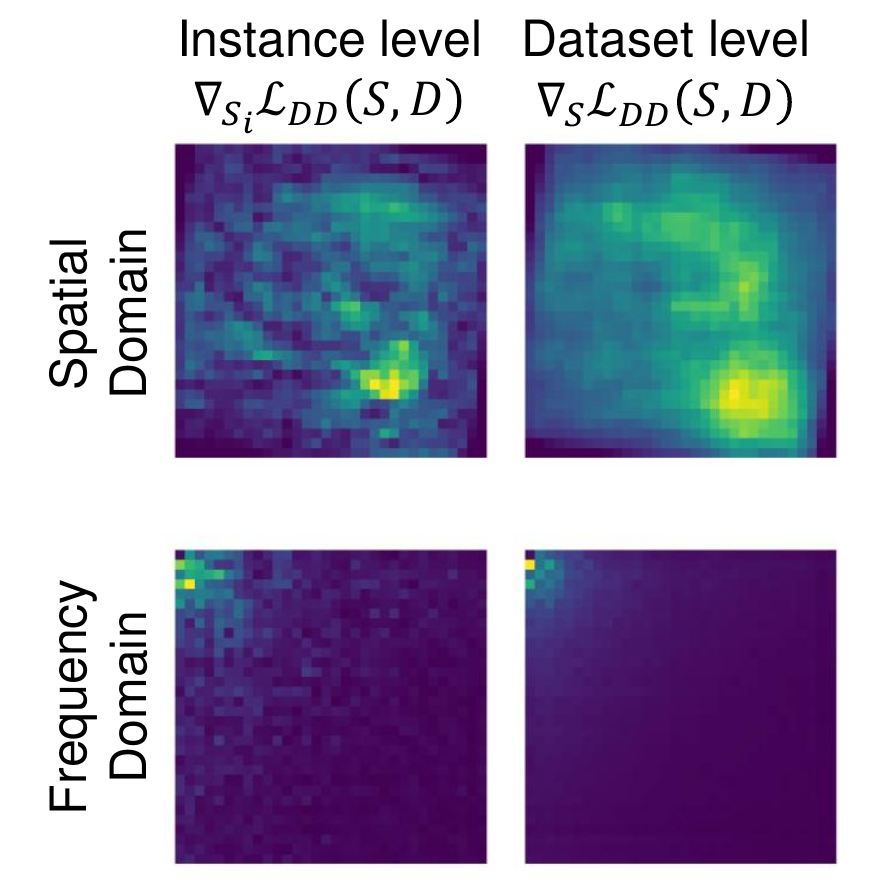}
        \caption{Tendency of dataset distillation loss.} \label{methodology:fig:domain_comparison_2}
    \end{subfigure}
    \caption{Visualization of the information concentrated property. (a) On the frequency domain, a large proportion of both amplitude (left) and explained variance ratio (right) are concentrated in a few dimensions. (b) The magnitude of the gradient for $\mathcal{L}_{DD}(S,D)$ is also concentrated in a few frequency dimensions. We utilize feature matching for $\mathcal{L}_{DD}(S,D)$. A brighter color denotes a higher value. Best viewed in color.} 
    \label{methodology:fig:motivation}
\end{figure*}
\vspace{-0.3in}
\section{Methodology} \label{methodology}
\subsection{Motivation}
The core idea of this paper is to utilize a frequency domain to compress information from the spatial domain into a small number of frequency dimensions. Therefore, we briefly introduce the frequency transform, which is a connector between spatial and frequency domains. Frequency-transform function, denoted as $\mathcal{F}$, converts an input signal $x$ to the frequency domain. It results in the frequency representation, $f = \mathcal{F}(x)\in\mathbb{R}^{d_{1}\times d_{2}}$.\footnote{Although some frequency transform handles the complex space, we use the real space for brevity.} The element-wise expression of $f$ with $x$ is as follows:
\begin{equation}
    f_{u,v}=\sum_{a=0}^{d_{1}-1}\sum_{b=0}^{d_{2}-1}x_{a,b}\,\phi(a,b,u,v)
\label{eq:ft}	
\end{equation}
Here, $\phi(a,b,u,v)$ is a basis function for the frequency domain, and the form of $\phi(a,b,u,v)$ differs by the choice of specific transform function, $\mathcal{F}$. In general, an inverse of $\mathcal{F}$, $\mathcal{F}^{-1}$, exists so that it enables the reconstruction of the original input $x$, from its frequency components $f$ i.e. $x=\mathcal{F}^{-1}(\mathcal{F}(x))$.

A characteristic of the frequency domain is that there exist specific frequency dimensions that encapsulate the major information of data instances from the other domain. According to \citep{ahmed1974discrete,wang2000energy}, natural images tend to exhibit the energy compaction property, which is a concentration of energy in the low-frequency region. Consequently, it becomes feasible to compress the provided instances by exclusively leveraging the low-frequency region. Remark \ref{remark} states that there exists a subset of frequency dimensions that minimize the reconstruction error of a given instance as follows:
\begin{remark}\label{remark}(\cite{rebaza2021first}) Given $d$-dimension data instance of a signal $x=[x_0,...,x_d]^T$, let $f=[f_0,...,f_d]^T$ be its frequency representation with discrete cosine transform (DCT), i.e. $f=DCT(x)$. Also, let $f^{(k)}$ denote the $k$ elements of $f$ while other elements are 0. Then, the minimizer of the reconstruction error $\|x-{DCT}^{-1}(f^{(k)})\|_{2}$ is $[f_0,,,,f_k,0,...0]$.
\end{remark}

Figure \ref{methodology:fig:domain_comparison} supports the remark. When we apply a frequency transform to an image, we can see that the coefficients in the low-frequency region have very large values compared to the coefficients in other regions. Also, as shown in Remark \ref{remark}, the reconstructed image is quite similar to the original image while only utilizing the lowest frequency dimensions. Beyond the instance level, The right side of Figure \ref{methodology:fig:domain_comparison} represents that certain dimensions in the frequency domain exhibit a concentration of variance. It suggests that the frequency domain requires only a small number of dimensions to describe the total variance of a dataset. On the other hand, the spatial domain does not exhibit this behavior; poor quality of reconstruction, and a similar variance ratio across all dimensions. 
\begin{figure*}[]
    \begin{subfigure}{0.63\textwidth}
    \centering
    \begin{multicols}{3}
        \subcaptionbox{Original dataset \label{methodology:fig:log_evr_all}}{\includegraphics[width=\linewidth]{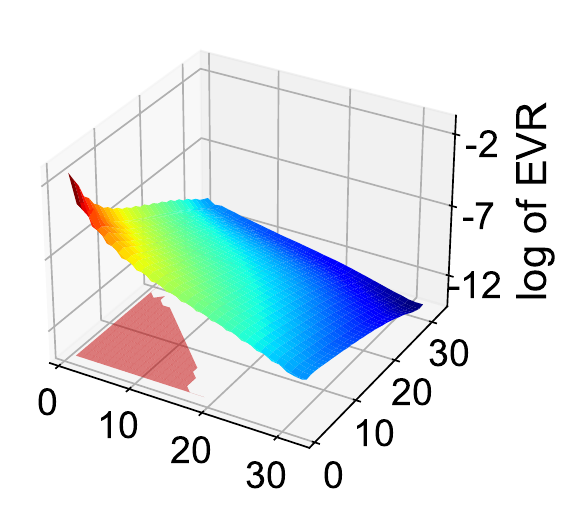}}\par 
        \subcaptionbox{DC (0.0411) \label{methodology:fig:log_evr_DC}}{\includegraphics[width=\linewidth]{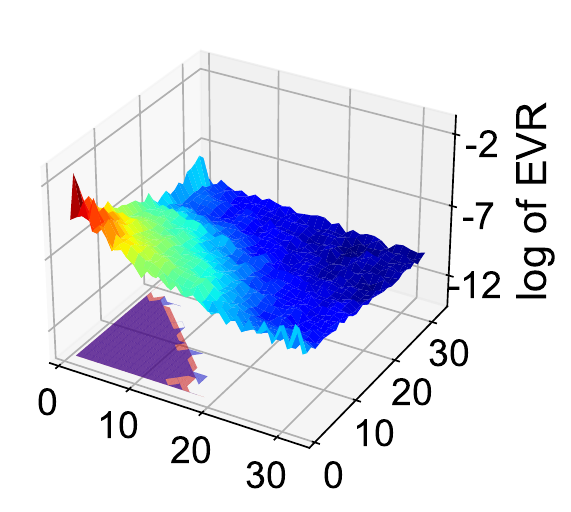}}\par 
        \subcaptionbox{DSA (0.2683) \label{methodology:fig:log_evr_DSA}}{\includegraphics[width=\linewidth]{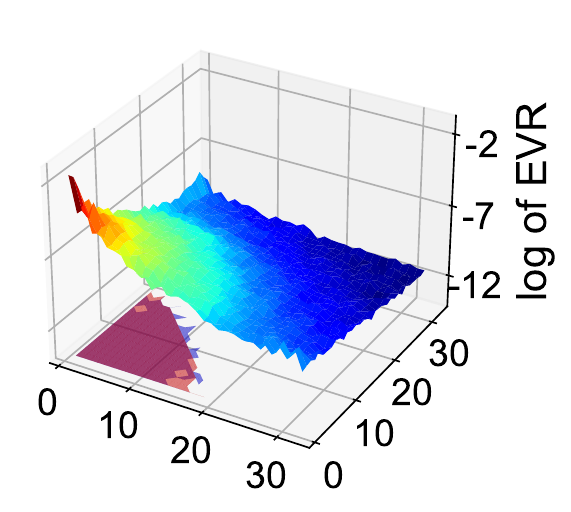}}\par
    \end{multicols}
    \begin{multicols}{3}
        \subcaptionbox{DM (0.0534) \label{methodology:fig:log_evr_DM}}{\includegraphics[width=\linewidth]{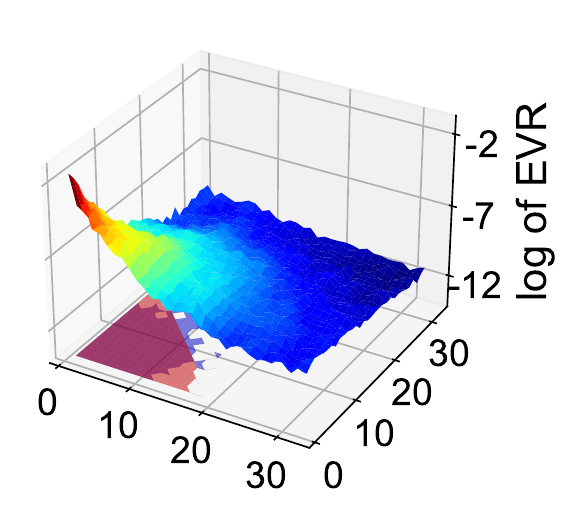}}\par
        \subcaptionbox{IDC (0.1020) \label{methodology:fig:log_evr_IDC}}{\includegraphics[width=\linewidth]{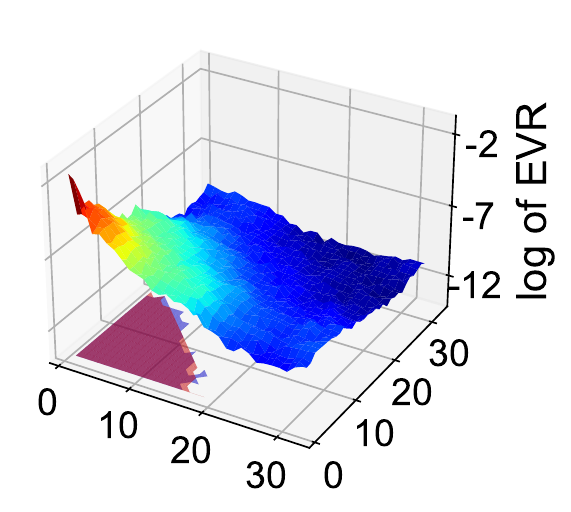}}\par
        \subcaptionbox{TM (0.0242) \label{methodology:fig:log_evr_TM}}{\includegraphics[width=\linewidth]{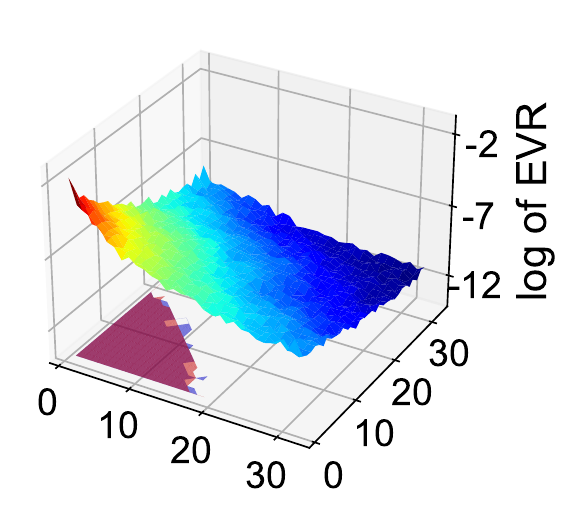}}\par
    \end{multicols}
    \end{subfigure}
    \hfill
    \begin{subfigure}{0.36\textwidth}
        \centering
        \includegraphics[width=\textwidth]{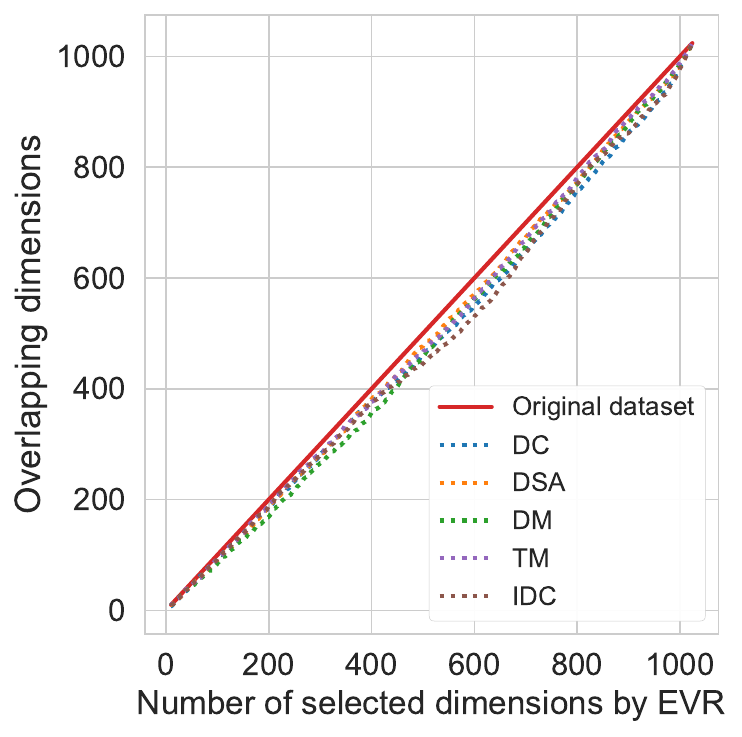}
        \caption{Preservation property of trained $S$} \label{methodology:fig:log_evr_compare}
    \end{subfigure}
    \caption{(a)-(f) The $\log$ of explained variance ratio on each frequency domain dimension (surface) and the degree of overlap between $D$ and $S$ for the top-200 dimensions (floor). The parenthesis is $L_2$ difference of explained variance ratio between $D$ and $S$. (g) It can be observed that trained $S$ mostly preserves the top-$k$ explained variance ratio dimensions of $D$, regardless of $k$.}
\vspace{-0.1in}
\end{figure*}

To investigate whether this tendency is maintained in the dataset distillation, we conducted a frequency analysis of various dataset distillation losses $\mathcal{L}_{DD}(S,D)$. Figure \ref{methodology:fig:domain_comparison_2} shows the magnitude of gradient for $\mathcal{L}_{DD}(S,D)$ i.e. $\| \nabla_{S}\mathcal{L}_{DD}(S,D) \|$. In contrast to the spatial domain, where the gradients w.r.t the spatial domain are uniformly distributed across each dimension; the gradients w.r.t frequency domain are condensed into certain dimensions. This result suggests that many dimensions are needed in the spatial domain to reduce the loss, but only a few specific dimensions are required in the frequency domain. Furthermore, we compared the frequency domain information of $D$ and $S$ which were trained with losses from previous research. As shown in Figures \ref{methodology:fig:log_evr_all} to \ref{methodology:fig:log_evr_TM}, the distribution of explained variance ratios in the frequency domain is very similar across different distillation losses. Also, Figure \ref{methodology:fig:log_evr_compare} shows that trained $S$ mostly preserves the top-$k$ explained variance ratio dimensions of $D$ in the frequency domain, regardless of $k$ and the distillation loss function. Consequently, if there exists a frequency dimension on $D$ with a low explained variance ratio of frequency representations, it can be speculated that the absence of the dimension will have little impact on the optimization of $S$.

\subsection{FreD: \underline{Fre}quency domain-based Dataset \underline{D}istillation}
We introduce the frequency domain-based dataset distillation method, coined FreD, which only utilizes the subset of entire frequency dimensions. Utilizing FreD on the construction of $S$ has several advantages. First of all, since the frequency domain concentrates information in a few dimensions, FreD can be easy to specify some dimensions that preserve the most information. As a result, by reducing the necessary dimensions for each instance, the remaining budget can be utilized to increase the number of condensed images. Second, FreD can be orthogonally applied to existing spatial-based methods without constraining the available loss functions. Figure \ref{methodology:fig:overview} illustrates the overview of FreD with the information flow. FreD consists of three main components: 1) \textit{Synthetic frequency memory}, 2) \textit{Binary mask memory}, and 3) \textit{Inverse frequency transform}.
\vspace{-0.1in}
\paragraph{Synthetic Frequency Memory $F$.} Synthetic frequency memory $F$ consists of $\lvert F \rvert$ frequency-label data pair, i.e. $F=\{(f^{(i)},y^{(i)})\}_{i=1}^{\lvert F \rvert}$. Each $f^{(i)}\in\mathbb{R}^{d}$ is initialized with a frequency representation acquired through the frequency transform of randomly sampled $(x^{(i)},y^{(i)}) \sim D$, i.e. $f^{(i)}=\mathcal{F}(x^{(i)})$.
\begin{figure*}[t]
    \centering
    \begin{subfigure}{0.65\textwidth}
        \centering
        \centerline{\includegraphics[width=\textwidth]{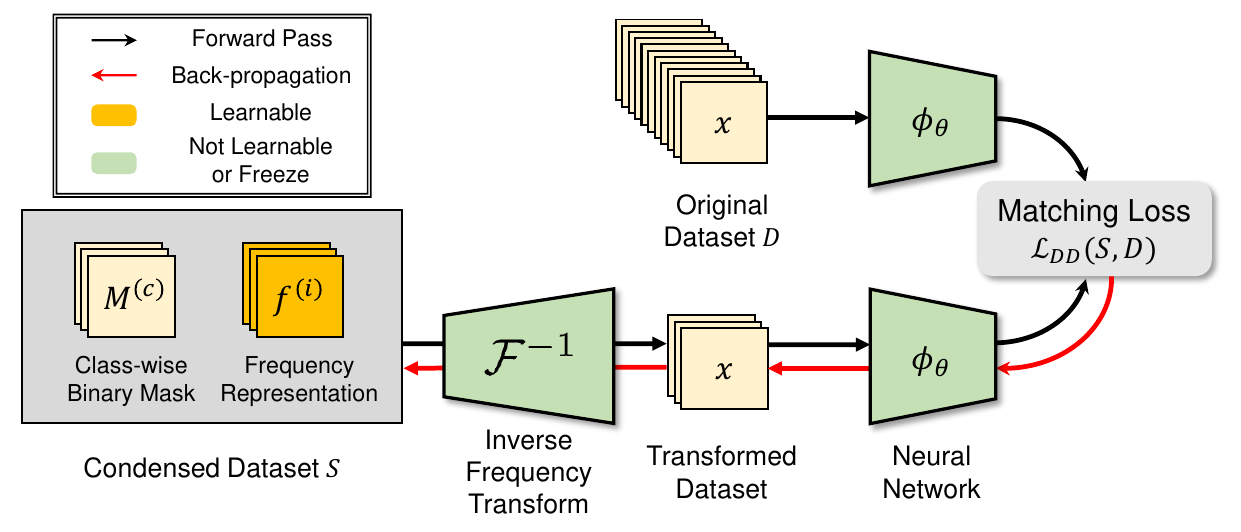}}
        \caption{Overall structure of FreD.} \label{methodology:fig:structure}
    \end{subfigure}
    \hfill
    \begin{subfigure}{0.32\textwidth}
        \centering
        \centerline{\includegraphics[width=\textwidth]{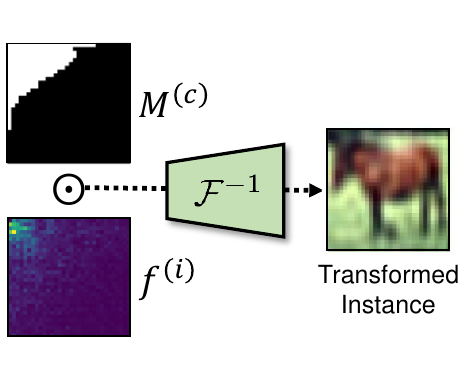}}
        \caption{Inverse frequency transform with binary masking.} \label{methodology:fig:ift}
    \end{subfigure}
\caption{Visualization of the proposed method, Frequency domain-based Dataset Distillation, FreD.} \label{methodology:fig:overview}
\end{figure*}
\vspace{-0.2in}
\paragraph{Binary Mask Memory $M$.} To filter out uninformative dimensions in the frequency domain, we introduce a set of binary masks. Assuming the disparity of each class in the frequency domain, we utilize class-wise masks as $M=\{ M^{(1)},..., M^{(C)} \}$, where each $M^{(c)}\in\{0,1\}^{d}$ denotes the binary mask of class $c$. To filter out superfluous dimensions, each $M^{(c)}$ is constrained to have $k$ non-zero elements, with the remaining dimensions masked as zero. Based on $M^{(c)}$, we acquire the class-wise filtered frequency representation, $M^{(c)}\odot f^{(i)}$, which takes $k$ frequency coefficients as follows: 
\begin{equation}
    M_{u,v}^{(c)}\odot f^{(i)}= 
    \begin{cases}
        f_{u,v}^{(i)} & \text{if } M_{u,v}^{(c)} = 1 \\
        0 & \text{otherwise}
    \end{cases}
\end{equation} 
PCA-based variance analysis is commonly employed for analyzing the distribution data instances. This paper measures the informativeness of each dimension in the frequency domain by utilizing the Explained Variance Ratio (EVR) of each dimension, which we denote as $\eta_{u,v}=\frac{\sigma_{u,v}^{2}}{\sum_{u',v'} \sigma_{u',v'}^{2}}$. Here, $\sigma_{u,v}^{2}$ is the variance of the ${u,v}$-th frequency dimension. We construct a mask, $M_{u,v}^{(c)}$, which only utilizes the top-$k$ dimensions based on the $\eta_{u,v}$ to maximally preserve the class-wise variance of $D$ as follows:
\begin{equation} \label{eq:construct_mask}
    M_{u,v}^{(c)} = 
    \begin{cases} 
        1 & \text{if } \eta_{u,v} \text{ is among the top-$k$ values} \\ 
        0 &  \text{otherwise} 
    \end{cases}
\end{equation}
This masking strategy is efficient as it can be computed solely using the dataset statistics and does not require additional training with a deep neural network. This form of modeling is different from traditional filters such as low/high-pass filters, or band-stop filters, which are common choices in image processing. These frequency-range-based filters pass only certain frequency ranges by utilizing the stylized fact that each instance contains varying amounts of information across different frequency ranges. We conjecture that this type of class-agnostic filter cannot capture the discriminative features of each class, thereby failing to provide adequate information for downstream tasks, i.e. classification. This claim is supported by our empirical studies in Section \ref{exp:ablation}.
\vspace{-0.1in}
\paragraph{Inverse Frequency Transform $\mathcal{F}^{-1}$.} We utilize the inverse frequency transform, $\mathcal{F}^{-1}$, to transform the inferred frequency representation to the corresponding instance on the spatial domain. The characteristics of the inverse frequency transform make it highly suitable as a choice for dataset distillation. First, $\mathcal{F}^{-1}$ is a differentiable function that enables the back-propagation of the gradient to the corresponding frequency domain. Second, $\mathcal{F}^{-1}$ is a static and invariant function, which does not require an additional parameter. Therefore, $\mathcal{F}^{-1}$ does not require any budget and training that could lead to inefficiency and instability. Third, the transformation process is efficient and fast. For $d$-dimension input vector, $\mathcal{F}^{-1}$ requires $\mathcal{O}(d\log{d})$ operation while the convolution layer with $r$-size kernel needs $\mathcal{O}(dr)$. Therefore, in the common situation, where $\log{d}<r$, $\mathcal{F}^{-1}$ is faster than the convolution layer.\footnote{There are several researches which utilize this property to replace the convolution layer with a frequency transform \cite{bengio2007scaling,mathieu2013fast,pratt2017fcnn}.} We provide an ablation study to compare the effectiveness of each frequency transform in Section \ref{exp:ablation} and Appendix \ref{appendix:additional ablation FT}.
\vspace{-0.1in}
\paragraph{Learning Framework.} Following the tradition of dataset distillation \cite{wang2018dataset,zhao2020dataset,kim2022dataset}, the training and evaluation stages of FreD are as follows:
\begin{align} 
    F^{*}&=\argmin_{F}\mathcal{L}_{DD}(\tilde{S},D)\, \,\text{where}\,\,\tilde{S}=\mathcal{F}^{-1}(M\odot F) \quad \quad \text{(Training)} \label{eq5}\\ 
    \theta^{*}&=\argmin_{\theta}\mathcal{L}(\tilde{S}^{*};\theta)\, \,\text{where}\,\,\tilde{S}^{*}=\mathcal{F}^{-1}(M\odot F^{*}) \quad \quad \text{(Evaluation)} \label{eq6}
\end{align} where $M\odot F$ denotes the collection of instance-wise masked representation i.e. $M\odot F=\{(M^{(y^{(i)})}(f^{(i)}), y^{(i)})\}_{i=1}^{\lvert F \rvert}$. By estimating $F^{*}$ from the training with Eq \eqref{eq5}, we evaluate the effectiveness of $F^{*}$ by utilizing $\theta^{*}$, which is a model parameter inferred from training with the transformed dataset, $\tilde{S}=\mathcal{F}^{-1}(M\odot F^{*})$. Algorithm \ref{alg:main} provides the instruction of FreD.
\vspace{-0.1in}
\paragraph{Budget Allocation.} When we can store $n$ instances which is $d$-dimension vector for each class, the budget for dataset distillation is limited by $n\times d$. In FreD, we utilize $k<d$ dimensions for each instance. Therefore, we can accommodate $\lvert F \rvert=\lfloor n(\frac{d}{k}) \rfloor>n$ instances with the same budget. After the training of FreD, we acquire $M\odot F^{*}$, which actually shares the same dimension as the original image. However, we only count $k$ non-zero elements on each $f$ because storing 0 in $d-k$ dimension is negligible by small bytes. Please refer to Appendix \ref{appendix:budget} for more discussion on budget allocation.
\begin{algorithm}[t]
	\caption{FreD: Frequency domain-based Dataset Distillation}
	\label{alg:main}
	\begin{algorithmic}[1]
		\STATE{\bfseries Input:} Original dataset $D$; Number of classes $C$; Frequency transform $\mathcal{F}$; Distillation loss $\mathcal{L}_{DD}$; Dimension budget per instance $k$; Number of frequency representations $\lvert F \rvert$; Learning rate $\alpha$
        \STATE Initialize $F=\emptyset$
        \FOR{$c=1$ {\bfseries to} $C$} 
            \STATE $F^{(c)} \leftarrow \{(\mathcal{F}(x^{(i)}), y^{(i)})\}_{i=1}^{\frac{\lvert F \rvert}{C}}$ from a class-wise mini-batch $\{(x^{(i)},y^{(i)})\}_{i=1}^{\frac{\lvert F \rvert}{C}}\sim D^{(c)}$
            \STATE Initialize $M^{(c)}$ by Eq. \eqref{eq:construct_mask} of the main paper
        \ENDFOR
        \REPEAT
            \STATE $F \leftarrow F-\alpha\nabla\mathcal{L}_{DD}(\mathcal{F}^{-1}(M\odot B_{F}),B_{D})$ from a mini-batch $B_{D}\sim D$ and $B_{F}\sim F$
        \UNTIL convergence
        \STATE{\bfseries Output:} Masked frequency representations $M \odot F$
	\end{algorithmic}
\end{algorithm}
\vspace{-0.05in}
\subsection{Theoretic Analysis of FreD} \label{section:theory}
This section provides theoretical justification for the dimension selection of FreD by EVR, $\eta$. For validation, we assume that $\mathcal{F}$ is linearly bijective.
\begin{restatable}{prop}{proposition} \label{methodology:prop1}
    Let domain $A$ and $B$ be connected by a linear bijective function, $W$. The sum of $\eta$ over a subset of dimensions in domain $A$ for a dataset $X$ is equal to the sum of $\eta$ for the dataset transformed to domain $B$ using only the corresponding subset of dimensions.
\end{restatable}
Proposition \ref{methodology:prop1} claims that if two different domains are linearly bijective, the sum of EVR that utilizes only specific dimensions remains the same even when transformed into a different domain. In other words, if the sum of EVR of specific dimensions in the frequency domain is high, this value can be maintained when transforming a dataset only with those dimensions into the other domain.
\begin{restatable}{corol}{corollary} \label{methodology:corol1}
    Assume that two distinct domains, $B$ and $C$, are linearly bijective with domain $A$ by $W_{B}$ and $W_{C}$. let $X$ be a dataset in domain $A$, and $X_{B}$ and $X_{C}$ be the datasets transformed to domains $B$ and $C$, respectively. Let $V^{*}_{B,k}$ and $V^{*}_{C,k}$ be the set of $k$ dimension indexes that maximize $\eta$ in each domain. Let $\eta^{*}_{B,k}$ and $\eta^{*}_{C,k}$ be the corresponding sum of $\eta$ for each domain. If $\eta^{*}_{B,k} \geq \eta^{*}_{C,k}$, then the sum of $\eta$ for $W_{V^{*}_{B,k}}X_{B}$ is greater than that of $W_{V^{*}_{C,k}}X_{C}$.
\end{restatable}
\begin{wrapfigure}{r}{0.575\textwidth}
\vspace{-0.15in}
    \centering
    \begin{subfigure}{0.49\linewidth}
        \centering
        \centerline{\includegraphics[width=\textwidth]{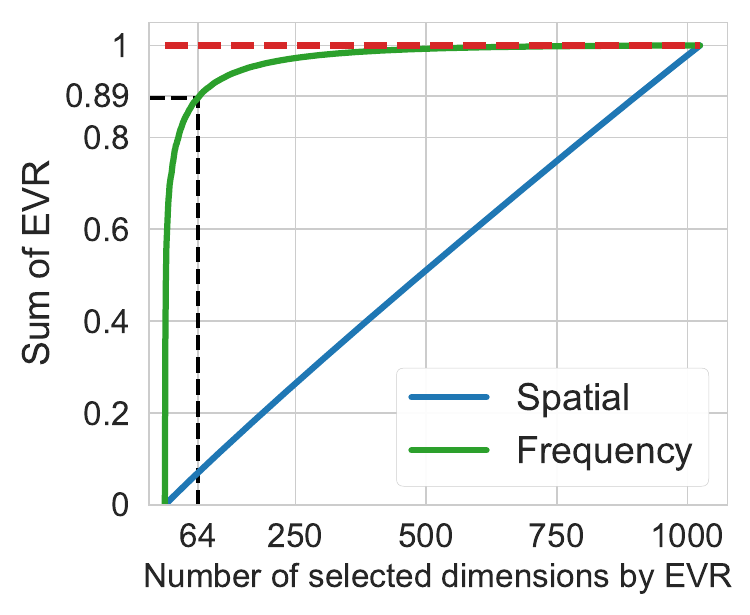}}
        \caption{Comparison of domain} \label{methdology:fig:theoretical evidence 1}
    \end{subfigure}
    \hfill
    \begin{subfigure}{0.49\linewidth}
    \centering
        \centerline{\includegraphics[width=\textwidth]{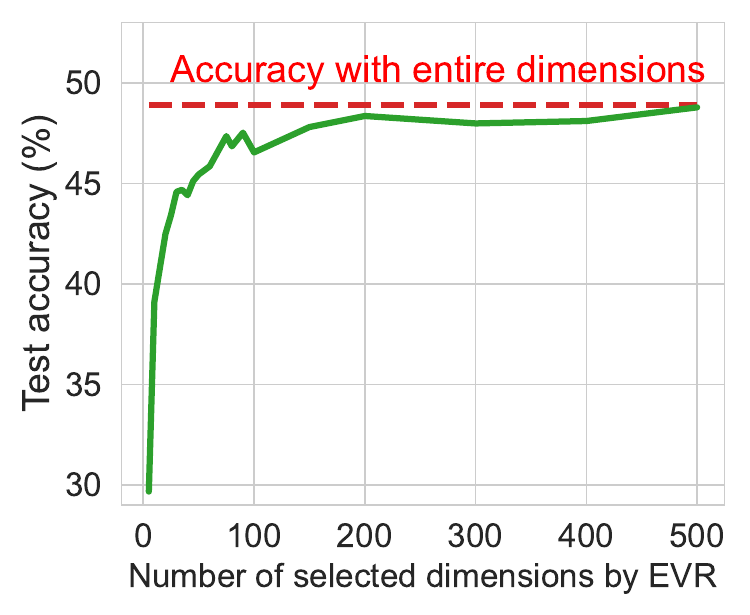}}
        \caption{Efficiency of EVR selection} \label{methodology:fig:theoretical evidence 2}
    \end{subfigure}
\caption{Empirical evidence for EVR-based selection.} \label{methodology:fig:theoretical evidence}
\vspace{-0.25in}
\end{wrapfigure}
Corollary \ref{methodology:corol1} claims that for multiple domains that have a bijective relationship with a specific domain, if a certain domain can obtain the high value of sum of EVR with the same number of dimensions, then when inverse transformed, it can improve the explanation in original domain. Based on the information concentration of the frequency domain, FreD can be regarded as a method that can better represent the distribution of the original dataset, $D$, based on $S$. PCA, which sets the principal components of the given dataset as new axes, can ideally preserve $\eta$. However, PCA, being a data-driven transform function, cannot be practically utilized as a method for dataset distillation. Please refer to Appendix \ref{appendix:comparison with PCA} for supporting evidence of this claim.

Figure \ref{methdology:fig:theoretical evidence 1} shows the sum of $\eta$ in descending order in both the spatial and frequency domains. First, we observe that as the number of selected dimensions in the frequency domain increases, frequency domain is much faster in converging to the total variance than the spatial domain. This result shows that the EVR-based dimension selection in the frequency domain is more effective than the selection in the spatial domain. In terms of training, Figure \ref{methodology:fig:theoretical evidence 2} also shows that the performance of a dataset constructed with a very small number of frequency dimensions converges rapidly to the performance of a dataset with the entire frequency dimension.

\begin{table}[t] 
    \centering
    \caption{Test accuracies (\%) on SVHN, CIFAR-10, and CIFAR-100. "IPC" denotes the number of images per class. "\#Params" denotes the total number of budget parameters. The best results and the second-best result are highlighted in \textbf{bold} and \underline{underline}, respectively.}
    \label{experiments:tab:main}
    \adjustbox{max width=\textwidth}{%
    \begin{tabular}{c c ccc ccc ccc}
    \toprule \toprule
     & & \multicolumn{3}{c}{SVHN} & \multicolumn{3}{c}{CIFAR-10} & \multicolumn{3}{c}{CIFAR-100} \\
    \cmidrule(lr){3-5} \cmidrule(lr){6-8} \cmidrule(lr){9-11} 
     & IPC & 1 & 10 & 50 & 1 & 10 & 50 & 1 & 10 & 50 \\
     & \#Params & 30.72k & 307.2k & 1536k & 30.72k & 307.2k & 1536k & 30.72k & 307.2k & 1536k \\
    \midrule
    \multirow{2}{*}{Coreset} & Random & 14.6 \small{$\pm1.6$} & 35.1 \small{$\pm4.1$} & 70.9 \small{$\pm0.9$} & 14.4 \small{$\pm2.0$} & 26.0 \small{$\pm1.2$} & 43.4 \small{$\pm1.0$} & 4.2 \small{$\pm0.3$} & 14.6 \small{$\pm0.5$} & 30.0 \small{$\pm0.4$} \\
     & Herding & 20.9 \small{$\pm1.3$} & 50.5 \small{$\pm3.3$} & 72.6 \small{$\pm0.8$} & 21.5 \small{$\pm1.3$} & 31.6 \small{$\pm0.7$} & 40.4 \small{$\pm0.6$} & 8.4 \small{$\pm0.3$} & 17.3 \small{$\pm0.3$} & 33.7 \small{$\pm0.5$} \\
    \midrule 
    \multirow{7}{*}{\makecell{Input-sized\\parameterization}} & DC & 31.2 \small{$\pm1.4$} & 76.1 \small{$\pm0.6$} & 82.3 \small{$\pm0.3$} & 28.3 \small{$\pm0.5$} & 44.9 \small{$\pm0.5$} & 53.9 \small{$\pm0.5$} & 12.8 \small{$\pm0.3$} & 25.2 \small{$\pm0.3$} & - \\
     & DSA & 27.5 \small{$\pm1.4$} & 79.2 \small{$\pm0.5$} & 84.4 \small{$\pm0.4$} & 28.8 \small{$\pm0.7$} & 52.1 \small{$\pm0.5$} & 60.6 \small{$\pm0.5$} & 13.9 \small{$\pm0.3$} & 32.3 \small{$\pm0.3$} & 42.8 \small{$\pm0.4$} \\
     & DM & - & - & - & 26.0 \small{$\pm0.8$} & 48.9 \small{$\pm0.6$} & 63.0 \small{$\pm0.4$} & 11.4 \small{$\pm0.2$} & 29.7 \small{$\pm0.2$} & 43.6 \small{$\pm0.4$} \\
     & CAFE+DSA & 42.9 \small{$\pm3.0$} & 77.9 \small{$\pm0.6$} & 82.3 \small{$\pm0.4$} & 31.6 \small{$\pm0.8$} & 50.9 \small{$\pm0.5$} & 62.3 \small{$\pm0.4$} & 14.0 \small{$\pm0.2$} & 31.5 \small{$\pm0.2$} & 42.9 \small{$\pm0.2$} \\
     & TM & 58.5 \small{$\pm1.4$} & 70.8 \small{$\pm1.8$} & 85.7 \small{$\pm0.1$} & 46.3 \small{$\pm0.8$} & 65.3 \small{$\pm0.7$} & 71.6 \small{$\pm0.2$} & 24.3 \small{$\pm0.2$} & 40.1 \small{$\pm0.4$} & \underline{47.7} \small{$\pm0.2$} \\
     & KIP & 57.3 \small{$\pm0.1$} & 75.0 \small{$\pm0.1$} & 80.5 \small{$\pm0.1$} & \underline{49.9} \small{$\pm0.2$} & 62.7 \small{$\pm0.3$} & 68.6 \small{$\pm0.2$} & 15.7 \small{$\pm0.2$} & 28.3 \small{$\pm0.1$} & - \\
     & FRePo & - & - & - & 46.8 \small{$\pm0.7$} & 65.5 \small{$\pm0.4$} & 71.7 \small{$\pm0.2$} & 28.7 \small{$\pm0.1$} & \underline{42.5} \small{$\pm0.2$} & 44.3 \small{$\pm0.2$} \\
    \midrule
    \multirow{3}{*}{Parameterization} & IDC & 68.1 \small{$\pm0.1$} & \underline{87.3} \small{$\pm0.2$} & \underline{90.2} \small{$\pm0.1$} & \underline{50.0} \small{$\pm0.4$} & 67.5 \small{$\pm0.5$} & \underline{74.5} \small{$\pm0.1$} & - & - & - \\
     & HaBa & \underline{69.8} \small{$\pm1.3$} & 83.2 \small{$\pm0.4$} & 88.3 \small{$\pm0.1$} & 48.3 \small{$\pm0.8$} & \underline{69.9} \small{$\pm0.4$} & 74.0 \small{$\pm0.2$} & \underline{33.4} \small{$\pm0.4$} & 40.2 \small{$\pm0.2$} & 47.0 \small{$\pm0.2$} \\
     & \gc FreD & \gc \textbf{82.2} \small{$\pm0.6$} & \gc \textbf{89.5} \small{$\pm0.1$} & \gc \textbf{90.3} \small{$\pm0.3$} & \gc \textbf{60.6} \small{$\pm0.8$} & \gc \textbf{70.3} \small{$\pm0.3$} & \gc \textbf{75.8} \small{$\pm0.1$} & \gc \textbf{34.6} \small{$\pm0.4$} & \gc \textbf{42.7} \small{$\pm0.2$} & \gc \textbf{47.8} \small{$\pm0.1$} \\
    \midrule
    \multicolumn{2}{c}{Entire original dataset} & \multicolumn{3}{c}{95.4 \small{$\pm0.1$}} & \multicolumn{3}{c}{84.8 \small{$\pm0.1$}} & \multicolumn{3}{c}{56.2 \small{$\pm0.3$}} \\
    \midrule \midrule
    \multirow{3}{*}{\makecell{Increment of\\decoded instances}} & IDC & $\times 5$ & $\times 5$ & $\times 5$ & $\times 5$ & $\times 5$ & $\times 5$ & - & - & - \\  
    & HaBa & $\times 5$ & $\times 5$ & $\times 5$ & $\times 5$ & $\times 5$ & $\times 5$ & $\times 5$ & $\times 5$ & $\times 5$ \\
    & FreD & $\times 16$ & $\times 8$ & $\times 4$ & $\times 16$ & $\times 6.4$ & $\times 4$ & $\times 8$ & $\times 2.56$ & $\times 2.56$ \\
    \bottomrule \bottomrule
  \end{tabular}
    }
\end{table}
\begin{figure}[t]
    \centering
    \begin{subfigure}{0.325\textwidth}
        \centering
        \centerline{\includegraphics[width=\textwidth]{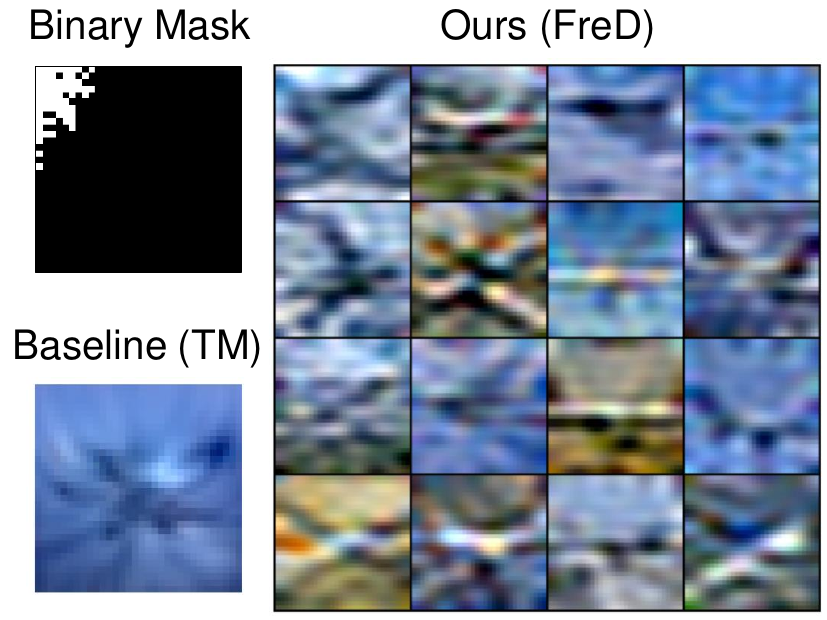}}
        \caption{Class: Airplane}
        \label{experiments:fig:qual-airplane}
    \end{subfigure}
    \hfill
    \begin{subfigure}{0.325\textwidth}
        \centering
        \centerline{\includegraphics[width=\textwidth]{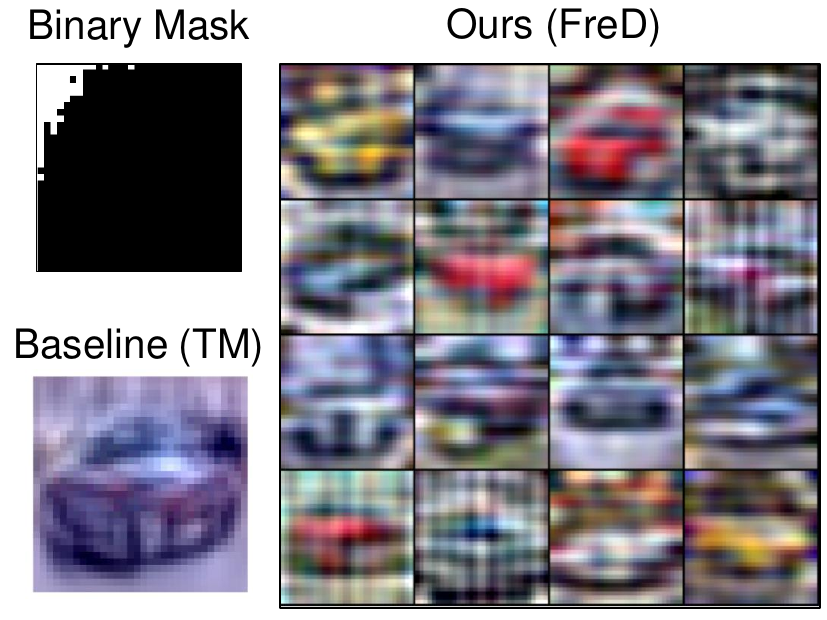}}
        \caption{Class: Automobile}
        \label{experiments:fig:qual-automobile}
    \end{subfigure}
    \hfill
    \begin{subfigure}{0.325\textwidth}
        \centering
        \centerline{\includegraphics[width=\textwidth]{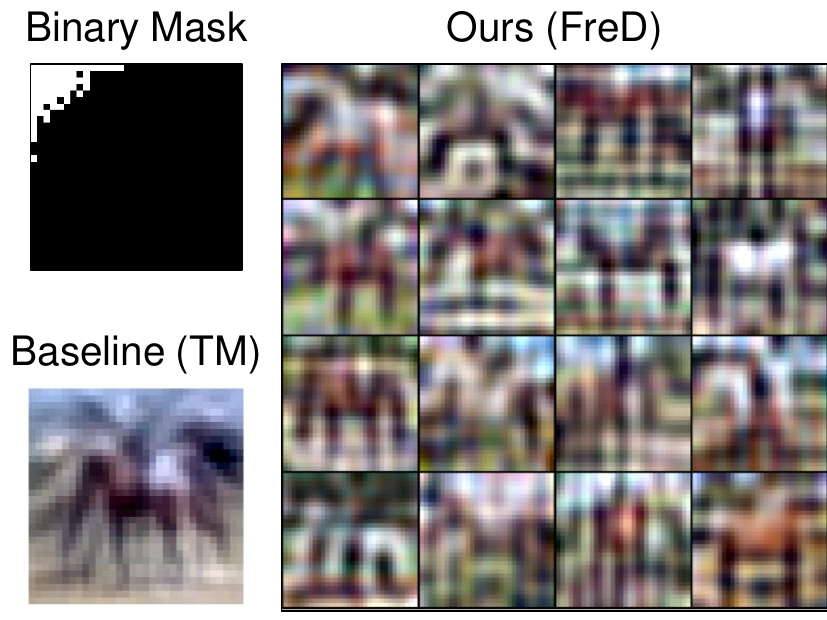}}
        \caption{Class: Horse}
        \label{experiments:fig:qual-horse}
    \end{subfigure}
    \caption{Visualization of the binary mask, condensed image from TM, and the transformed images of FreD on CIFAR-10 with IPC=1 (\#Params=30.72k). Binary masks and trained synthetic images are all the same size. The size of mask and baseline image were enlarged for a better layout.}
    \label{experiments:fig:qual}
\vspace{-0.2in}
\end{figure}
\vspace{-0.1in}
\section{Experiments}
\subsection{Experiment Setting}
\vspace{-0.05in}
We evaluate the efficacy of FreD on various benchmark datasets, i.e. SVHN \cite{netzer2011reading}, CIFAR-10, CIFAR-100 \cite{krizhevsky2009learning} and ImageNet-Subset \cite{howard2019smaller,cazenavette2022dataset,cazenavette2023generalizing}. Please refer to Appendix \ref{appendix:additional results} for additional experimental results on other datasets.
We compared FreD with both methods: 1) a method to model $S$ as an input-sized variable; and 2) a method, which parameterizes $S$, differently. As the learning with input-sized $S$, we chose baselines as DD \cite{wang2018dataset}, DSA \cite{zhao2021dataset}, DM \cite{zhao2023dataset}, CAFE+DSA \cite{wang2022cafe}, TM \cite{cazenavette2022dataset}, KIP \cite{nguyen2021dataset} and FRePo \cite{zhou2022dataset}. We also selected IDC \cite{kim2022dataset} and HaBa \cite{liu2022dataset} as baselines, which is categorized as the parameterization on $S$. we also compare FreD with core-set selection methods, such as random selection and herding \cite{welling2009herding}. We use trajectory matching objective (TM) \cite{cazenavette2022dataset} for $\mathcal{L}_{DD}$ as a default although FreD can use any dataset distillation loss. We evaluate each method by training 5 randomly initialized networks from scratch on optimized $S$. Please refer to Appendix \ref{appendix:settings} for a detailed explanation of datasets and experiment settings.

\subsection{Experimental Results}
\paragraph{Performance Comparison.} Table \ref{experiments:tab:main} presents the test accuracies of the neural network, which is trained on $S$ inferred from each method. FreD achieves the best performances in all experimental settings. Especially, when the limited budget is extreme, i.e. IPC=1 (\#Params=30.72k); FreD shows significant improvements compared to the second-best performer: 12.4\%p in SVHN and 10.6\%p in CIFAR-10. This result demonstrates that using the frequency domain, where information is concentrated in specific dimensions, has a positive effect on efficiency and performance improvement, especially in situations where the budget is very small. Please refer to Appendix \ref{appendix:additional results} for additional experimental results on other datasets.
\vspace{-0.1in}

\paragraph{Qualitative Analysis.} Figure \ref{experiments:fig:qual} visualizes the synthetic dataset by FreD on CIFAR-10 with IPC=1 (\#Params=30.72k). In this setting, we utilize 64 frequency dimensions per channel, which enables the construction of 16 images per class under the same budget. The results show that each class contains diverse data instances. Furthermore, despite of huge reduction in dimensions i.e. $64/1024=6.25\%$, each image contains class-discriminative features. We also provide the corresponding binary masks, which are constructed by EVR value. As a result, the low-frequency dimensions in the frequency domain were predominantly selected. It supports that the majority of frequency components for image construction are concentrated in the low-frequency region \cite{ahmed1974discrete,wang2000energy}. Furthermore, it should be noted that our EVR-based mask construction does not enforce keeping the low-frequency components, what EVR only enforces is keeping the components with a higher explanation ratio on the image feature. Therefore, the constructed binary masks are slightly different for each class. Please refer to Appendix \ref{appendix:more visualization} for more visualization.
\vspace{-0.1in}

\paragraph{Compatibility of Parameterization.} The parameterization method in dataset distillation should show consistent performance improvement across different distillation losses and test network architectures. Therefore, we conduct experiments by varying the dataset distillation loss and test network architectures. In the case of HaBa, conducting an experiment at IPC=1 is structurally impossible due to the existence of a hallucination network. Therefore, for a fair comparison, we basically follow HaBa's IPC setting such as IPC=2,11,51. In Table \ref{experiments:tab:compatibility}, FreD shows the highest performance improvement for all experimental combinations. Specifically, FreD achieves a substantial performance gap to the second-best performer up to 10.8\%p in training architecture and 10.6\%p in cross-architecture generalization. Furthermore, in the high-dimensional dataset cases, Table \ref{experiments:tab:imagenet} verifies that FreD consistently outperforms other parameterization methods. These results demonstrate that the frequency domain exhibits high compatibility and consistent performance improvement, regardless of its association with dataset distillation objective and test network architecture.
\begin{table}[t]
    \centering
    \caption{Test accuracies (\%) on CIFAR-10 under various dataset distillation loss and cross-architecture. "DC/DM/TM" denote the gradient/feature/trajectory matching for dataset distillation loss, respectively. We utilize AlexNet \cite{krizhevsky2017imagenet}, VGG11 \cite{simonyan2014very}, and ResNet18 \cite{he2016deep} for cross-architecture.} \label{experiments:tab:compatibility}
    \adjustbox{max width=\textwidth}{%
    \begin{tabular}{c c ccc ccc ccc}
    \toprule \toprule
    & & \multicolumn{3}{c}{DC} & \multicolumn{3}{c}{DM} & \multicolumn{3}{c}{TM} \\
    \cmidrule(lr){3-5} \cmidrule(lr){6-8} \cmidrule(lr){9-11} 
     & IPC & 2 & 11 & 51 & 2 & 11 & 51 & 2 & 11 & 51 \\
     & \#Params & 61.44k & 337.92k & 1566.72k & 61.44k & 337.92k & 1566.72k & 61.44k & 337.92k & 1566.72k \\
    \midrule
    
    \multirow{4}{*}{ConvNet} & Vanilla & 31.4 \small{$\pm0.2$} & 45.3 \small{$\pm0.3$} & 54.2 \small{$\pm0.6$} & 34.6 \small{$\pm0.5$} & 50.4 \small{$\pm0.4$} & 62.0 \small{$\pm0.3$} & 50.6 \small{$\pm1.0$} & 63.9 \small{$\pm0.3$} & 69.8 \small{$\pm0.5$} \\
     & w/ IDC & \underline{35.2} \small{$\pm0.5$} & \underline{53.8} \small{$\pm0.4$} & 56.4 \small{$\pm0.4$} & \underline{45.1} \small{$\pm0.5$} & \underline{59.3} \small{$\pm0.4$} & \underline{64.6} \small{$\pm0.3$} & 56.1 \small{$\pm0.4$} & 60.9 \small{$\pm0.4$} & 71.1 \small{$\pm0.4$} \\
     & w/ HaBa & 34.1 \small{$\pm0.5$} & 49.9 \small{$\pm0.5$} & \underline{58.9} \small{$\pm0.2$} & 37.3 \small{$\pm0.1$} & 56.8 \small{$\pm0.1$} & 64.4 \small{$\pm0.4$} & \underline{56.8} \small{$\pm0.4$} & \underline{69.5} \small{$\pm0.3$} & \underline{73.3} \small{$\pm0.2$} \\
     & \gc w/ FreD & \gc \textbf{45.3} \small{$\pm0.5$} & \gc \textbf{55.8} \small{$\pm0.4$} & \gc \textbf{59.8} \small{$\pm0.5$} & \gc \textbf{55.9} \small{$\pm0.4$} & \gc \textbf{61.3} \small{$\pm0.8$} & \gc \textbf{66.6} \small{$\pm0.6$} & \gc \textbf{61.4} \small{$\pm0.3$} & \gc \textbf{70.7} \small{$\pm0.5$} & \gc \textbf{75.5} \small{$\pm0.2$} \\
    \midrule

    \multirow{4}{*}{\makecell{Average of\\Cross-\\Architectures}} & Vanilla & 22.0 \small{$\pm0.9$} & 29.2 \small{$\pm0.9$} & 34.1 \small{$\pm0.6$} & 21.5 \small{$\pm2.2$} & 39.5 \small{$\pm1.1$} & 52.6 \small{$\pm0.7$} & 33.1 \small{$\pm1.1$} & 43.9 \small{$\pm1.4$} & 55.0 \small{$\pm1.0$} \\
     & w/ IDC & \underline{28.7} \small{$\pm1.2$} & \underline{35.4} \small{$\pm0.6$} & \underline{40.2} \small{$\pm0.7$} & \underline{37.3} \small{$\pm1.1$} & \underline{50.5} \small{$\pm0.6$} & \underline{61.3} \small{$\pm0.5$} & 42.5 \small{$\pm1.5$} & 48.7 \small{$\pm1.8$} & 61.5 \small{$\pm1.0$} \\
     & w/ HaBa & 25.4 \small{$\pm0.9$} & 31.4 \small{$\pm0.7$} & 35.5 \small{$\pm0.9$} & 30.1 \small{$\pm0.6$} & 47.0 \small{$\pm0.5$} & 60.1 \small{$\pm0.6$} & \underline{46.4} \small{$\pm1.0$} & \underline{55.8} \small{$\pm1.8$} & \underline{64.0} \small{$\pm0.9$} \\
     & \gc w/ FreD & \gc \textbf{37.3} \small{$\pm0.9$} & \gc \textbf{37.4} \small{$\pm0.7$} & \gc \textbf{42.7} \small{$\pm0.8$} & \gc \textbf{48.1} \small{$\pm0.7$} & \gc \textbf{57.3} \small{$\pm0.8$} & \gc \textbf{65.0} \small{$\pm0.7$} & \gc \textbf{49.7} \small{$\pm1.0$} & \gc \textbf{60.1} \small{$\pm0.7$} & \gc \textbf{69.1} \small{$\pm0.7$} \\
    \bottomrule \bottomrule
  \end{tabular}}
\vspace{-0.1in}
\end{table}
\begin{table}[t]
  \centering
  \begin{minipage}[t]{0.68\textwidth}
    \centering
    \caption{Test accuracies (\%) on ImageNet-Subset ($128 \times 128$) under IPC=2 (\#Params=983.04k).} \label{experiments:tab:imagenet}
    \adjustbox{max width=\textwidth}{%
      \begin{tabular}{c cccccc}
    \toprule \toprule
    Model & ImgNette & ImgWoof & ImgFruit & ImgYellow & ImgMeow & ImgSquawk \\
    \midrule
    TM & 55.2 \small{$\pm1.1$} & 30.9 \small{$\pm1.3$} & 31.8 \small{$\pm1.6$} & 49.7 \small{$\pm1.4$} & 35.3 \small{$\pm2.2$} & 43.9 \small{$\pm0.6$} \\
    w/ IDC & \underline{65.4} \small{$\pm1.2$} & \underline{37.6} \small{$\pm1.6$} & \underline{43.0} \small{$\pm1.5$} & \underline{62.4} \small{$\pm1.7$} & \underline{43.1} \small{$\pm1.2$} & \underline{55.5} \small{$\pm1.2$} \\
    w/ HaBa & 51.9 \small{$\pm1.7$} & 32.4 \small{$\pm0.7$} & 34.7 \small{$\pm1.1$} & 50.4 \small{$\pm1.6$} & 36.9 \small{$\pm0.9$} & 41.9 \small{$\pm1.4$} \\
    \gc w/ FreD & \gc \textbf{69.0} \small{$\pm0.9$} & \gc \textbf{40.0} \small{$\pm1.4$} & \gc \textbf{46.3} \small{$\pm1.2$} & \gc \textbf{66.3} \small{$\pm1.1$} & \gc \textbf{45.2} \small{$\pm1.7$} & \gc \textbf{62.0} \small{$\pm1.3$} \\
    \bottomrule \bottomrule
  \end{tabular}}
  \end{minipage}
  \hfill
  \begin{minipage}[t]{0.30\textwidth}
    \centering
    \caption{Test accuracies (\%) on 3D MNIST.} \label{experiments:tab:3D-MNIST}
    \adjustbox{max width=\textwidth}{%
      \begin{tabular}{c ccc}
        \toprule \toprule
        IPC & 1 & 10 & 50 \\
        \#Params & 40.96k & 409.6k & 2048k \\
        \midrule
        Random & 17.2 \small{$\pm0.5$} & 49.6 \small{$\pm0.7$} & 60.3 \small{$\pm0.7$} \\
        \midrule
        DM & 42.5 \small{$\pm0.9$} & \underline{58.6} \small{$\pm0.8$} & \underline{64.7} \small{$\pm0.5$} \\
          w/ IDC & \underline{51.9} \small{$\pm1.5$} & 54.0 \small{$\pm0.5$} & 56.8 \small{$\pm0.3$} \\
        \gc   w/ FreD & \gc \textbf{54.9} \small{$\pm0.5$} & \gc \textbf{62.9} \small{$\pm0.5$} & \gc \textbf{66.6} \small{$\pm0.7$} \\
        \midrule
          Entire dataset & \multicolumn{3}{c}{78.7 \small{$\pm1.1$}} \\
        \bottomrule \bottomrule
      \end{tabular}}
  \end{minipage}
\vspace{-0.1in}
\end{table}

\begin{table}[t]
  \centering
  \begin{minipage}[t]{0.26\textwidth}
    \centering
    \includegraphics[width=\textwidth]{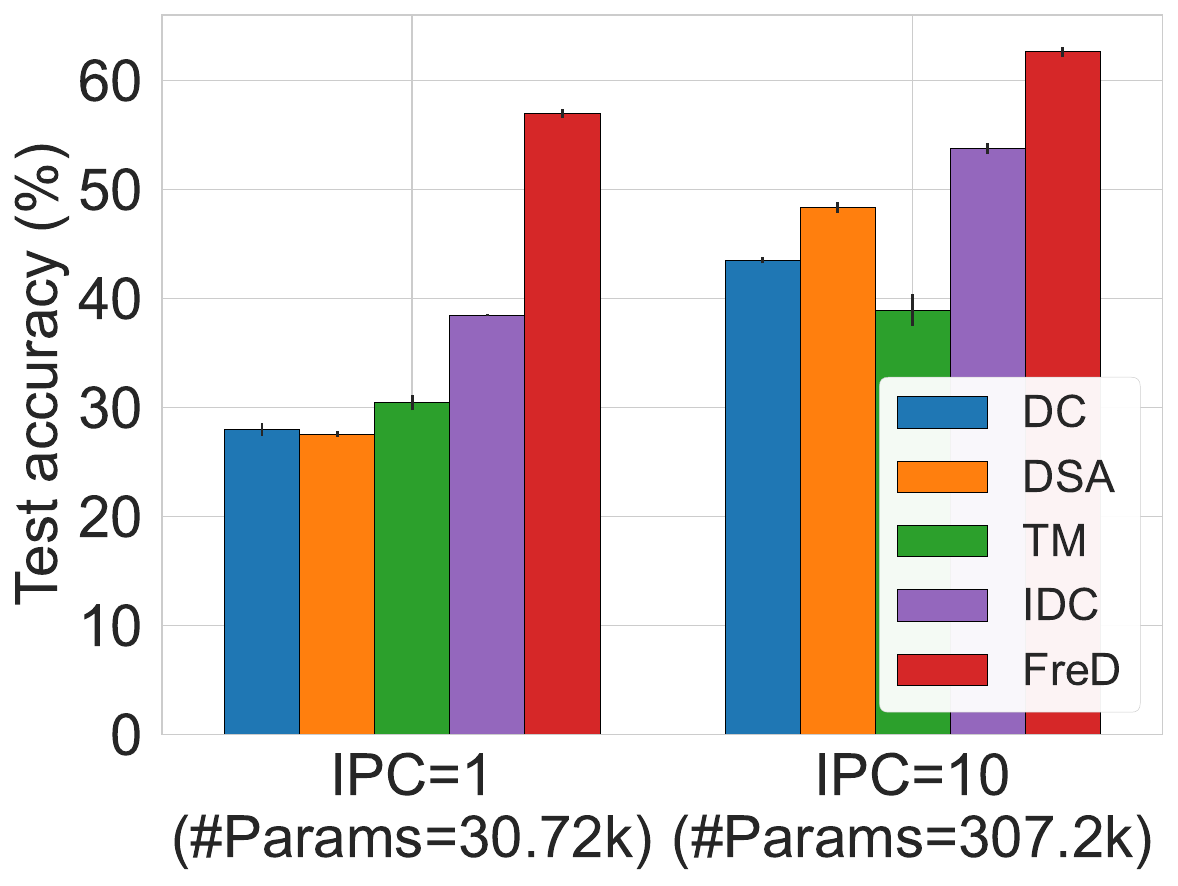}
    \captionof{figure}{Test accuracies (\%) on CIFAR-10-C.}
    \label{experiments:fig:cifar10-c}
  \end{minipage}
  \hfill
  \begin{minipage}[t]{0.73\textwidth}
    \vspace{-1.2in}
    \centering
    \caption{Test accuracies (\%) on ImageNet-Subset-C. Note that the TM on ImgSquawk is excluded because the off-the-shelf synthetic dataset is not the default size of $128\times128$.}
    \label{experiments:tab:imagenet-c}
    \adjustbox{max width=\textwidth}{%
      \begin{tabular}{cc cccccc}
        \toprule \toprule
        \#Params & Model & ImgNette-C & ImgWoof-C & ImgFruit-C & ImgYellow-C & ImgMeow-C & ImgSquawk-C \\
        \midrule
        \multirow{3}{*}{\makecell{491520\\(IPC=1)}} & TM & \underline{38.0} \small{$\pm1.6$} & \underline{23.8} \small{$\pm1.0$} & 22.7 \small{$\pm1.1$} & 35.6 \small{$\pm1.7$} & \underline{23.3} \small{$\pm1.1$} & - \\
                                                             & w/ IDC & 34.5 \small{$\pm0.6$} & 18.7 \small{$\pm0.4$} & \underline{28.5} \small{$\pm0.9$} & \underline{36.8} \small{$\pm1.4$} & 22.2 \small{$\pm1.2$} & \underline{26.8} \small{$\pm0.5$} \\
                                                             & \gc w/ FreD & \gc \textbf{51.2} \small{$\pm 0.6$} & \gc \textbf{31.0} \small{$\pm0.9$} & \gc \textbf{32.3} \small{$\pm1.4$} & \gc \textbf{48.2} \small{$\pm1.0$} & \gc \textbf{30.3} \small{$\pm0.3$} & \gc \textbf{45.9} \small{$\pm0.6$}\\
        \midrule
        \multirow{3}{*}{\makecell{4915200\\(IPC=10)}} & TM     & \underline{50.9} \small{$\pm0.7$} & \underline{30.9} \small{$\pm0.7$} & \underline{32.3} \small{$\pm0.8$} & \underline{45.6} \small{$\pm1.0$} & \underline{30.1} \small{$\pm0.5$} & \underline{44.4} \small{$\pm1.8$} \\
                                                               & w/ IDC & 40.4 \small{$\pm1.0$} & 21.9 \small{$\pm0.3$} & 32.2 \small{$\pm0.7$} & 39.6 \small{$\pm0.5$} & 23.9 \small{$\pm0.8$} & 40.5 \small{$\pm0.7$} \\
                                                               & \gc w/ FreD & \gc \textbf{55.2} \small{$\pm0.8$} & \gc \textbf{33.8} \small{$\pm0.8$} & \gc \textbf{35.7} \small{$\pm0.6$} & \gc \textbf{47.9} \small{$\pm0.4$} & \gc \textbf{31.3} \small{$\pm0.9$} & \gc \textbf{52.5} \small{$\pm 0.8$}\\
        \bottomrule \bottomrule
    \end{tabular}}
  \end{minipage}
\vspace{-0.3in}
\end{table}

\paragraph{3D Point Cloud Dataset.} As the spatial dimension of the data increases, the required dimension budget for each instance also grows exponentially. To validate the efficay of FreD on data with dimensions higher than 2D, we assess FreD on 3D point cloud data, 3D MNIST.\footnote{\url{https://www.kaggle.com/datasets/daavoo/3d-mnist}} Table \ref{experiments:tab:3D-MNIST} shows the test accuracies on the 3D MNIST dataset. FreD consistently achieves significant performance improvement over the baseline methods. This confirms the effectiveness of FreD in 2D image domain as well as 3D point cloud domain.
\vspace{-0.1in}

\paragraph{Robustness against Corruption.} Toward exploring the application ability of dataset distillation, we shed light on the robustness against the corruption of a trained synthetic dataset. We utilize the following test datasets: CIFAR-10.1 and CIFAR-10-C for CIFAR-10, ImageNet-Subset-C for ImagNet-Subset. For CIFAR-10.1 and CIFAR-10-C experiments, we utilize the off-the-shelf synthetic datasets which are released by the authors of each paper. We report the average test accuracies across 15 types of corruption and 5 severity levels for CIFAR-10-C and ImageNet-Subset-C. 

Figure \ref{experiments:fig:cifar10-c} and Table \ref{experiments:tab:imagenet-c} show the results of robustness on CIFAR-10-C and ImageNet-Subset-C, respectively. From both results, FreD shows the best performance over the whole setting which demonstrates the superior robustness against corruption. We want to note that IDC performs worse than the baseline in many ImageNet-Subset-C experiments (see Table \ref{experiments:tab:imagenet-c}) although it shows performance improvement on the ImageNet-Subset (see Table \ref{experiments:tab:imagenet}). On the other hand, FreD consistently shows significant performance improvement regardless of whether the test dataset is corrupted. It suggests that the frequency domain-based parameterization method shows higher domain generalization ability than the spatial domain-based parameterization method. Please refer to Appendix \ref{appendix:corruption} for the results of CIFAR-10.1 and detailed results based on corruption types of CIFAR-10-C.

To explain the rationale, corruptions that diminish the predictive ability of a machine learning model often occur at the high-frequency components. Adversarial attacks and texture-based corruptions are representative examples \cite{long2022frequency, yang2020patchattack}. Unlike FreD, which can selectively store information about an image's frequency distribution, transforms such as factorization or upsampling are well-known for not preserving frequency-based information well. Consequently, previous methods are likely to suffer a decline in predictive ability on datasets that retain class information while adding adversarial noise. In contrast, FreD demonstrates relatively good robustness against distribution shifts by successfully storing the core frequency components that significantly influence class recognition, regardless of the perturbations applied to individual data instances.
\vspace{-0.1in}

\paragraph{Collaboration with Other Parameterization.} 
\begin{wraptable}{r}{0.4\textwidth}
    \centering
    \caption{Test accuracies (\%) of each collaboration on CIFAR-10.} \label{experiments:tab:collaboration}
    \resizebox{0.35\textwidth}{!}{
    \begin{tabular}{c cc}
        \toprule \toprule
        IPC & 2 & 11 \\
        \#Params & 61.44k & 337.92k \\
        \midrule
        TM & 50.6 \small{$\pm1.0$} & 63.9 \small{$\pm0.3$}  \\
          w/ HaBa & 56.8 \small{$\pm0.4$} & 69.5 \small{$\pm0.3$} \\
          w/ IDC \& HaBa & \underline{61.3} \small{$\pm0.3$} & \underline{70.9} \small{$\pm0.4$} \\
        \gc w/ FreD \& HaBa & \gc \textbf{62.3} \small{$\pm0.1$} & \gc \textbf{72.9} \small{$\pm0.2$} \\
        \bottomrule \bottomrule
      \end{tabular}}
\end{wraptable}
The existing method either performs resolution resizing in the spatial domain or uses a neural network to change the dimension requirement of the spatial domain. On the other hand, FreD optimizes the coefficient of the frequency domain dimension and transforms it into the spatial domain through the inverse frequency transform. Therefore, FreD can be applied orthogonally to the existing spatial domain-based parameterization methods. Table \ref{experiments:tab:collaboration} shows the performance of different parameterizations applied to HaBa. From the results, we observed that FreD further enhances the performance of HaBa. Furthermore, it is noteworthy that the performance of HaBa integrated with FreD is higher than the combination of IDC and HaBa. These results imply that FreD can be well-integrated with spatial domain-based parameterization methods. 

\subsection{Ablation Studies} \label{exp:ablation}
\begin{wrapfigure}{r}{0.6\textwidth}
    \begin{subfigure}{0.495\linewidth}
        \centering
        \includegraphics[width=\textwidth]{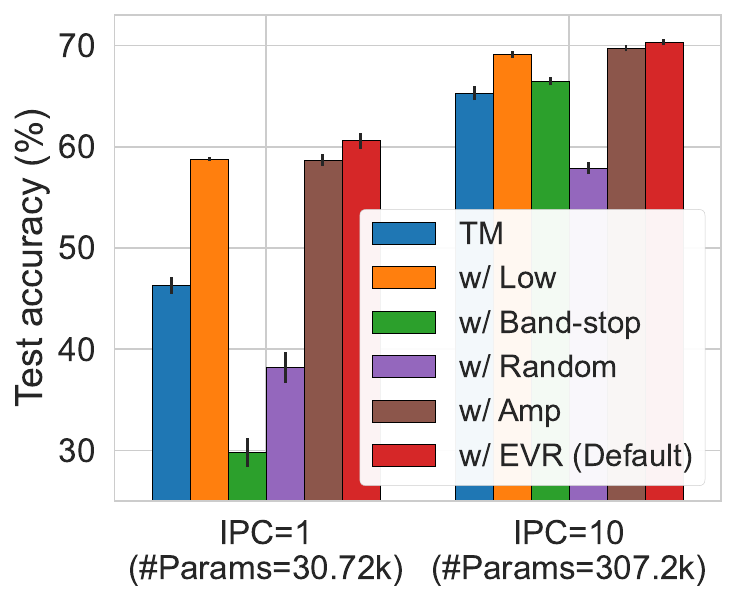}
        \caption{Binary mask $M$} \label{experiments:fig:ablationM}
    \end{subfigure}
    \begin{subfigure}{0.495\linewidth}
        \centering
        \includegraphics[width=\textwidth]{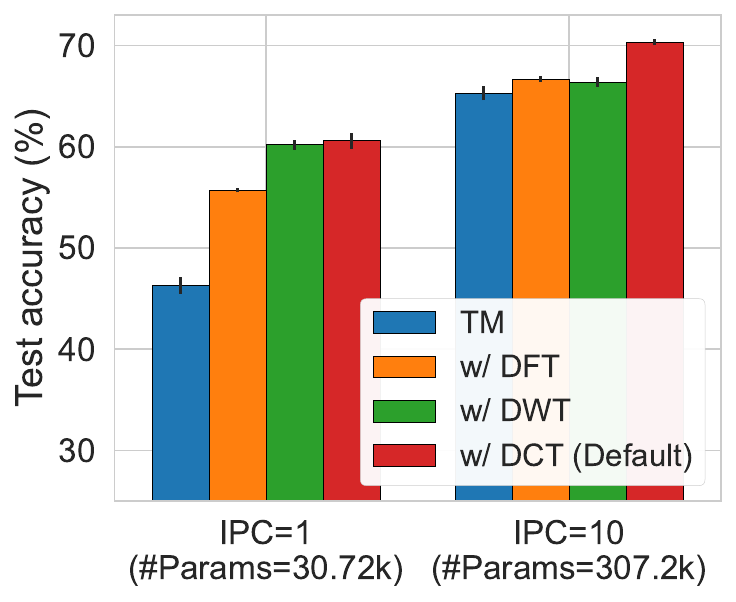}
        \caption{Frequency transform $\mathcal{F}$} \label{experiments:fig:ablationF}
    \end{subfigure}
\caption{Ablation studies on (a) the binary mask $M$, and (b) the frequency transform $\mathcal{F}$.} \label{experiments:fig:ablation}
\vspace{-0.1in}
\end{wrapfigure}
\paragraph{Effectiveness of Binary Mask $M$.} We conducted a comparison experiment to validate the explained variance ratio as a criterion for the selection of frequency dimensions. We selected the baselines for the ablation study as follows: Low-pass, Band-stop, High-pass, Random, and the magnitude of the amplitude in the frequency domain. We fixed the total budget and made $k$ the same. Figure \ref{experiments:fig:ablationM} illustrates the ablation study on different variations of criterion for constructing $M$. We skip the high-pass mask because of its low performance: 14.32\% in IPC=1 (\#Params=30.72k) and 17.11\% in IPC=10 (\#Params=307.2k). While Low-pass and Amplitude-based dimension selection also improves the performance of the baseline, EVR-based dimension selection consistently achieves the best performance.
\vspace{-0.1in}

\paragraph{Effectiveness of Frequency Transform $\mathcal{F}$.} We also conducted an ablation study on the frequency transform. Note that the FreD does not impose any constraints on the utilization of frequency transform. Therefore, we compared the performance of FreD when applying widely used frequency transforms such as the Discrete Cosine Transform (DCT), Discrete Fourier Transform (DFT), and Discrete Wavelet Transform (DWT). For DWT, we utilize the Haar wavelet function and low-pass filter instead of an EVR mask. As shown in Figure \ref{experiments:fig:ablationF}, we observe a significant performance improvement regardless of the frequency transform. Especially, DCT shows the highest performance improvement than other frequency transforms. Please refer to Appendix \ref{appendix:additional ablation FT} for additional experiments and detailed analysis of the ablation study on frequency transform.
\vspace{-0.1in}

\paragraph{Budget Allocation.} 
\begin{wrapfigure}{r}{0.4\textwidth}
    \centering
    \includegraphics[width=\linewidth]{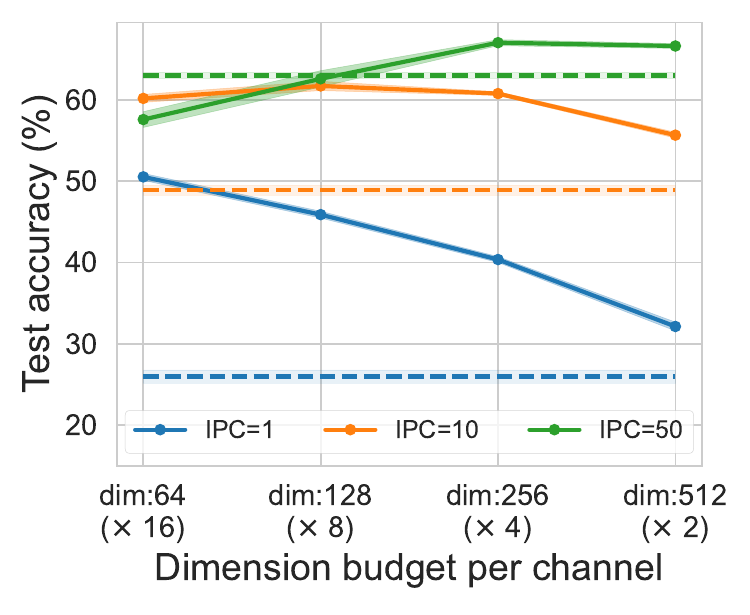}
    \caption{Ablation study on budget allocation of FreD with DM. Dashed line indicates the performance of DM.}
    \label{experiments:fig:budget allocation}
\end{wrapfigure}
Dataset distillation aims to include as much information from the original dataset as possible on a limited budget. FreD can increase the number of data $\lvert F \rvert$ by controlling the dimension budget per instance $k$, and FreD stores the frequency coefficients selected by EVR as $k$-dimensional vector. For example, with a small value of $k$, more data can be stored i.e. large $\lvert F \rvert$. This is a way to increase the quantity of instances while decreasing the quality of variance and reconstruction. By utilizing this flexible trade-off, we can pick the balanced point between quantity and quality to further increase the efficiency of our limited budget. Figure \ref{experiments:fig:budget allocation} shows the performance of selecting the dimension budget per channel under different budget situations. Note that, for smaller budgets i.e. IPC=1 (\#Params=30.72k), increasing $\lvert F \rvert$ performs better. For larger budget cases, such as IPC=50 (\#Params=1536k), allocating more dimensions to each instance performs better i.e. large $k$. This result shows that there is a trade-off between the quantity and the quality of data instances depending on the budget size.

\section{Conclusion}
This paper proposes a new parameterization methodology, FreD, that utilizes the augmented frequency domain. FreD selectively utilizes a set of dimensions with a high variance ratio in the frequency domain, and FreD only optimizes the frequency representations of the corresponding dimensions in the junction with the frequency transform. Based on the various experiments conducted on benchmark datasets, the results demonstrate the efficacy of utilizing the frequency domain in dataset distillation. Please refer to Appendix \ref{appendix:limitation} for the limitation of the frequency domain-based dataset distillation.

\section*{Acknowledgement}
This work was supported by the National Research Foundation of Korea (NRF) grant funded by the Korea government (MSIT) (No.2021R1A2C200981613). ※ MSIT: Ministry of Science and ICT

\hypersetup{urlcolor = black}

\hypersetup{urlcolor = magenta}

\clearpage
\begin{appendices}

\section{Literature Reviews on Related Works}
\subsection{Dataset Distillation} \label{appendix:literature reviews}
In this section, we briefly review the methodology of constructing $S$ as an input-sized vector and provide a detailed review of our main comparative methods, HaBa \cite{liu2022dataset}, IDC \cite{kim2022dataset} and GLaD \cite{cazenavette2023generalizing}.
\paragraph{Input-sized Parameterization.} Dataset Distillation (DD) \cite{wang2018dataset} aims at finding the synthetic dataset $S$ with a bi-level optimization. The main idea of bi-level optimization is that the network parameter $\theta_{S}$, which is trained on $S$, minimizes the (population) risk of the original dataset $D$. Dataset Condensation (DC) \cite{zhao2020dataset} introduces a proxy objective, which aims at matching the layer-wise gradients of a network over the optimization path of $S$. Differentiable Siamese Augmentation (DSA) \cite{zhao2021dataset} applies the differentiable and identical data augmentation to original data instances and synthetic data instances at each training step. Contrary to gradient matching i.e. short-range trajectory matching \cite{zhao2020dataset}, Trajectory Matching (TM) aims at transferring the knowledge of long-range trajectory from pre-trained with the original dataset. It minimizes the difference between the training trajectory on synthetic data and the training trajectory on real data. Distribution Matching (DM) \cite{zhao2023dataset} points out the computation cost of precedents. Therefore, the authors propose a new objective that aims at aligning the feature distributions of both the original dataset and the synthetic dataset within sampled embedding spaces. CAFE \cite{wang2022cafe} extends the DM by layer-wise feature matching. Kernel Inducing Point (KIP) \cite{nguyen2021dataset} introduces a kernel-based objective that leverages infinite-width neural networks. It optimizes to let condensed datasets be kernel inducing points in kernel ridge regression. FRePo \cite{zhou2022dataset} points out the meta-gradient computation and overfitting in dataset distillation. FRePo overcomes these challenges by utilizing the truncated backpropagation through time and model pool.
\paragraph{HaBa \cite{liu2022dataset}.} HaBa proposed a technique for dataset factorization, which involves breaking the synthetic dataset into bases and hallucinator networks. The hallucinator takes bases as input and generates image instances. By learning bases and hallucinators, the resulting model could produce more diverse samples based on the available budget. However, incorporating an additional network in distillation requires a separate budget, which is distinct from the data instances. For example, HaBa does not perform when 1 image per class setting, although using light-weight hallucinators. Furthermore, joint learning of both the network and data usually leads to instability in the training.
\paragraph{IDC \cite{kim2022dataset}.} IDC configures the synthetic dataset as several reduced-size of spatial images and utilizes the differentiable multi-formation function to restore to the original size. The usual choice of multi-formation function is an upsampling module, which does not require training. Therefore, efficient parameterization enables the increasing the available number of data instances. However, the compression process still takes place in the spatial domain, leading to the loss of information and inefficient utilization of the budget. Additionally, \cite{durall2020watch} empirically showed that upsampling methods cause distortion in the spectral distribution of natural images. 
\paragraph{GLaD \cite{cazenavette2023generalizing}.} GLaD employs a pre-trained generative model and distills the synthetic dataset in the latent space of the generative model, such as Generative Adversarial Networks (GAN). By leveraging the generative model, GLaD encourages better generalization to unseen architecture and scale to high-dimensional datasets. However, generative models typically require a large number of parameters, which introduces several inefficiencies. Firstly, storing a generative model with many parameters is burdensome in dataset distillation which is budget-constrained. Due to the budget constraint, GLaD proposes to use a generative model for training, and at the end of training, create a distilled instance by combining the distilled latent code and generative model to create a distilled instance and store it in the budget. This eliminates the need to allocate a budget for the generative model, but the amount of budget occupied by one distilled instance is the same as the input-sized parameterization method. As a result, it creates the same number of instances as the input-sized parameterization method, which makes the synthetic dataset insufficiently expressive. There is also a computational inefficiency because it takes more time to move forward and backward due to the large number of parameters. Finally, Frequency transform is dataset agnostic, while the deep generative model needs to apply a suitable structure to the dataset. This has the inefficiency of selecting the appropriate structure based on the dataset.

\subsection{Frequency Transform} \label{appendix:survery_Freq}
\paragraph{Additional Review of Frequency Transform.} As mentioned in the main paper, the form of the frequency transform depends on the selection of the basis function $\phi(a,b,u,v)$ (see Eq. \eqref{eq:ft} in the main paper). Discrete Cosine Transform (DCT) uses the cosine function as the basis function i.e. $\phi(a,b,u,v)=cos\bigl(\frac{\pi}{d_1}(a+\frac{1}{2}u)\bigr)cos\bigl(\frac{\pi}{d_2}(b+\frac{1}{2}v)\bigr)$. Discrete Fourier Transform (DFT) utilizes the exponential function as the basis function i.e. $\phi(a,b,u,v)=e^{-i2\pi(\frac{ua}{d_1}+\frac{vb}{d_2})}$. Discrete Wavelet Transform (DWT) employs the wavelet function, such as the Haar wavelet function or the Daubechies wavelet. In the case of images with multiple color channels, both frequency transform and inverse frequency transform can be independently applied to each channel. There are various research areas in machine learning, which use the property of frequency domain. In the following paragraphs, we review research conducted in the direction of utilizing the property of frequency domain such as adversarial attacks and analyze the neural network.
\paragraph{Adversarial Attack.} Recently, in adversarial attack areas, there has been a discussion suggesting that attacks in the frequency domain exhibit higher effectiveness compared to attacks based in the spatial domain \cite{guo2018low,sharma2019effectiveness}. In \cite{guo2018low}, the authors propose a method to constrain the search space of adversarial attacks to the low-frequency domain. This method consistently reduces the black-box attack's query cost. Furthermore, the authors of \cite{sharma2019effectiveness} show empirical evidence of the effectiveness of the low-frequency attack.
\paragraph{Analyzing Neural Network.} There are a bunch of studies that analyze neural networks in terms of frequency transforms. Spectral bias in deep neural networks \cite{spectralbias,spectralbias2} is a well-known problem in machine learning, which describes the tendency of the network to prefer specific frequency components over other components while training. The presence of spectral bias in a deep neural network can have a significant impact on its ability to generalize to new data instances by restricting its ability to capture crucial patterns or features for a given task \cite{spectralbias2,spectralbias3}. To prevent such biased training, \cite{frequencyinput1,frequencyinput2} designed a network and the corresponding loss function that takes transformed values in the frequency domain as input. To prevent spectral bias during the training, \cite{frequencylearning1,frequencylearning2} introduced frequency-based regularization techniques, while \cite{frequencyaugment1} proposed augmentation methods based on the frequency domain.

\section{Proofs of Theoretical Evidences} \label{appendix:proofs}
\subsection{Proof of Proposition 1}
\proposition*
\begin{proof}
    Mathematically, let $X$ be the $d$-dimensional dataset with $n$ samples in the domain $A$, and let $X_B$ be the transformed dataset in the domain $B$. Let $S \subseteq \{1, 2, ..., d\}$ be the subset of dimensions for which we want to calculate the sum of explained variance ratio. Then, the sum of explained variance ratio for $S$ in the domain $A$ is given by:
    \begin{equation}
    R_A(S) = \frac{\sum_{i \in S}\lambda_i}{\sum_{i=1}^d\lambda_i}
    \end{equation}
    where $\lambda_i$ is the eigenvalues of $i$-th dimension of the covariance matrix of $X$. As noted in the assumption, bijective function $W$ exists to transform the $X$ into $X_B$, i.e. $X_B = WX$. We can write the covariance matrix of $X_B$ as:
    \begin{equation}
    \Sigma_B = \frac{1}{n}X_B X_B^T = \frac{1}{n}(WX)(WX)^T = W(\frac{1}{n}XX^T)W^T = W\Sigma_A W^T
    \end{equation}
    where $\Sigma_A$ is the covariance matrix of $X$ in the domain $A$. Having said that, let $S_B = \{j \,|\, j = W(i), i \in S\}$ be the corresponding subset of dimensions in the domain $B$. It should be noted that each element in $S_B$ do not have to be one-hot dimension. Also, in a linear bijective transformation, orthogonality in the original space is preserved in the transformed space. Then, the sum of explained variance ratios for the dimension subset, $S_B$, in the domain $B$ is given as follows:
    \begin{equation}
    R_B(S_B) = \frac{\sum_{j^{'} \in S_B}\lambda_j^{'}}{\sum_{j=1}^d\lambda_j}
    \end{equation}
    where $\lambda_j$ is an eigenvalue of $j$-th dimension of the covariance matrix of $X_B$. Now, we can show that $R_A(S) = R_B(S_B)$ as follows:
    \begin{equation}
    R_B(S_B) = \frac{\sum_{j \in S_B}\lambda_j}{\sum_{j=1}^d\lambda_j} = \frac{\sum_{j \in S_B}\lambda_{W(i)}}{\sum_{j=1}^d\lambda_{W(i)}} = \frac{\sum_{i \in S}\lambda_i}{\sum_{i=1}^d\lambda_i} = R_A(S)
    \end{equation}
    where we used the fact that the eigenvalues of the covariance matrix are the same for $X$ and $X_B$ (i.e., $\lambda_i = \lambda_{W(j)}$ for all $i,j$), and the fact that the sum of eigenvalues is invariant under bijective linear transformation.
\end{proof}
Therefore, we have shown that the sum of the explained variance ratio for a subset of dimensions in the domain $A$ is the same as the explained variance ratio sum for the domain $B$ when transforming the domain $A$ dataset to the domain $B$ using only that subset of dimensions. We re-arrange the claim as follows: The sum of explained variance ratios of a masked dataset for a specific dimension subset remains preserved even under linearly bijective transformations between domains.

\subsection{Proof of Corollary 1}
\corollary*
\begin{proof}
    In Proposition 1, we proved that the sum of explained variance ratios of a masked dataset for a specific dimension subset remains preserved even under linearly bijective transformations between domains. Having said that, $W_{V^{*}_{C,k}}X_{C}$ is a transformed dataset of $X_{C}$ from domain $C$ to domain $A$, where only the top-$k$ dimensions that maximize the sum of explained variance ratios are utilized. Proposition 1 states that this transformation preserves the sum of explained variance ratios, and since $\eta^{*}_{B,k} \geq \eta^{*}_{C,k}$, the sum of explained variance ratios is preserved even in terms of the relative magnitude between the explained variance ratios sum of the transformed datasets.
\end{proof}

\section{Experimental Details} \label{appendix:settings}
\subsection{Dataset}
In this paper, we evaluate FreD on a variety of benchmark datasets, including those widely used in dataset distillation.
\begin{itemize}
  \item MNIST \cite{lecun1998gradient} is a handwritten digit image dataset with 60,000 images for training and 10,000 images for testing. Each image is a $28\times28$ gray-scale image and is categorized into 10 classes (digits from 0 to 9). 
  \item Fashion MNIST \cite{xiao2017/online} contains various fashion items images such as clothing and shoe. It consists of a training set of 60,000 grayscale images and a test set of 10,000 images. Each image has a $28\times28$ size. Fashion MNIST has 10 classes in total.
  \item SVHN \cite{netzer2011reading} is a real-world digit image dataset with 73,257 images for training and 26,032 images for testing. Each image in the dataset is a $32\times32$ RGB image and belongs to one of 10 classes ranging from 0 to 9. 
  \item CIFAR-10 \cite{krizhevsky2009learning} consists of $32\times32$ RGB images with 50,000 images for training and 10,000 images for testing. It has 10 classes in total and each class contains 5,000 images for training and 1,000 images for testing.
  \item CIFAR-100 \cite{krizhevsky2009learning} comprises a total of 60,000 $32\times32$ RGB images distributed across 100 classes. Within each class, 500 images are allocated for training, while 100 images are for testing. These 100 classes are further grouped into 20 superclasses, with each superclass consisting of 5 classes at a more specific level.
  \item 3D MNIST (\url{https://www.kaggle.com/datasets/daavoo/3d-mnist}) consists of 10,000 training data and 1,000 test data, and each data has $1\times16\times16\times16$ size. Each data instance is categorized into 10 classes. 
  \item Tiny-ImageNet \cite{le2015tiny} is a downsampled subset of ImageNet \cite{deng2009imagenet} to a size of $64\times64$. This dataset consists of 200 classes and each class contains 500 images for training and 100 images for testing.
  \item ImageNet-Subset is a dataset consisting of a subset of similar features in the ImageNet \cite{deng2009imagenet}. By following the previous work, we consider diverse types of subsets: ImageNette (various objects)\cite{howard2019smaller}, ImageWoof (dog breeds)\cite{howard2019smaller}, ImageFruit (fruits) \cite{cazenavette2022dataset}, ImageMeow (cats) \cite{cazenavette2022dataset}, ImageSquawk (birds) \cite{cazenavette2022dataset}, ImageYellow (yellowish things) \cite{cazenavette2022dataset}, and ImageNet-[A, B, C, D, E] (based on ResNet50 performance) \cite{cazenavette2023generalizing}. Each subset has 10 classes. We consider two types of resolution: $128\times128$ and $256\times256$.
  \item LSUN \cite{yu2015lsun} aims at understanding the large-scale scene images. The original LSUN dataset has 10 classes and each class contains a large number of images, ranging from 120k to 3,000k for training. We consider two datasets, coined as LSUN-10k/LSUN-25k, which randomly sampled 10k/25k instances per class which resulted in a total 100k/250k instances, respectively. We also downsize each instance to a $128 \times 128$ size. 
  \item CIFAR-10.1 \cite{recht2018cifar} consists of 2,000 new test images which have same classes as CIFAR-10. 
  \item CIFAR-10-C and ImageNet-C \cite{hendrycks2019benchmarking} aim at measuring the robustness of object recognition based on CIFAR-10 and ImageNet, respectively. They have 15 types of corruption and each corruption has five levels with level 5 indicating the most severest. We create ImageNet-Subset-C by selecting data from ImageNet-C that matches the ImageNet-Subset classes.
\end{itemize}

\subsection{Architecture}
For 2D image datasets, we basically employ an $n$-depth convolutional neural network, coined ConvNetD$n$, by following the previous works. The ConvNetD$n$ has $n$ duplicate blocks, which consist of a convolution layer with $3\times3$-shape 128 filters, an instance normalization layer \cite{ulyanov2016instance}, ReLU, and an average pooling with $2\times2$ kernel size with stride 2. After the convolution blocks, a linear classifier outputs the logits. We utilize a different number of blocks depending on the resolution: ConvNetD3 for $28\times28$ and $32\times32$, ConvNetD4 for $64\times64$, ConvNetD5 for $128\times128$ and ConvNetD6 for $256\times256$. For the performance comparison for different test network architectures, we also follow the precedent: ResNet \cite{he2016deep}, VGG \cite{simonyan2014very}, AlexNet \cite{krizhevsky2017imagenet}, and ViT \cite{dosovitskiy2010image}.

For the 3D point cloud dataset; 3D MNIST, we implement a 3D version of ConvNet, coined Conv3DNet. Similarly, Conv3DNet has three duplicate blocks; a convolution layer with $3\times3\times3$-shape 64 filters, a 3D instance normalization, ReLU, and a 3D average pooling with $2\times2\times2$ with stride 2. A linear layer follows these convolution blocks.

\subsection{Implementation Configurations} \label{appendix:implementation}
We use trajectory matching objective (TM) \cite{cazenavette2022dataset} for $\mathcal{L}_{DD}$ as a default although FreD can use any dataset distillation loss. Similarly, we utilize Discrete Cosine Transform (DCT) as a default frequency transform $\mathcal{F}$. For the implementation of frequency transform, we utilize the open-source PyTorch library; \textbf{torch-dct} (\url{https://github.com/zh217/torch-dct}) for DCT and \textbf{pytorch\_wavelets} (\url{https://github.com/fbcotter/pytorch\_wavelets}) for Discrete Wavelet Transform (DWT). We utilize the built-in function of PyTorch for the Discrete Fourier Transform (DFT). We separately apply the frequency transform to each channel for RGB image datasets. We use an SGD optimizer with a momentum rate of 0.5 for all our experiments. Each experiment is trained with 15,000 iterations. Contrary to previous research \cite{cazenavette2022dataset,liu2022dataset}, FreD does not use the ZCA Whitening. We used four RTX 3090 GPUs by default and two Tesla A100 GPUs for CIFAR-100, Tiny-ImageNet, and ImageNet-Subset. We basically follow the evaluation protocol of the previous works \cite{zhao2020dataset,zhao2023dataset,cazenavette2022dataset}. We evaluate each method by training 5 randomly initialized networks from scratch on optimized $S$. We provide the detailed hyper-parameters in Table \ref{appendix:tab:hyperparams} (see the end of Appendix).

\section{Additional Experimental Results} \label{appendix:additional results}
\subsection{Performance Comparison on Low-dimensional Datasets}
We evaluate our proposed method on low-dimensional datasets ($\leq 64 \times 64$ resolution) such as MNIST, Fashion MNIST, and Tiny-ImageNet. Table \ref{appendix:tab:low-dimensional dataset} shows that FreD achieves improved or competitive performances in most experimental settings. These results repeatedly support our conjecture: the utilization of the frequency domain yields beneficial outcomes in terms of enhancing performance.

FreD's motivation lies in leveraging select important dimensions of the frequency domain, which can contain much of the spatial domain's information, to utilize the given memory budget more efficiently. This efficiency manifests greater utility when the available memory budget is more limited. Through extensive experiments results, FreD demonstrates more substantial performance improvement in most experiments with an IPC=1 setting. TinyImageNet is originally a dataset with 500 instances per class, and the IPC=50 setting for this dataset could be considered a not-so-drastic reduction. In situations where such a significant reduction doesn't occur, FreD's motivation may be weakened. Excluding this particular setting, FreD consistently demonstrates performance improvement compared to the baseline across evaluations.
\begin{table}[h!] 
    \vspace{-0.15in}
    \centering
    \caption{Test accuracies (\%) on MNIST, Fashion MNIST, and Tiny-ImageNet. The best results and the second-best result are highlighted in \textbf{bold} and \underline{underline}, respectively. Note that IDC does not provide the standard deviation on MNIST and Fashion MNIST experiments in the original paper.} \label{appendix:tab:low-dimensional dataset}
    \adjustbox{max width=\textwidth}{%
    \begin{tabular}{c c cc cc ccc}
    \toprule \toprule
     & & \multicolumn{2}{c}{MNIST} & \multicolumn{2}{c}{Fashion MNIST} & \multicolumn{3}{c}{Tiny-ImageNet} \\
    \cmidrule(lr){3-4} \cmidrule(lr){5-6} \cmidrule(lr){7-9} 
     & IPC & 1 & 10 & 1 & 10 & 1 & 10 & 50 \\
     & \#Params & 7.84k & 78.4k & 7.84k & 78.4k & 2457.6k & 24576k & 122880k \\
    \midrule
    \multirow{2}{*}{Coreset} & Random & 64.9 \small{$\pm3.5$} & 95.1\small{$\pm0.9$} & 51.4 \small{$\pm3.8$} & 73.8 \small{$\pm0.7$} & 1.4 \small{$\pm0.1$} & 5.0 \small{$\pm0.2$} & 15.0 \small{$\pm0.4$} \\
                             & Herding & 89.2 \small{$\pm1.6$} & 93.7 \small{$\pm0.3$} & 67.0 \small{$\pm1.9$} & 71.1 \small{$\pm0.7$} & 2.8 \small{$\pm0.2$} & 6.3 \small{$\pm0.2$} & 16.7 \small{$\pm0.3$} \\
    \midrule
    \multirow{7}{*}{\makecell{Input-sized\\parameterization}} & DC & 91.7 \small{$\pm0.5$} & 97.4 \small{$\pm0.2$} & 70.5 \small{$\pm0.6$} & 82.3 \small{$\pm0.4$} & - & - & - \\
     & DSA & 88.7 \small{$\pm0.6$} & 97.8 \small{$\pm0.1$} & 70.6 \small{$\pm0.6$} & 84.6 \small{$\pm0.3$} & - & - & - \\
     & DM & 89.7 \small{$\pm0.6$} & 97.5 \small{$\pm0.1$} & - & - & 3.9 \small{$\pm0.2$} & 12.9 \small{$\pm0.4$} & 24.1 \small{$\pm0.2$} \\
     & CAFE+DSA & 90.8 \small{$\pm0.5$} & 97.5 \small{$\pm0.1$} & 73.7 \small{$\pm0.7$} & 83.0 \small{$\pm0.3$} & - & - & - \\
     & TM & 88.7 \small{$\pm1.0$} & 96.6 \small{$\pm0.4$} & 75.7 \small{$\pm1.5$} & \underline{88.4} \small{$\pm0.4$} & 8.8 \small{$\pm0.3$} & 23.2 \small{$\pm0.2$} & \textbf{28.0} \small{$\pm0.2$} \\
     & KIP & 90.1 \small{$\pm0.1$} & 87.5 \small{$\pm0.0$} & 73.5 \small{$\pm0.5$} & 86.8 \small{$\pm0.1$} & - & - & - \\
     & FRePo & 93.0 \small{$\pm0.4$} & \textbf{98.6} \small{$\pm0.1$} & 75.6 \small{$\pm0.3$} & 86.2 \small{$\pm0.2$} & \underline{15.4} \small{$\pm0.3$} & \textbf{25.4} \small{$\pm0.2$} & - \\
    \midrule
    \multirow{3}{*}{Parameterization} & IDC & \underline{94.2} & \underline{98.4} & \underline{81.0} & 86.0 & - & - & - \\
     & HaBa & 92.4 \small{$\pm0.4$} & 97.4 \small{$\pm0.2$} & 80.9 \small{$\pm0.7$} & \underline{88.6} \small{$\pm0.2$} & - & - & - \\
     & \gc FreD & \gc \textbf{95.8} \small{$\pm0.2$} & \gc 97.6 \small{$\pm0.8$} & \gc \textbf{84.6} \small{$\pm0.2$} & \gc \textbf{89.1} \small{$\pm0.2$} & \gc \textbf{19.2} \small{$\pm0.4$} & \gc \underline{24.2} \small{$\pm0.4$} & \gc \underline{26.4} \small{$\pm0.4$} \\
    \midrule
    \multicolumn{2}{c}{Entire original dataset} & \multicolumn{2}{c}{99.6 \small{$\pm0.0$}} & \multicolumn{2}{c}{93.5  \small{$\pm0.1$}} & \multicolumn{3}{c}{37.6 \small{$\pm0.4$}} \\
    \bottomrule \bottomrule
    \end{tabular}}
\end{table}

\subsection{Performance Comparison on High-dimensional Datasets}
We further evaluate our proposed method on high-dimensional datasets ($\geq 128 \times 128$ resolution). Table \ref{appendix:tab:imagenet128_1} and \ref{appendix:tab:imagenet128_2} present the results of extensive experiments on $128\times 128$ resolution ImageNet-Subset. As in the case of low-dimensional datasets, FreD consistently achieves the highest performance improvement among the parameterization methods in most experimental settings. Since the performance of FreD at IPC=10 (\#Params=4915.2k) already overwhelms the performance of HaBa of IPC=11 (\#Params=5406.72k), we did not conduct the experiment of FreD on IPC=11 (\#Params=5406.72k). Furthermore, in Table \ref{appendix:tab:imagenet256}, FreD repeatedly shows better performance on $256\times 256$ resolution ImageNet-Subset.

It should be noted that FreD significantly improves the performance of cross-architecture generalization. For instance, GLaD also improves cross-architecture performance, but it shows the performance degradation in the architecture used for training when the utilized dataset distillation loss is TM. On the other hand, FreD shows the best performance in all experiments. It means that FreD provides insight into how well the frequency domain-based parameterization method understands the task, rather than overfitting to a particular architecture.

In summary, these extensive experimental results continuously demonstrate the efficacy of utilizing the frequency domain in dataset distillation regardless of the image's resolution.
\begin{table}[h!]
    \vspace{-0.1in}
    \centering
    \caption{Test accuracies (\%) on ImageNet-Subset (Image-[Nette, Woof, Fruit, Yellow, Meow, Squawk], $128 \times 128$). Note that HaBa is structurally disabled to experiment in IPC=1 (\#Params=491.52k) due to the nature of its methodology.} \label{appendix:tab:imagenet128_1}
    \adjustbox{max width=\textwidth}{%
    \begin{tabular}{cc cccccc}
    \toprule \toprule
    \#Params & Model & ImageNette & ImageWoof & ImageFruit & ImageYellow & ImageMeow & ImageSquawk \\
    \midrule
    \multirow{4}{*}{\makecell{491.52k\\(IPC=1)}} & TM     & 47.7 \small{$\pm0.9$} & 28.6 \small{$\pm0.8$} & 26.6 \small{$\pm0.8$} & 45.2 \small{$\pm0.8$} & 30.7 \small{$\pm1.6$} & 39.4 \small{$\pm1.5$} \\
                                                 & w/ IDC & \underline{61.4} \small{$\pm1.0$} & \underline{34.5} \small{$\pm1.1$} & \underline{38.0} \small{$\pm1.1$} & \underline{56.5} \small{$\pm1.8$} & \underline{39.5} \small{$\pm1.5$} & \underline{50.2} \small{$\pm1.5$} \\
                                                 & w/ HaBa & - & - & - & - & - & - \\
                                                 & \gc w/ FreD & \gc \textbf{66.8} \small{$\pm0.4$} & \gc \textbf{38.3} \small{$\pm1.5$} & \gc \textbf{43.7} \small{$\pm1.6$} & \gc \textbf{63.2} \small{$\pm1.0$} & \gc \textbf{43.2} \small{$\pm0.8$} & \gc \textbf{57.0} \small{$\pm0.8$} \\
    \midrule
    \multirow{4}{*}{\makecell{983.04k\\(IPC=2)}} & TM & 55.2 \small{$\pm1.1$} & 30.9 \small{$\pm1.3$} & 31.6 \small{$\pm1.6$} & 49.7 \small{$\pm1.4$} & 35.3 \small{$\pm2.2$} & 43.9 \small{$\pm0.6$} \\
                                                 & w/ IDC & \underline{65.4} \small{$\pm1.2$} & \underline{37.6} \small{$\pm1.6$} & \underline{43.0} \small{$\pm1.5$} & \underline{62.4} \small{$\pm1.7$} & \underline{43.1} \small{$\pm1.2$} & \underline{55.5} \small{$\pm1.2$} \\
                                                 & w/ HaBa & 51.9 \small{$\pm1.7$} & 32.4 \small{$\pm0.7$} & 34.7 \small{$\pm1.1$} & 50.4 \small{$\pm1.6$} & 36.9 \small{$\pm0.9$} & 41.9 \small{$\pm1.4$} \\
                                                 & \gc w/ FreD & \gc \textbf{69.0} \small{$\pm0.9$} & \gc \textbf{40.0} \small{$\pm1.4$} & \gc \textbf{46.3} \small{$\pm1.2$} & \gc \textbf{66.3} \small{$\pm1.1$} & \gc \textbf{45.2} \small{$\pm1.7$} & \gc \textbf{62.0} \small{$\pm1.3$} \\
    \midrule
    \multirow{4}{*}{\makecell{4915.2k\\(IPC=10)}} & TM & 63.0 \small{$\pm 1.3$} & 35.8 \small{$\pm 1.8$} & 40.3 \small{$\pm 1.3$} & 60.0 \small{$\pm 1.5$} & 40.4 \small{$\pm 2.2$} & 52.3 \small{$\pm 1.0$} \\
                                                  & w/ IDC & \underline{70.8} \small{$\pm0.5$} & \underline{39.8} \small{$\pm0.9$} & \underline{46.3} \small{$\pm1.4$} & \underline{68.7} \small{$\pm0.8$} & \underline{47.9} \small{$\pm1.4$} & \underline{65.4} \small{$\pm1.2$} \\
                                                  & w/ HaBa & - & - & - & - & - & - \\
                                                  & \gc w/ FreD & \gc \textbf{72.0} \small{$\pm0.8$} & \gc \textbf{41.3} \small{$\pm1.2$} & \gc \textbf{47.0} \small{$\pm1.1$} & \gc \textbf{69.2} \small{$\pm 0.6$} & \gc \textbf{48.6} \small{$\pm0.4$} & \gc \textbf{67.3} \small{$\pm0.8$} \\
    \midrule
    \multirow{2}{*}{\makecell{5406.72k\\(IPC=11)}} & TM & 63.9 \small{$\pm0.5$} & 36.6 \small{$\pm0.8$} & 40.1 \small{$\pm1.9$} & 60.4 \small{$\pm1.5$} & 41.0 \small{$\pm1.5$} & 54.6 \small{$\pm1.0$} \\
                                                   & w/ HaBa & 64.7 \small{$\pm1.6$} & 38.6 \small{$\pm1.3$} & 42.5 \small{$\pm1.6$} & 63.0 \small{$\pm1.6$} & 42.9 \small{$\pm0.9$} & 56.8 \small{$\pm1.0$} \\                                             
    \bottomrule \bottomrule
  \end{tabular}}
\end{table}
\begin{table}[h!]
    \centering
    \caption{Test accuracies (\%) on ImageNet-Subset (ImageNet-[A, B, C, D, E], $128 \times 128$) with IPC=1 (\#Params=491.52k). "Cross" denotes the average test accuracy of trained AlexNet, VGG11, ResNet18, and ViT on each synthetic dataset.} \label{appendix:tab:imagenet128_2}
    \adjustbox{max width=\textwidth}{%
    \begin{tabular}{c cc cc cc cc cc}
    \toprule \toprule
     & \multicolumn{2}{c}{ImageNet-A} & \multicolumn{2}{c}{ImageNet-B} & \multicolumn{2}{c}{ImageNet-C} & \multicolumn{2}{c}{ImageNet-D} & \multicolumn{2}{c}{ImageNet-E} \\
    \cmidrule(lr){2-3} \cmidrule(lr){4-5} \cmidrule(lr){6-7} \cmidrule(lr){8-9} \cmidrule(lr){10-11}
     & ConvNet & Cross & ConvNet & Cross & ConvNet & Cross & ConvNet & Cross & ConvNet & Cross \\
    \midrule
    DC      & 43.2 \small{$\pm0.6$} & 38.7 \small{$\pm4.2$} & 47.2 \small{$\pm0.7$} & 38.7 \small{$\pm1.0$} & 41.3 \small{$\pm0.7$} & 33.3 \small{$\pm1.9$} & 34.3 \small{$\pm1.5$} & 26.4 \small{$\pm1.1$} & 34.9 \small{$\pm1.5$} & 27.4 \small{$\pm0.9$} \\
    w/ GLaD & \underline{44.1} \small{$\pm2.4$} & \underline{41.8} \small{$\pm1.7$} & \underline{49.2} \small{$\pm1.1$} & \underline{42.1} \small{$\pm1.2$} & \underline{42.0} \small{$\pm0.6$} & \underline{35.8} \small{$\pm1.4$} & \underline{35.6} \small{$\pm0.9$} & \underline{28.0} \small{$\pm0.8$} & \underline{35.8} \small{$\pm0.9$} & \underline{29.3} \small{$\pm1.3$} \\
    \gc w/ FreD & \gc \textbf{53.1} \small{$\pm1.0$} & \gc \textbf{48.0} \small{$\pm1.4$} & \gc \textbf{54.8} \small{$\pm1.2$} & \gc \textbf{47.6} \small{$\pm1.5$} & \gc \textbf{54.2} \small{$\pm1.2$} & \gc \textbf{47.8} \small{$\pm1.2$} & \gc \textbf{42.8} \small{$\pm1.1$} & \gc \textbf{36.3} \small{$\pm1.4$} & \gc \textbf{41.0} \small{$\pm1.1$} & \gc \textbf{35.0} \small{$\pm1.1$} \\
    \midrule
    
    DM      & 39.4 \small{$\pm1.8$} & 27.2 \small{$\pm1.2$} & 40.9 \small{$\pm1.7$} & 24.4 \small{$\pm1.1$} & 39.0 \small{$\pm1.3$} & 23.0 \small{$\pm1.4$} & 30.8 \small{$\pm0.9$} & 18.4 \small{$\pm0.7$} & 27.0 \small{$\pm0.8$} & 17.7 \small{$\pm0.9$} \\
    w/ GLaD & \underline{41.0} \small{$\pm1.5$} & \underline{31.6} \small{$\pm1.4$} & \underline{42.9} \small{$\pm1.9$} & \underline{31.3} \small{$\pm3.9$} & \underline{39.4} \small{$\pm0.7$} & \underline{26.9} \small{$\pm1.2$} & \underline{33.2} \small{$\pm1.4$} & \underline{21.5} \small{$\pm1.0$} & \underline{30.3} \small{$\pm1.3$} & \underline{20.4} \small{$\pm0.8$} \\
    \gc w/ FreD & \gc \textbf{58.0} \small{$\pm1.7$} & \gc \textbf{48.7} \small{$\pm1.5$} & \gc \textbf{58.6} \small{$\pm1.3$} & \gc \textbf{47.5} \small{$\pm1.5$} & \gc \textbf{55.6} \small{$\pm1.4$} & \gc \textbf{47.1} \small{$\pm1.0$} & \gc \textbf{46.3} \small{$\pm1.2$} & \gc \textbf{35.9} \small{$\pm2.0$} & \gc \textbf{45.0} \small{$\pm1.8$} & \gc \textbf{32.1} \small{$\pm1.6$} \\
    \midrule

    TM      & \underline{51.7} \small{$\pm0.2$} & 33.4 \small{$\pm1.5$} & \underline{53.3} \small{$\pm1.0$} & 34.0 \small{$\pm3.4$} & \underline{48.0} \small{$\pm0.7$} & 31.4 \small{$\pm3.4$} & \underline{43.0} \small{$\pm0.6$} & 27.7 \small{$\pm2.7$} & \underline{39.5} \small{$\pm0.9$} & 24.9 \small{$\pm1.8$} \\
    w/ GLaD & 50.7 \small{$\pm0.4$} & \underline{39.9} \small{$\pm1.2$} & 51.9 \small{$\pm1.3$} & \underline{39.4} \small{$\pm1.3$} & 44.9 \small{$\pm0.4$} & \underline{34.9} \small{$\pm1.1$} & 39.9 \small{$\pm1.7$} & \underline{30.4} \small{$\pm1.5$} & 37.6 \small{$\pm0.7$} & \underline{29.0} \small{$\pm1.1$} \\
    \gc w/ FreD & \gc \textbf{67.7} \small{$\pm1.0$} & \gc \textbf{51.9} \small{$\pm1.1$} & \gc \textbf{69.3} \small{$\pm1.2$} & \gc \textbf{50.7} \small{$\pm1.2$} & \gc \textbf{63.6} \small{$\pm2.0$} & \gc \textbf{48.4} \small{$\pm1.1$} & \gc \textbf{54.4} \small{$\pm1.0$} & \gc \textbf{39.2} \small{$\pm1.4$} & \gc \textbf{55.4} \small{$\pm1.7$} & \gc \textbf{39.8} \small{$\pm1.1$} \\
    \bottomrule \bottomrule
  \end{tabular}}
\end{table}

\begin{table}[h!]
    \centering
    \caption{Test accuracies (\%) on ImageNet-Subset (ImageNet-[A, B, C, D, E], $256 \times 256$) with IPC=1 (\#Params=1966.08k). "Cross" denotes the average test accuracy of trained AlexNet, VGG11, ResNet18, and ViT on each synthetic dataset.} \label{appendix:tab:imagenet256}
    \adjustbox{max width=\textwidth}{%
    \begin{tabular}{c cc cc cc cc cc}
    \toprule \toprule
     & \multicolumn{2}{c}{ImageNet-A} & \multicolumn{2}{c}{ImageNet-B} & \multicolumn{2}{c}{ImageNet-C} & \multicolumn{2}{c}{ImageNet-D} & \multicolumn{2}{c}{ImageNet-E} \\
    \cmidrule(lr){2-3} \cmidrule(lr){4-5} \cmidrule(lr){6-7} \cmidrule(lr){8-9} \cmidrule(lr){10-11}
     & ConvNet & Cross & ConvNet & Cross & ConvNet & Cross & ConvNet & Cross & ConvNet & Cross \\
    \midrule
    DC      & - & \underline{38.3} \small{$\pm4.7$} & - & 32.8 \small{$\pm4.1$} & - & 27.6 \small{$\pm3.3$} & - & 25.5 \small{$\pm1.2$} & - & 23.5 \small{$\pm2.4$} \\
    w/ GLaD & - & 37.4 \small{$\pm5.5$} & - & \underline{41.5} \small{$\pm1.2$} & - & \underline{35.7} \small{$\pm4.0$} & - & \underline{27.9} \small{$\pm1.0$} & - & \underline{29.3} \small{$\pm1.2$} \\
    \gc w/ FreD & \gc 54.8 \small{$\pm0.9$} & \gc \textbf{48.0} \small{$\pm0.9$} & \gc 56.2 \small{$\pm1.0$} & \gc \textbf{48.2} \small{$\pm1.7$} & \gc 53.5 \small{$\pm1.4$} & \gc \textbf{47.3} \small{$\pm1.0$} & \gc 41.6 \small{$\pm1.2$} & \gc \textbf{37.8} \small{$\pm1.0$} & \gc 39.1 \small{$\pm1.5$} & \gc \textbf{33.4} \small{$\pm1.2$} \\
    \bottomrule \bottomrule
  \end{tabular}}
\end{table}

\subsection{Performance Comparison on Large-size Dataset}
\begin{wraptable}{r}{.45\textwidth}
    \centering
    \caption{Test accuracies (\%) on LSUN.} \label{appendix:tab:LSUN}
    \resizebox{.45\textwidth}{!}{
    \begin{tabular}{c cc cc}
        \toprule \toprule
         & \multicolumn{2}{c}{LSUN-10k} & \multicolumn{2}{c}{LSUN-25k} \\
        \cmidrule(lr){2-3} \cmidrule(lr){4-5}
         & DC & DM & DC & DM \\
        \midrule
        Vanilla & \underline{24.0} \small{$\pm1.1$} & 22.3 \small{$\pm0.4$} & \underline{23.9} \small{$\pm0.5$} & 22.3 \small{$\pm0.4$} \\
        w/ IDC & 22.7 \small{$\pm0.3$} & \underline{27.4} \small{$\pm0.8$} & 22.7 \small{$\pm0.7$} & \underline{27.1} \small{$\pm0.4$} \\
        \gc w/ FreD & \gc \textbf{30.3} \small{$\pm0.9$} & \gc \textbf{37.1} \small{$\pm0.2$} & \gc \textbf{32.1} \small{$\pm0.2$} & \gc \textbf{36.3} \small{$\pm0.6$} \\
        \midrule
        Entire dataset & \multicolumn{2}{c}{71.8 \small{$\pm0.3$}} & \multicolumn{2}{c}{72.8 \small{$\pm0.3$}} \\
        \bottomrule \bottomrule
    \end{tabular}}
\end{wraptable}
Distilling the dataset into a small cardinality synthetic dataset can be more effective when the size of the original is large. Therefore, we further investigate the usefulness of our method and several baselines on a dataset with a large number of instances. We choose LSUN dataset \cite{yu2015lsun} as the large-size dataset. Table \ref{appendix:tab:LSUN} provides performances of FreD and other baselines on the LSUN dataset. As a result, FreD achieves the best performance compared to the implemented baselines.

\subsection{More Results on Compatibility of Parameterization.}
In Table \ref{experiments:tab:compatibility}, we reported an average performance over the unseen test network architecture such as AlexNet, VGG11, and ResNet18 for evaluating the cross-architecture generalization. Herein, we provide detailed performance for each test network architecture. Table \ref{apppendix:tab:compatibility} repeatedly shows the significant performance improvement of FreD in terms of cross-architecture generalization. These experimental results validate the effectiveness of frequency domain-based parameterization on both the dataset distillation objective and unseen test architectures.
\begin{table}[h]
    \centering
    \caption{Test accuracies (\%) on CIFAR-10 under various dataset distillation loss and cross-architecture. We distill the synthetic dataset by using ConvNet.} \label{apppendix:tab:compatibility}
    \adjustbox{max width=\textwidth}{%
    \begin{tabular}{c c ccc ccc ccc}
    \toprule \toprule
    & & \multicolumn{3}{c}{DC} & \multicolumn{3}{c}{DM} & \multicolumn{3}{c}{TM} \\
    \cmidrule(lr){3-5} \cmidrule(lr){6-8} \cmidrule(lr){9-11} 
     & IPC & 2 & 11 & 51 & 2 & 11 & 51 & 2 & 11 & 51 \\
     & \#Params & 61.44k & 337.92k & 1566.72k & 61.44k & 337.92k & 1566.72k & 61.44k & 337.92k & 1566.72k \\
    \midrule

    \multirow{4}{*}{AlexNet} & Vanilla & 20.0 \small{$\pm1.3$} & 22.4 \small{$\pm1.4$} & 29.5 \small{$\pm0.9$} & 20.7 \small{$\pm3.6$} & 37.0 \small{$\pm0.9$} & 49.1 \small{$\pm0.9$} & 26.1 \small{$\pm1.0$} & 36.0 \small{$\pm1.5$} & 49.2 \small{$\pm1.3$} \\
     & w/ IDC & \underline{26.8} \small{$\pm1.8$} & \underline{41.5} \small{$\pm0.5$} & \underline{44.2} \small{$\pm0.7$} & \underline{36.4} \small{$\pm1.1$} & \underline{47.7} \small{$\pm0.6$} & \underline{59.2} \small{$\pm0.7$} & 32.5 \small{$\pm2.2$} & 43.7 \small{$\pm3.0$} & 54.9 \small{$\pm1.1$} \\
     & w/ HaBa & 22.2 \small{$\pm1.1$} & 33.0 \small{$\pm0.9$} & 33.4 \small{$\pm1.4$} & 32.1 \small{$\pm0.6$} & 44.1 \small{$\pm0.7$} & 53.1 \small{$\pm0.9$} & \underline{43.6} \small{$\pm1.5$} & \underline{49.0} \small{$\pm3.0$} & \underline{60.1} \small{$\pm1.4$} \\
     & \gc w/ FreD & \gc \textbf{39.8} \small{$\pm0.4$} & \gc \textbf{42.4} \small{$\pm0.6$} & \gc \textbf{46.4} \small{$\pm0.5$} & \gc \textbf{46.4} \small{$\pm0.7$} & \gc \textbf{55.7} \small{$\pm0.5$} & \gc \textbf{65.7} \small{$\pm0.5$} & \gc \textbf{44.1} \small{$\pm1.3$} & \gc \textbf{55.9} \small{$\pm0.8$} & \gc \textbf{65.9} \small{$\pm0.8$} \\
    \midrule

     \multirow{4}{*}{VGG11} & Vanilla & 28.0 \small{$\pm0.3$} & 35.9 \small{$\pm0.7$} & 38.7 \small{$\pm0.5$} & 22.3 \small{$\pm1.0$} & 41.6 \small{$\pm0.6$} & 55.2 \small{$\pm0.5$} & 38.0 \small{$\pm1.2$} & 50.5 \small{$\pm1.0$} & 61.4 \small{$\pm0.3$} \\
     & w/ IDC & \underline{34.3} \small{$\pm0.7$} & \textbf{40.0} \small{$\pm0.5$} & \underline{42.4} \small{$\pm0.8$} & \underline{38.2} \small{$\pm0.6$} & \underline{52.8} \small{$\pm0.5$} & 62.2 \small{$\pm0.3$} & 48.2 \small{$\pm1.2$} & 52.1 \small{$\pm0.7$} & 65.2 \small{$\pm0.6$} \\
     & w/ HaBa & 29.4 \small{$\pm0.9$} & 37.0 \small{$\pm0.4$} & 41.9 \small{$\pm0.6$} & 26.9 \small{$\pm0.6$} & 49.4 \small{$\pm0.4$} & \textbf{67.5} \small{$\pm0.4$} & \underline{48.3} \small{$\pm0.5$} & \textbf{60.5} \small{$\pm0.6$} & \underline{67.5} \small{$\pm0.4$} \\
     & \gc w/ FreD & \gc \textbf{38.8} \small{$\pm0.9$} & \gc \textbf{40.0} \small{$\pm0.8$} & \gc \textbf{44.8} \small{$\pm0.9$} & \gc \textbf{48.1} \small{$\pm0.9$} & \gc \textbf{59.0} \small{$\pm0.6$} & \gc \underline{66.6} \small{$\pm0.2$} & \gc \textbf{51.0} \small{$\pm0.8$} & \gc \underline{60.0} \small{$\pm0.6$} & \gc \textbf{69.9} \small{$\pm0.4$} \\
    \midrule
    
    \multirow{4}{*}{ResNet18} & Vanilla & 18.1 \small{$\pm0.8$} & 18.4 \small{$\pm0.4$} & 22.1 \small{$\pm0.4$} & 22.3 \small{$\pm1.0$} & 40.0 \small{$\pm1.5$} & 53.4 \small{$\pm0.7$} & 35.2 \small{$\pm1.0$} & 45.1 \small{$\pm1.5$} & 54.5 \small{$\pm1.0$} \\
     & w/ IDC & \underline{24.9} \small{$\pm0.9$} & \underline{24.8} \small{$\pm0.7$} & \underline{34.1} \small{$\pm0.7$} & \underline{37.3} \small{$\pm1.5$} & \underline{50.9} \small{$\pm0.7$} & \underline{62.5} \small{$\pm0.5$} & 46.7 \small{$\pm0.9$} & 50.2 \small{$\pm0.6$} & \underline{64.5} \small{$\pm1.2$} \\
     & w/ HaBa & 24.5 \small{$\pm0.6$} & 24.3 \small{$\pm0.6$} & 31.1 \small{$\pm0.3$} & 31.3 \small{$\pm0.7$} & 47.6 \small{$\pm0.5$} & 59.6 \small{$\pm0.4$} & \underline{47.4} \small{$\pm0.7$} & \underline{58.0} \small{$\pm0.9$} & 64.4 \small{$\pm0.6$} \\
     & \gc w/ FreD & \gc \textbf{33.0} \small{$\pm1.1$} & \gc \textbf{29.8} \small{$\pm0.6$} & \gc \textbf{37.0} \small{$\pm0.9$} & \gc \textbf{49.7} \small{$\pm0.3$} & \gc \textbf{57.3} \small{$\pm1.2$} & \gc \textbf{62.6} \small{$\pm1.0$} & \gc \textbf{53.9} \small{$\pm0.7$} & \gc \textbf{64.4} \small{$\pm0.6$} & \gc \textbf{71.4} \small{$\pm0.7$} \\
    \bottomrule \bottomrule
  \end{tabular}}
\end{table}

\subsection{More Results on Robustness against Corruption.} \label{appendix:corruption}
\begin{wrapfigure}{r}{0.3\textwidth}
\vspace{-0.3in}
    \centering
    \includegraphics[width=\linewidth]{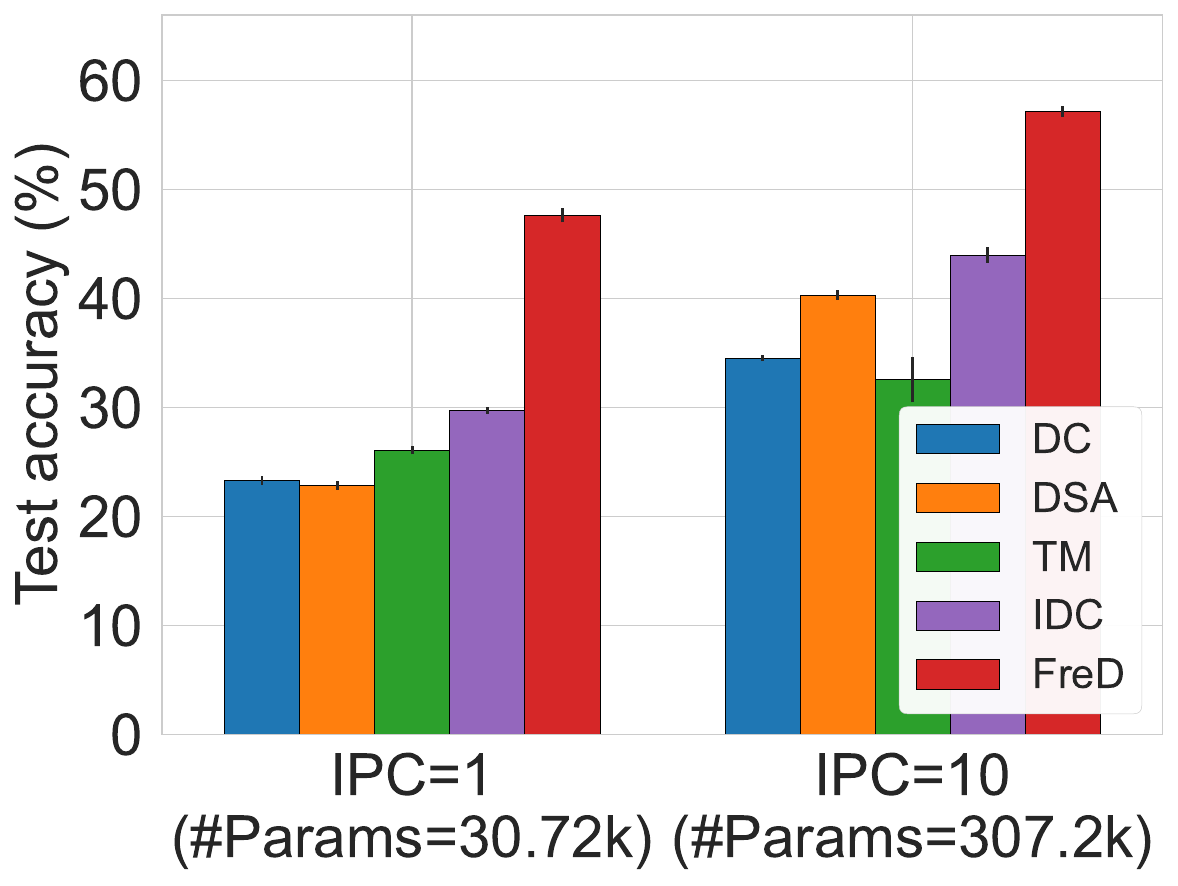}
    \caption{Test accuracies (\%) on CIFAR-10.1} \label{appendix:fig:cifar-10.1}
\vspace{-0.1in}
\end{wrapfigure}
Our proposed method, FreD, demonstrates substantial robustness against corruption, evidence supported by the findings in Figure \ref{experiments:fig:cifar10-c} and Table \ref{experiments:tab:imagenet-c}. In this context, we provide further experimental results: 1) the test accuracies results for CIFAR-10.1, and 2) a detailed breakdown of test accuracies based on different types of corruption in CIFAR-10-C. For detailed results on CIFAR-10-C, we report the performance of the severest level and average across all severity levels. In Figure \ref{appendix:fig:cifar-10.1}, FreD achieves the best performance with a significant gap over the baseline methods on CIFAR-10.1. Furthermore, Table \ref{appendix:tab:cifar10-c} verifies the superior robustness regardless of corruption type. 
\begin{table}[h]
    \centering
    \caption{Test accuracies (\%) on CIFAR-10-C with IPC=1 (\#Params=30.72k).} \label{appendix:tab:cifar10-c}
    \begin{subtable}{\textwidth}
        \caption{Severity level 5 (most severest)}
        \adjustbox{max width=\textwidth}{%
        \begin{tabular}{c ccc cccc cccc cccc| c}
        \toprule \toprule
        & Gauss. & Shot & Impul. & Defoc. & Glass & Motion & Zoom & Snow & Frost & Fog & Brit. & Contr & Elastic & Pixel & JPEG & Avg. \\
        \midrule
        DC  & 28.0 & 28.1 & 27.5 & 28.4 & 28.1 & 27.9 & 27.4 & 27.7 & 27.9 & 19.3 & 28.1 & 28.2 & 28.1 & 28.5 & 28.5 & 27.4 \\
        DSA & 27.5 & 27.5 & 27.0 & 27.9 & 27.5 & 27.5 & 26.7 & 27.2 & 27.7 & 18.8 & 27.3 & 28.9 & 27.9 & 28.1 & 27.9 & 27.0 \\
        TM  & 30.2 & 30.4 & 28.6 & 29.0 & 28.0 & 28.2 & 28.6 & 30.4 & 29.6 & 23.0 & 32.5 & 31.4 & 29.3 & 30.0 & 31.2 & 29.4 \\
        IDC & \underline{36.4} & \underline{36.2} & \underline{33.3} & \underline{39.4} & \underline{37.6} & \underline{38.6} & \underline{38.4} & \underline{38.1} & \underline{38.7} & \underline{29.7} & \underline{37.9} & \underline{39.1} & \underline{38.4} & \underline{39.2} & \underline{38.7} & \underline{37.3} \\
        \gc FreD & \gc \textbf{54.4} & \gc \textbf{54.2} & \gc \textbf{48.9} & \gc \textbf{57.1} & \gc \textbf{54.8} & \gc \textbf{55.4} & \gc \textbf{55.4} & \gc \textbf{55.9} & \gc \textbf{54.7} & \gc \textbf{43.3} & \gc \textbf{56.3} & \gc \textbf{41.6} & \gc \textbf{57.1} & \gc \textbf{58.5} & \gc \textbf{58.1} & \gc \textbf{53.7} \\
        \bottomrule \bottomrule
        \end{tabular}}
    \end{subtable}
    \begin{subtable}{\textwidth}
        \caption{Average across all severity levels}
        \adjustbox{max width=\textwidth}{%
        \begin{tabular}{c ccc cccc cccc cccc| c}
        \toprule \toprule
        & Gauss. & Shot & Impul. & Defoc. & Glass & Motion & Zoom & Snow & Frost & Fog & Brit. & Contr & Elastic & Pixel & JPEG & Avg. \\
        \midrule
        DC  & 28.2 & 28.3 & 28.0 & 28.5 & 28.3 & 28.2 & 27.8 & 28.0 & 28.2 & 24.3 & 28.4 & 28.6 & 28.1 & 28.5 & 28.5 & 28.0 \\
        DSA & 27.8 & 27.8 & 27.5 & 28.1 & 27.8 & 27.8 & 27.3 & 27.6 & 27.8 & 23.7 & 27.9 & 28.8 & 27.6 & 28.2 & 28.0 & 27.6 \\
        TM  & 30.8 & 31.1 & 29.9 & 30.5 & 29.0 & 29.5 & 29.6 & 31.0 & 30.5 & 28.0 & 32.3 & 32.4 & 29.5 & 31.1 & 31.5 & 30.4 \\
        IDC & \underline{37.4} & \underline{37.8} & \underline{36.3} & \underline{39.7} & \underline{38.2} & \underline{39.0} & \underline{39.0} & \underline{38.9} & \underline{38.6} & \underline{35.7} & \underline{39.3} & \underline{40.4} & \underline{38.5} & \underline{39.4} & \underline{39.1} & \underline{38.5} \\
        \gc FreD & \gc \textbf{56.7} & \gc \textbf{57.3} & \gc \textbf{54.4} & \gc \textbf{58.9} & \gc \textbf{56.4} & \gc \textbf{57.2} & \gc \textbf{57.3} & \gc \textbf{58.0} & \gc \textbf{56.5} & \gc \textbf{53.6} & \gc \textbf{59.2} & \gc \textbf{53.5} & \gc \textbf{57.2} & \gc \textbf{59.6} & \gc \textbf{58.9} & \gc \textbf{57.0} \\
        \bottomrule \bottomrule
        \end{tabular}}
    \end{subtable}
\end{table}

\subsection{Performance Comparison with Memory Addressing}
Memory addressing (MA) \cite{deng2022remember} is a new parameterization method to create a common representation by encapsulating the features shared among different classes into a set of bases. The reported performances of \cite{deng2022remember} show mixed results under various settings. However, the performance of \cite{deng2022remember} is not solely due to the MA but also includes the effects of other components. For a fair comparison between MA and FreD, we standardized the distillation loss and evaluated their performances.

Figure \ref{appendix:fig:transfer_MAnFreD_10} shows that MA and FreD exhibit competitive performances on CIFAR-10 with each other under the implemented settings of DM and TM losses. We further assessed the robustness of each approach by evaluating the transferability of the synthetic datasets against diverse distribution shifts. Figure \ref{appendix:fig:transfer_MAnFreD_10.1} represents mixed result performance on CIFAR-10.1. As a result in Figure \ref{appendix:fig:transfer_MAnFreD_10-C}, FreD particularly shows better performances than MA in most corrupted versions of datasets. Furthermore, it should be noted that FreD achieves higher performance than MA when the severity level becomes higher. We conjecture that because FreD selects informative dimensions in the frequency domain, it has good robustness to corruptions that typically occur in the high-frequency domain. In terms of computational time, MA requires about nearly three times more than FreD. Please refer to Section \ref{section:complexity} for the detailed discussion.

While both methods are distinct approaches, the implementation of MA in the spatial domain allows a further transformation to the frequency domain. This enables the orthogonal application of MA and FreD. One possible combination is to define the bases of MA in the frequency domain and select the informative dimensions. It allows more flexible parameterization. We leave it as future work.
\begin{figure}[h]
    \centering
    \begin{subfigure}{0.315\textwidth}
        \centering
        \centerline{\includegraphics[width=\textwidth]{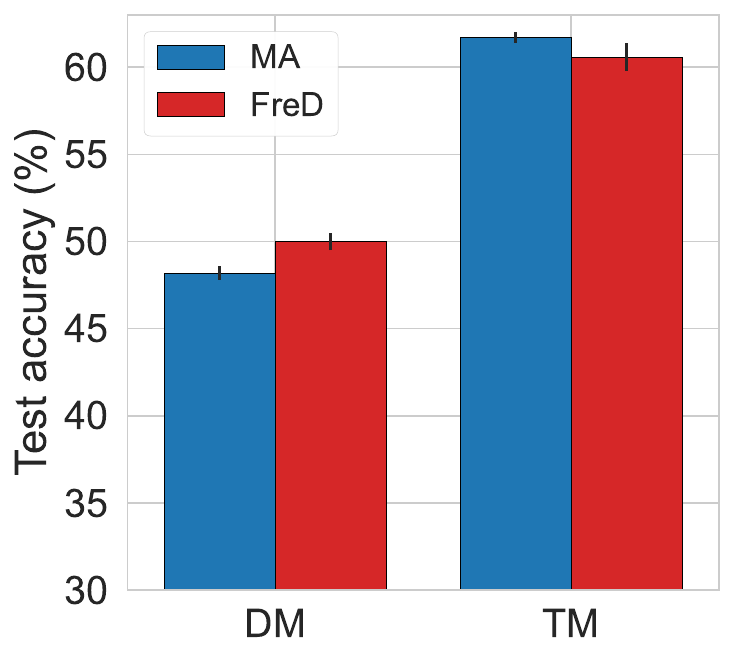}}
        \caption{Target dataset: CIFAR-10}
        \label{appendix:fig:transfer_MAnFreD_10}
    \end{subfigure}
    \begin{subfigure}{0.315\textwidth}
        \centering
        \centerline{\includegraphics[width=\textwidth]{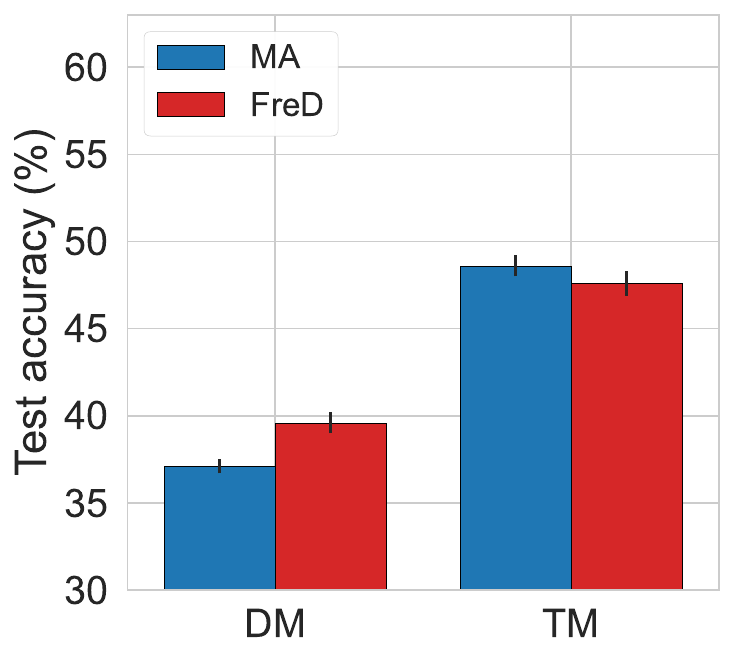}}
        \caption{Target dataset: CIFAR-10.1}
        \label{appendix:fig:transfer_MAnFreD_10.1}
    \end{subfigure}
    \begin{subfigure}{0.315\textwidth}
        \centering
        \centerline{\includegraphics[width=\textwidth]{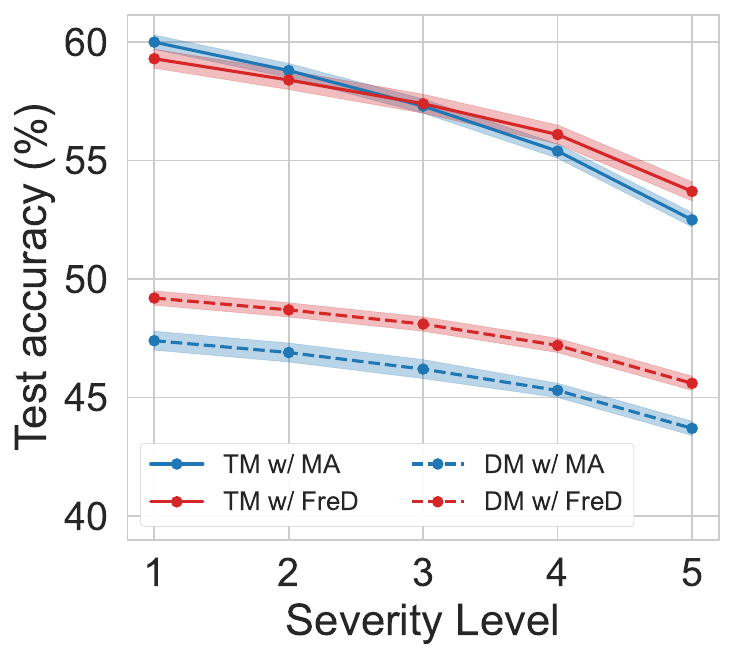}}
        \caption{Target dataset: CIFAR-10-C}
        \label{appendix:fig:transfer_MAnFreD_10-C}
    \end{subfigure}
    \caption{Performance comparison of MA and FreD on each target dataset (Source dataset: CIFAR-10). Note that higher level indicates higher corruption.} \label{appendix:fig:transfer_MAnFreD}
\end{figure}

\subsection{Compatibility with BPTT}
\begin{wrapfigure}{h}{.35\textwidth}
    \centering
        \includegraphics[width=\linewidth]{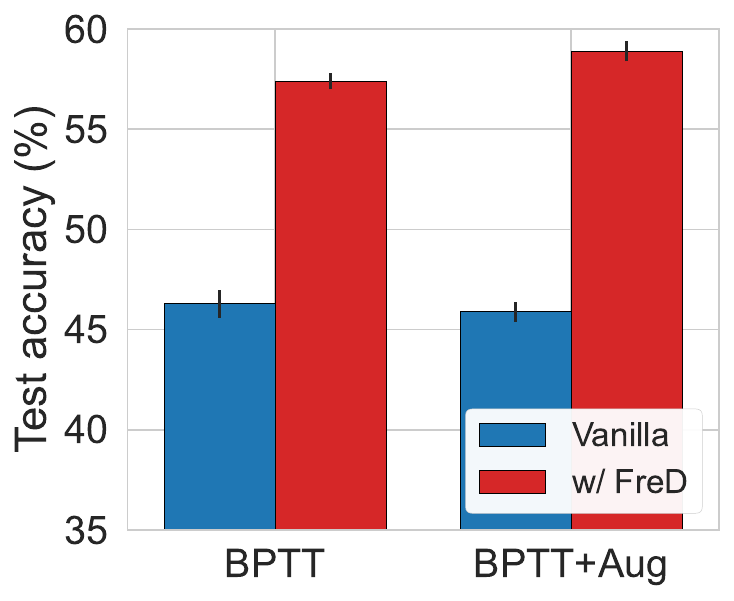}
    \caption{Study of FreD on BPTT} \label{appendix:fig:BPTT with FreD}
\end{wrapfigure}
Back-propagation through time (BPTT) is another optimization framework that effectively solves the bi-level optimization problem. \cite{deng2022remember} suggests BPTT to train the synthetic dataset in dataset distillation. FreD is a new type of parameterization framework for dataset distillation, while BPTT is introduced as a new optimization framework for dataset distillation. Hence, they can be utilized orthogonally. To verify the efficiency of FreD, we conduct an experiment on models that combine the BPTT framework with FreD.

Figure \ref{appendix:fig:BPTT with FreD} shows the performance of the model with FreD in the BPTT framework on CIFAR-10 under IPC=1 (\#Params=30.72k). As mentioned in \cite{deng2022remember}, we considered two variants of BPTT framework with and without augmentation. We reduced the number of training iterations for each baseline and FreD from 50,000 to 5,000.

As a result, BPTT with FreD outperformed BPTT without FreD under BPTT framework regardless of whether or not the augmentation was used. Furthermore, even when compared to the performance of BPTT with full iteration training reported in the original paper ($49.1\pm0.6$), BPTT w/ FreD achieved higher performance ($57.4\pm0.4$). It indicates that FreD is an efficient methodology that can also be applied in the BPTT framework.

\subsection{Additional Ablation Study on Frequency Transform} \label{appendix:additional ablation FT}
We basically utilized three frequency transforms: DCT, DFT, and DWT. We especially want to highlight the energy compaction property of DCT, where most of the signal information tends to be concentrated in a few low-frequency components (Please refer to Figure 2 in \cite{ahmed1974discrete} and Figure 1 in \cite{yaroslavsky2015compression}). This characteristic aligns well with the motivation of FreD, and Figure \ref{experiments:fig:ablationF} demonstrates that DCT is the best choice among the possible options.

To further analyze the effect of frequency transforms on FreD, we have conducted various experiments. Figure \ref{appendix:fig:ablationF} presents the results as follows:
\begin{itemize}
    \item Across all settings, we observe improved performances of FreD than the baseline regardless of the type of frequency transform employed.
    \item DCT outperforms DFT and DWT in most cases, highlighting the effective exploitation of DCT's energy compaction property within the FreD framework.
    \item DFT exhibits relatively lower performance in comparison to DCT and DWT. This discrepancy is attributed to the complex-valued nature of DFT. Unlike DCT and DWT, which operate in real space, DFT requires additional resources to represent a single instance due to its complex space. As a result, the quantity of synthetic instances that can be generated within an identical budget is reduced by half than others, leading to lower performance.
\end{itemize}
\begin{figure}[h]
    \centering
    \subcaptionbox{SVHN}{\includegraphics[width=\linewidth]{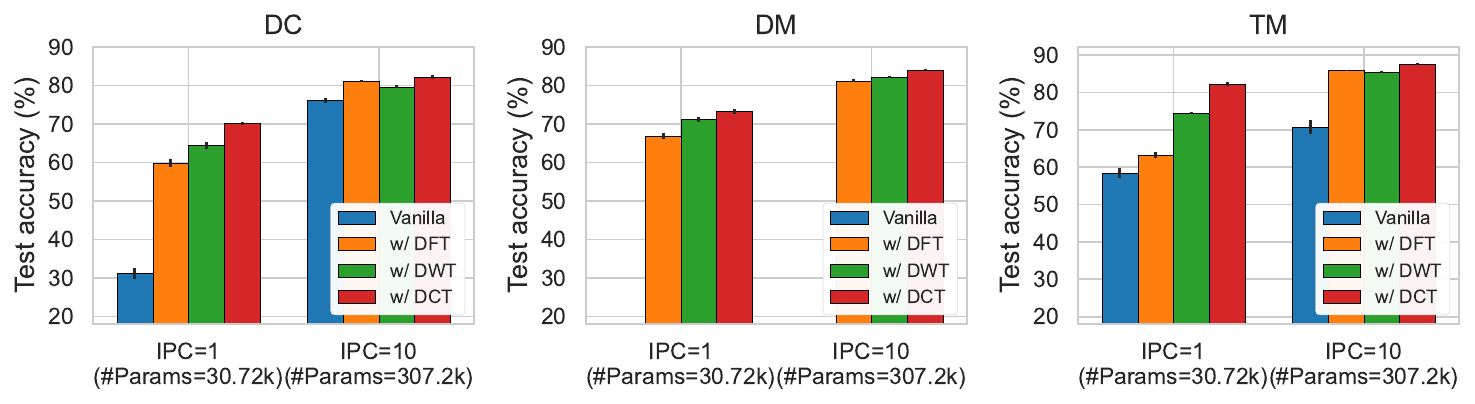}}\par 
    \subcaptionbox{CIFAR-10}{\includegraphics[width=\linewidth]{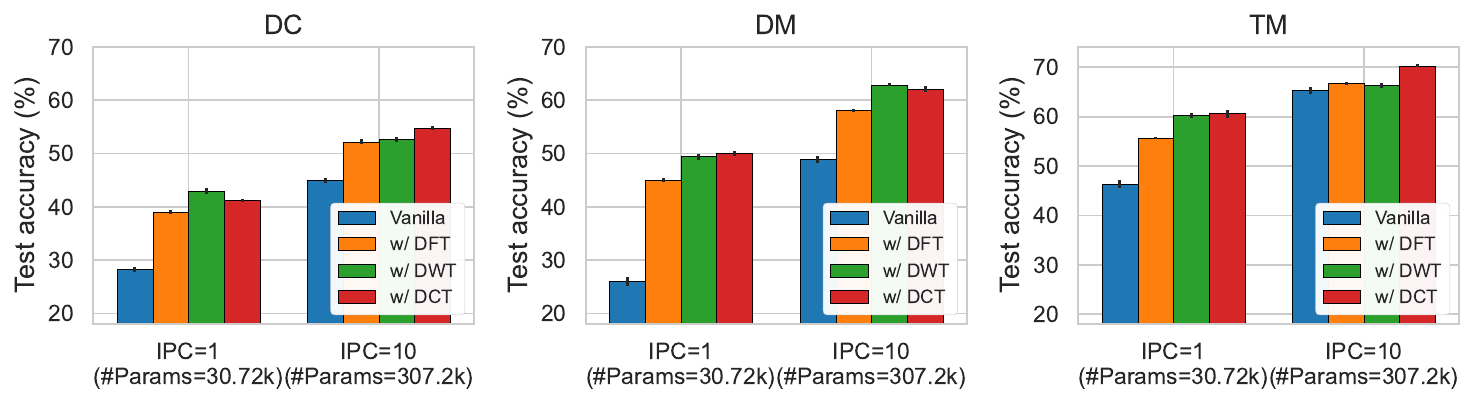}}\par
    \caption{Ablation study on the frequency transform. Note that DM does not provide the test accuracies on SVHN in the original paper.}
    \label{appendix:fig:ablationF}
\end{figure}

\subsection{Performance Comparison with Post-downsampling}
As mentioned by IDC, the most basic methodology for dataset distillation is to generate the large-cardinality $S$ and compress them with post-processing. In the previous study, the comparison was conducted only in the spatial domain, but this paper extends it to consider post-processing in the frequency domain. Post-processing is the compression of vanilla in each domain. Table \ref{appendix:tab:post} presents the results as follows:
\begin{itemize}
    \item Post-downsampling in both domains achieves lower performance than DM since they compress the trained synthetic dataset. These results indicate the inevitable information loss. While post-downsampling shows information loss, frequency domain-based downsampling achieves higher performance than spatial domain. It demonstrates the frequency domain stores task-relevant information more effectively than spatial domain.
    \item End-to-end methods achieve higher performance than post-downsampling methods. Among them, the frequency domain-based method (FreD) achieves higher performance than the spatial domain-based method (IDC).
    \item FreD shows a higher cross-architecture generalization despite spending a quarter of the budget of the vanilla model.
\end{itemize}
\begin{table}[h]        
    \centering
    \caption{Test accuracies (\%) comparison under various test network architecture on CIFAR-10. We utilize DM for the dataset distillation loss and ConvNet for the training architecture.} \label{appendix:tab:post}
    \resizebox{\textwidth}{!}{
    \begin{tabular}{cc cc cccc}
    \toprule \toprule
    \makecell{Decoded instances\\per class}& \#Params & \multicolumn{2}{c}{Model} & ConvNet & AlexNet & VGG11 & ResNet18 \\
    \midrule
    \multirow{5}{*}{40}         & 1228.8k  & \multicolumn{2}{c}{DM}    & \textbf{61.2} \small{$\pm0.4$} & \underline{48.8} \small{$\pm0.5$} & \underline{53.9} \small{$\pm0.5$} & \underline{52.1} \small{$\pm0.5$} \\
                                \cmidrule{2-8}
                                & \multirow{4}{*}{307.2k}   & \multirow{2}{*}{Post-downsampling} & Spatial & 56.7 \small{$\pm0.5$} & 44.6 \small{$\pm0.8$} & 49.9 \small{$\pm0.6$} & 49.5 \small{$\pm0.6$} \\
                                & & & Frequency & 59.3 \small{$\pm0.4$} & 47.4 \small{$\pm0.5$} & 52.5 \small{$\pm0.5$} & 51.2 \small{$\pm0.6$} \\
                                \cmidrule{3-8}
                                & & \multirow{2}{*}{End-to-End} & IDC & 59.6 \small{$\pm0.5$} & 47.6 \small{$\pm0.7$} & 52.2 \small{$\pm0.6$} & 50.8 \small{$\pm0.5$} \\
                                & & & \gc FreD & \gc \underline{60.5} \small{$\pm0.3$} & \gc \textbf{50.9} \small{$\pm0.4$} & \gc \textbf{54.8} \small{$\pm0.3$} & \gc \textbf{53.1} \small{$\pm0.9$} \\
    \bottomrule \bottomrule
  \end{tabular}}
\end{table}

\subsection{More Visualization of Binary Mask and Transformed Images} \label{appendix:more visualization}
We provide the binary mask and transformed images from our proposed method on various datasets: SVHN (see Figure \ref{appendix:fig:qual_svhn}), CIFAR-10 (see Figure \ref{appendix:fig:qual_cifar10}), CIFAR-100 (see Figure \ref{appendix:fig:qual_cifar100}), Tiny-ImageNet (see Figure \ref{appendix:fig:qual_tiny}), and ImageNet-Subset (see Figure \ref{appendix:fig:qual_imagenet1} and \ref{appendix:fig:qual_imagenet3}). For CIFAR-100 and Tiny-ImageNet, we visualize the first 10 classes. For a better layout, we plot these visualizations at the end of the paper. Through these results, we can observe that the constructed synthetic dataset by FreD contains both intra-class diversity and inter-class discriminative features, regardless of the image resolution.

For 3D MNIST experiments, we provide Figure \ref{appendix:fig:3D MNIST}, which displays the original image and a set of trained synthetic data through each distillation method. To enable visualization of the $16\times16\times16$ dimension point cloud, we sliced each instance's depth dimension into 16 images and displayed them separately. The image located at the top left represents the frontmost view, while the image at the bottom right corresponds to the rearmost view. From Figure \ref{appendix:fig:3D MNIST}, FreD effectively captures the class-discriminative information that class 0 should possess. It indicates that the proposed frequency-based dataset distillation framework is applicable to higher-dimensional data than two-dimensional data. Furthermore, compared to the DM and IDC, the synthesized instance by FreD shows more clearer boundary in dimensions $6\sim11$ which is the key class-discriminative information of 0. This result demonstrates that the selection of informative dimensions in the frequency domain is effective. In the revised paper, we will add the visualization of the 3D MNIST cloud synthesized by each method.

\section{Additional Discussions}
\subsection{Comparison between FreD and PCA-based Transform} \label{appendix:comparison with PCA}
\begin{wrapfigure}{h!}{0.3\textwidth}
\vspace{-0.3in}
    \centering
    \includegraphics[width=\linewidth]{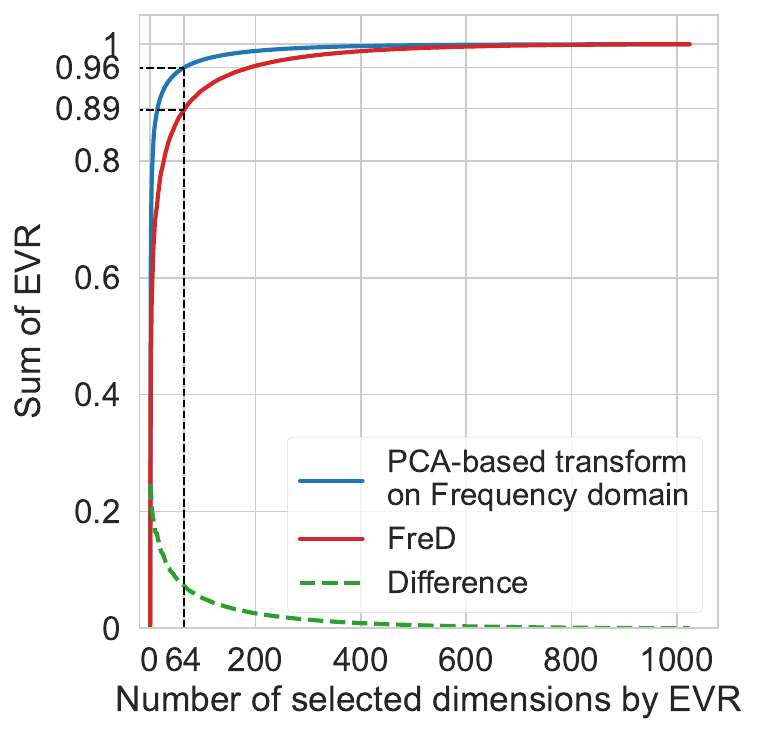}
\caption{PCA vs FreD.} \label{appendix:fig:evr_supp}
\vspace{-0.4in}
\end{wrapfigure}
PCA-based transform, which sets the principal components of the given dataset as new axes, can ideally preserve the sum of explained variance ratio by selecting top-$k$ principal components as a subset of new dimensions. However, there are some evidence for the claim that PCA cannot be practically utilized as a method of dataset distillation.

First, PCA-based transform requires an additional budget to store the transform matrix. PCA-based transform utilizes top-$k$ principal components as new axes of the introduced domain. As these axes are composed of a weighted sum of each dimension value, and therefore, it is not possible to implement a feature like FreD, which selects a subset of dimensions from the overall dimensions of the domain. Therefore, the transform matrix created for projection is a $d \times k$ dimensional matrix consisting of the top-$k$ principal components. This matrix needs to be stored separately from the condensed dataset $S$, as it represents a distinct component for transformation, which means an additional budget is needed. Unlike PCA-based transforms, in the frequency domain transform, once you choose a specific frequency transform, the corresponding transform function and inverse transform function remain fixed. Therefore, there is no need to manage these functions separately with an additional budget. 

Secondly, the commonly used linear PCA fails to capture the correlations present in the spatial domain of the data (e.g., correlations between adjacent pixels in an image). Although there are spatial principal component analysis \cite{spca} methods specifically designed for spatial domains, such methodologies utilize spatial kernel matrix to model the correlation information between adjacent pixels, which could introduce the possibility of information loss. In this subsection, we refer to information loss specifically to the loss of information that occurs during the utilization of a spatial kernel or converting kernel-extracted information into linear features. In other words, it is challenging to accurately determine the principal components for a given dataset during the implementation, making it difficult to use PCA transforms.

Having said that, we conducted the comparison between 1) the sum of explained variance ratio (EVR) obtained through principal component analysis using a dataset transformed into the frequency domain and 2) the sum of EVR by using FreD, which is based on dimension selection in the frequency domain. The sum of EVR based on eigenvectors is maximum in terms of other comparable baselines. However, it should be noted that even if PCA is performed based on the frequency domain, the constraint of storing the projection matrix in the memory budget still remains. Figure \ref{appendix:fig:evr_supp} illustrates the Cumulative EVR based on the different number of selected dimensions for each method. In our whole experiments, the smallest dimension selection was 64 dimensions. Based on this dimension selection, the Cumulative EVR of FreD, compared to the sum of the EVR of the top-64 eigenvectors in PCA, differs by only around 7$\%$. Furthermore, when more dimensions are selected, this difference becomes smaller. In this regard, FreD can be considered an efficient methodology that sacrifices slightly in EVR while not requiring an additional memory budget for an additional transform matrix.

\subsection{More Visualization of $\| \nabla_{S}\mathcal{L}_{DD}(S,D) \|$ in Frequency Domain}
\begin{wrapfigure}{r}{0.25\textwidth}
\vspace{-0.1in}
    \begin{subfigure}{0.45\linewidth}
        \centering
        \includegraphics[width=\textwidth]{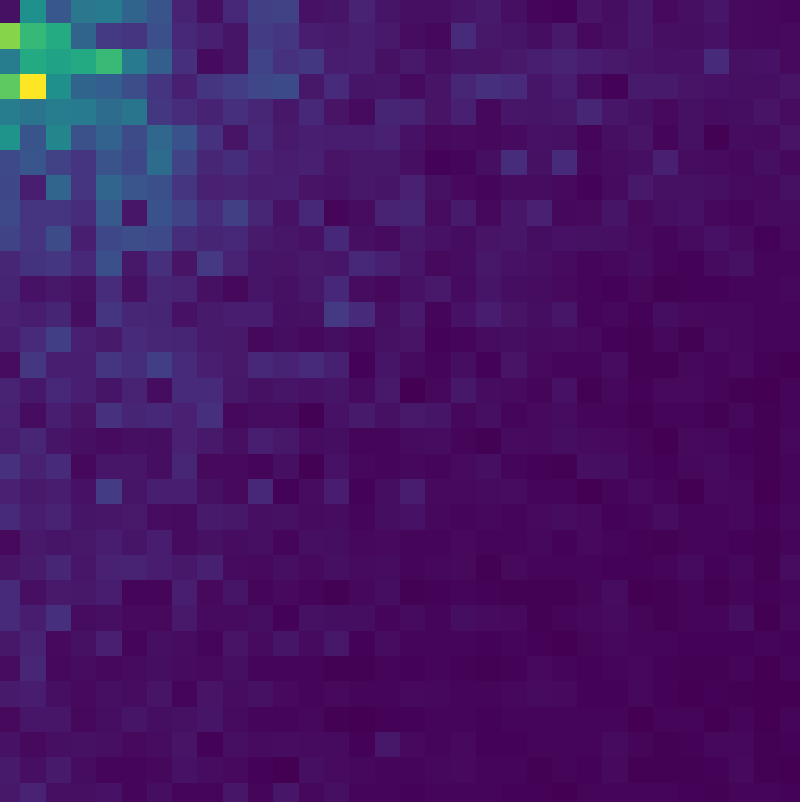}
        \caption{DC} \label{appendix:fig:gradcam DC}
    \end{subfigure}
    \begin{subfigure}{0.45\linewidth}
        \centering
        \includegraphics[width=\textwidth]{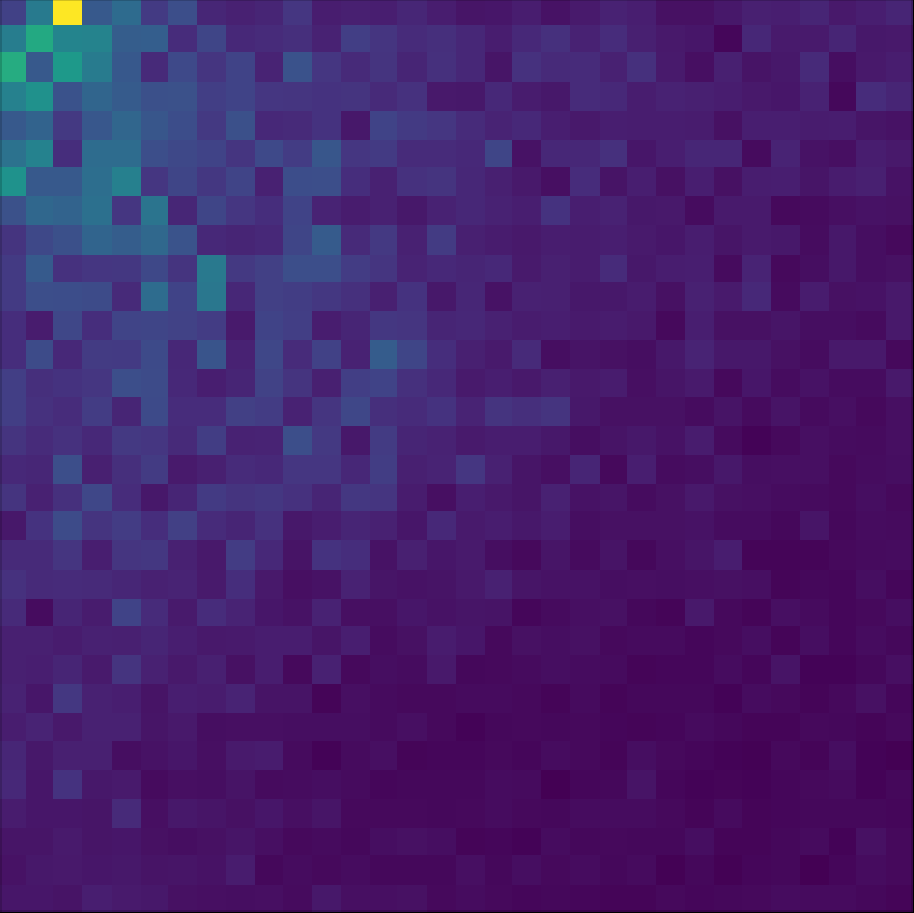}
        \caption{TM} \label{appendix:fig:gradcam TM}
    \end{subfigure}
\caption{Visualization of $\| \nabla_{S}\mathcal{L}_{DD}(S,D) \|$} \label{appendix:fig:additional gradcam}
\vspace{-0.1in}
\end{wrapfigure}
The main idea of this paper is to compress spatial domain information into fewer frequency dimensions. To verify our idea, we investigate the magnitude of the gradient for $\mathcal{L}_{DD}(S,D)$ i.e. $\| \nabla_{S}\mathcal{L}_{DD}(S,D) \|$. Specifically, we visualize the magnitude of the gradient of DM distillation loss in Figure \ref{methodology:fig:domain_comparison_2}. To demonstrate that our observation is not confined to a specific distillation loss, we provide the magnitude of the gradient in the frequency domain across different distillation losses, which are DC and TM in Figure \ref{appendix:fig:additional gradcam}. It shows that the concentration of gradient exhibits a consistent pattern regardless of the type of distillation loss.

\subsection{Discussion on Budget Allocation} \label{appendix:budget}
In this section, we will delve further into how the frequency domain-based dimension subset selection by FreD leads to a reduction in the actual budget. Let $x \in \mathbb{R}^{d_{1}\times d_{2}}$ as a data instance in the spatial domain, whose dimension size is $d = d_{1}\times d_{2}$. If each element of $x$ is a 32-bit float, each image would occupy $32\times d$ bits in memory. Let's assume the same instance is transformed into the frequency domain with the same dimension size, and we only utilize $k$ dimensions in that domain. In that case, the budget required to represent the values would decrease to $32\times k$ bits. However, in addition to simply storing the values on selected dimensions, we also need to store information about the positions where each value is located. One advantage of FreD is that instead of having separate masks for each instance, it has separate masks for each class. It means that we only need to store the position of the dimension being passed through, once per class. Therefore, we can prevent budget waste by storing the indices of the selected dimension $\Tilde{M}$ and the frequency coefficient value of that dimension $\Tilde{f}$, rather than storing the entire frequency representation $f$:
\begin{equation}
    f = \mathcal{F}(x) = 
    \begin{pmatrix}
    0.2335 & 0.0000 & 0.1246 \\
    0.1243 & 1.0442 & 0.0000 \\
    0 & 0.0000 & 0.0000
    \end{pmatrix}
    \iff
    \begin{cases}
        \Tilde{f} = [0.2335, 0.1246, 0.1243, 1.0442] \\
        \Tilde{M} = [0,2,3,4]
    \end{cases}
    \nonumber
\end{equation}
It should be noted that the masking list $\Tilde{M}$, which contains the dimension indices, only needs to store integers. Additionally, since only one masking list per class is required, it reaches a level that can be ignored in terms of the budget. Therefore, the total required budget becomes $32\times k$ bits. With the frequency information stored in this manner, it becomes possible to reassemble it into a tensor for future use without any information loss.

\subsection{Algorithm Complexity} \label{section:complexity}
As a complexity analysis, we consider the computation time for single image retrieval based on FreD and other parameterization methods. For sized 2D image in a spatial domain, IDC utilizes a resizing method as a parameterization, so its complexity is $\mathcal{O}(HW)$ where $H$ is the height and $W$ is the width of the image. HaBa, which requires an additional network for computation, inherits the complexity, $\mathcal{O}(HW{F}^{2})$, where $F$ is the filter size of the convolution neural network. MA operates by performing matrix multiplication between the matrix and downsampled bases. Given that $K$ represents the number of bases and s is the downsampling scale, MA's complexity is $\mathcal{O}(HW \times \frac{K}{s^2})$. FreD, when used with DCT, involves two operations: masking and inverse frequency transform. The complexities of these steps are $\mathcal{O}(HW)$ and $\mathcal{O}(HW\log W)$ respectively. Thus, the total complexity of FreD is $\mathcal{O}(HW\log W)$.

Based on the above complexity, Figure \ref{appendix:fig:complexity} shows the empirical wall-clock time of single-image retrieval for each method. As a result, we show that FreD inherits the second-best single-image retrieval complexity. Although IDC is most efficient in time complexity, we empirically demonstrated that the parameterization based on IDC falls behind FreD in terms of performance in most settings. Furthermore, we want to note that HaBa and GLaD, which utilize the parameterized transform, show extremely high computation cost in terms of single-image retrieval.
\begin{figure}[h]
    \centering
    \begin{subfigure}{0.32\textwidth}
        \centering
        \centerline{\includegraphics[width=\textwidth]{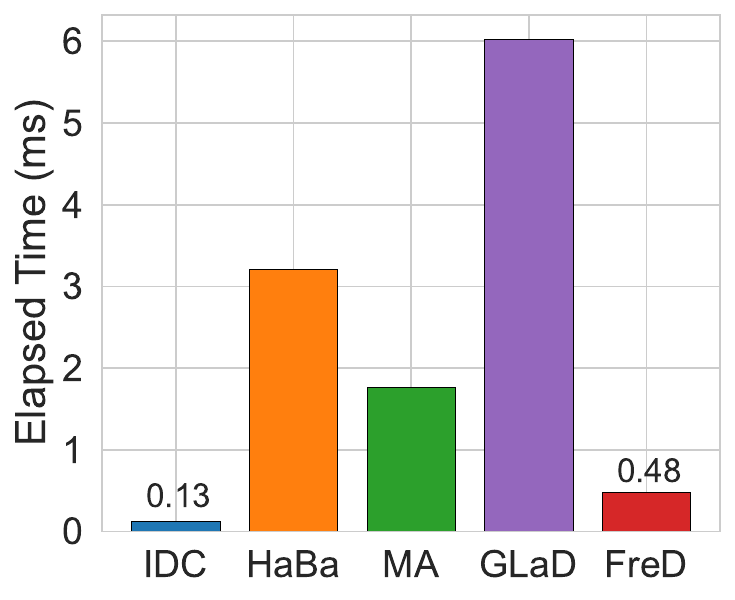}}
        \caption{CIFAR-10 ($32\times32$)}
    \end{subfigure}
    \begin{subfigure}{0.32\textwidth}
        \centering
        \centerline{\includegraphics[width=\textwidth]{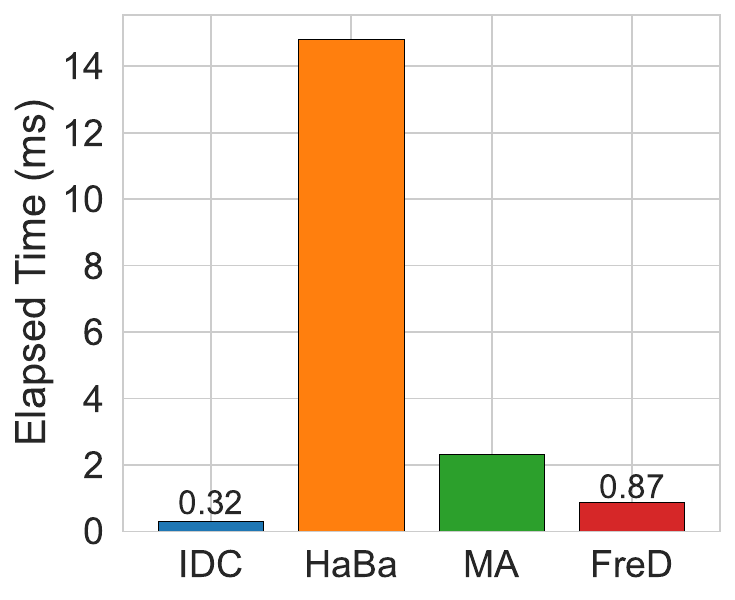}}
        \caption{Tiny-ImageNet ($64\times64$)}
    \end{subfigure}
    \begin{subfigure}{0.32\textwidth}
        \centering
        \centerline{\includegraphics[width=\textwidth]{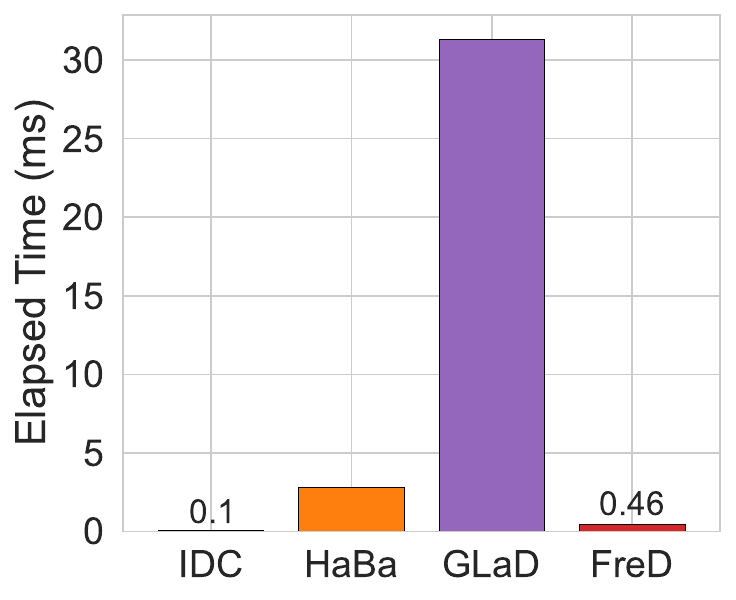}}
        \caption{ImageNet-Subset ($128\times128$)}
    \end{subfigure}
    \caption{Wall-clock time of single image retrieval for each method. "ms" denotes the millisecond.} \label{appendix:fig:complexity}
\end{figure}

\subsection{Influence of Spectral Bias on Frequency Domain-based Approach}
Spectral bias refers to the phenomenon that neural networks are prone to prioritize learning the low-frequency components over relatively higher-frequency components \cite{rahaman2019spectral, xu2022overview, durall2020watch}. For FreD, even though we employed a mask based on the explained variance ratio (EVR), which is not a simple low-frequency filter: the low-frequency components in the frequency domain were predominantly selected in most experimental scenarios. It should be noted that our EVR does not enforce keeping the low-frequency components, unlike the neural network's spectral bias; what EVR only enforces is keeping the components with a higher explanation ratio on the image feature.

However, this characteristic can become a risk to performance if the task-specific information of data is mostly found in the high-frequency components. These cases include 1) medical imaging on fine details like tumors \cite{mohan2018mri} and 2) digital watermarking \cite{nagai2018digital}. In such cases, there may be a requirement for new masking that allows FreD to capture important high-frequency components. The masking strategy of FreD can be flexibly operated, and depending on the characteristics of the given dataset and task, it can readily employ other strategies as needed.

\subsection{Impact of Linear Bijectivity Assumption of $\mathcal{F}$}
\paragraph{Impact of linear bijectivity} For the theory presented in Section \ref{section:theory} to hold, the function $\mathcal{F}$ must be linearly bijective. Note that FreD could utilize any kind of frequency transforms, such as DCT, DFT, and DWT. Since these transforms are all linearly bijective in theory, the proposed proposition could be applicable without any limitation.

Having said that, we elaborate on the potential issues that might emerge, when integrating not linearly bijective transforms into the framework of FreD, as follows:
\paragraph{When $\mathcal{F}$ is not bijective.} If $\mathcal{F}$ is not bijective, then its inverse $\mathcal{F}^{-1}$ does not exist. This absence makes the process of transforming to another domain and then restoring back to the original domain infeasible. It potentially results in information loss that interferes accurate reconstruction of the original image. An alternative solution could be separating the transform into distinct encoder and decoder components, which enable a procedure for one-to-one mapping. However, the construction of such components necessitates additional training costs, when $\mathcal{F}$ of FreD does not require any kind of additional training. Furthermore, if the encoding is not one-to-one, information might be lost during the encoding process.

\paragraph{When $\mathcal{F}$ is not linear but bijective.} FreD employs a subset of dimensions in the frequency domain. The choice is feasible as the linearly bijective transform maintains the EVR in the selected dimensions. Proposition and corollary in Section \ref{section:theory} theoretically support this attribute, where these do not apply to nonlinear bijective transforms. For 2D-image processing, nonlinear bijective transforms include 1) Log-Polar Transform, and 2) Radial Basis Function (RBF) transform. Domains from these transforms do not exhibit the concentration of the original dataset's variance on specific dimensions.

\section{Broader Impact}
How to parameterize $S$ for dataset distillation is highly versatile as it enhances efficiency and performance by determining the form of the data optimized and stored, regardless of the form of the objective of dataset distillation. Furthermore, in contrast to previous research that was solely conducted in the spatial domain, the exploration of the frequency domain introduces a new perspective in interpreting datasets. In our study, the analyzed dataset consists of pure images without any injected noise. However, real-world datasets can often contain unintended adversarial noise or other types of noise during the processing stages. Analyzing datasets based on the frequency domain enables the detection of noise that may not be visually apparent to the human eye. Moreover, by separately treating specific frequency information that is susceptible to noise, it becomes possible to extend the research to areas such as noisy-filtered dataset distillation.

\section{Limitation} \label{appendix:limitation}
\paragraph{Efficacy Differences Depending on Applied Domain.} The applicability and efficacy of FreD's frequency transform are demonstrated specifically within the spatial domain of 2D/3D images. Among the data domains commonly used in machine learning frameworks, natural language domain would likely be challenging to connect with the frequency domain directly. While Fourier Transform can still be applied to text after it's been converted into a numerical format, such as a time series, this conversion is non-trivial and the resulting frequency domain representation may not be as intuitively meaningful. Thus, for certain domains, the effectiveness of frequency transform may not be as substantial as it is for 2D/3D images.

However, in multi-modal tasks, major components like audio signals and video data naturally align with the frequency domain. Tools such as the 1D Fourier transform for audio signals and 3D Fourier transform for video data already exist to process these types directly. Excluding a few specific domains, FreD would be a framework that can be applied across a broader range of domains.

\paragraph{Performances Highly Dependent on Masking Strategy.} FreD's frequency-based parameterization is motivated by the fact that the spatial domain information of a 2D image can be concentrated in specific components of the transformed frequency domain. The EVR based masking selects important dimensions from the frequency domain, consistently showing strong performances across the various experiments conducted in this study.

Having said that, Figure \ref{experiments:fig:ablationM} demonstrates that there could be significant performance disparities depending on the masking strategy employed. Therefore, there could be substantial issues if the chosen masking strategy fails to select the important dimensions accurately. For certain datasets or tasks, essential task-specific information might be contained in the high-frequency region. These cases could include 1) medical imaging on fine details like tumors \cite{mohan2018mri} and 2) digital watermarking \cite{shao2023detecting}. This highlights a limitation that EVR masking may not be suitable for all data and tasks.

It should be noted that the masking strategy of FreD can be flexibly operated, and depending on the characteristics of the given dataset and task, it is not restricted to using an EVR-based mask and can readily employ other strategies as needed.

One potential solution to identify task-related frequency components is to utilize the gradient of the given task loss. If specific frequency components have a substantial gradient distribution, it indicates that the component greatly influences the task. By substituting with gradient-based masking, we could address the potential limitations that EVR masks might have.

\clearpage
\begin{table}[h]
    \caption{List of hyper-parameters.} \label{appendix:tab:hyperparams}
    \begin{subtable}[t]{0.49\textwidth}
        \centering
        \caption{Gradient matching (DC)}
        \adjustbox{max width=\textwidth}{%
        \begin{tabular}{c c cccc}
            \toprule \toprule
            Dataset & \#Params & \makecell{Synthetic\\batch size} & \makecell{Learning rate\\\small(Frequency)} & \makecell{Selected dimension\\per channel} & \makecell{Increment\\of instances} \\
            \midrule
            \multirow{6}{*}{CIFAR-10} & \makecell{61.44k\\\small(IPC=2)} & - & \large$10^{3}$ & \large32 & \large$\times32$ \\
                                      \cmidrule{2-6}
                                      & \makecell{337.92k\\\small(IPC=11)} & - & \large$10^{3}$ & \large128 & \large$\times8$ \\
                                      \cmidrule{2-6}
                                      & \makecell{1566.72k\\\small(IPC=51)} & \large256 & \large$10^{2}$ & \large256 & \large$\times4$ \\
            \midrule
            LSUN & \makecell{491.52k\\\small(IPC=1)} & \large80 & \large$10^{5}$ & \large128 & \large$\times128$ \\
            \midrule
            \makecell{ImageNet-\\Subset\\($128\times128$)} & \makecell{491.52k\\\small(IPC=1)} & - & \large$10^{5}$ & \large2048 & \large$\times8$ \\
            \midrule
            \makecell{ImageNet-\\Subset\\($256\times256$)} & \makecell{1966.08k\\\small(IPC=1)}  & - & \large$10^{6}$ & \large8192 & \large$\times8$ \\
            \bottomrule \bottomrule
        \end{tabular}}
    \end{subtable}
    \hfill
    \begin{subtable}[t]{0.49\textwidth}
        \centering
        \caption{Feature matching (DM)}
        \adjustbox{max width=\textwidth}{%
        \begin{tabular}{c c cccc}
            \toprule \toprule
            Dataset & \#Params & \makecell{Synthetic\\batch size} & \makecell{Learning rate\\\small(Frequency)} & \makecell{Selected dimension\\per channel} & \makecell{Increment\\of instances} \\
            \midrule
            \multirow{6}{*}{CIFAR-10} & \makecell{61.44k\\\small(IPC=2)} & - & \large$10^{6}$ & \large64 & \large$\times16$ \\
                                      \cmidrule{2-6}
                                      & \makecell{337.92k\\\small(IPC=11)} & - & \large$10^{5}$ & \large128 & \large$\times8$ \\
                                      \cmidrule{2-6}
                                      & \makecell{1566.72k\\\small(IPC=51)} & - & \large$10^{5}$ & \large256 & \large$\times4$ \\
            \midrule
            LSUN & \makecell{491.52k\\\small(IPC=1)} & \large40 & \large$10^{5}$ & \large256 & \large$\times64$\\
            \midrule
            \makecell{ImageNet-\\Subset\\($128\times128$)} & \makecell{491.52k\\\small(IPC=1)} & - & \large$10^{6}$ & \large2048 & \large$\times8$ \\
            \midrule
            \multirow{6}{*}{3D MNIST} & \makecell{40.96k\\\small(IPC=1)} & - & \large$10^{6}$ & \large512 & \large$\times8$ \\
                                      \cmidrule{2-6}
                                      & \makecell{409.6k\\\small(IPC=10)} & - & \large$10^{6}$ & \large1024 & \large$\times4$ \\
                                      \cmidrule{2-6}
                                      & \makecell{2048k\\\small(IPC=50)} & - & \large$10^{6}$ & \large1024 & \large$\times4$ \\
            \bottomrule \bottomrule
        \end{tabular}}
    \end{subtable}
    
    \begin{subtable}{\textwidth}
        \centering
        \caption{Trajectory matching (TM)}
        \adjustbox{max width=\textwidth}{%
        \begin{tabular}{c c cccccccc c}
            \toprule \toprule
            Dataset & \#Params & \makecell{Synthetic\\steps} & \makecell{Expert\\epochs} & \makecell{Max start\\epoch} & \makecell{Synthetic\\batch size} & \makecell{Learning rate\\\small(Frequency)} & \makecell{Learning rate\\\small(Step size)} & \makecell{Learning rate\\\small(Teacher)} & \makecell{Selected dimension\\per channel} & \makecell{Increment\\of instances} \\
            \midrule
            \multirow{3}{*}{MNIST} & \makecell{7.84k\\\small(IPC=1)}   & \large50 & \large2 & \large5  & - & \large$10^{6}$ & \large$10^{-7}$ & \large$10^{-2}$ & \large49 & \large$\times16$ \\
                                   \cmidrule{2-11}
                                   & \makecell{78.4k\\\small(IPC=10)} & \large30 & \large2 & \large15 & - & \large$10^{5}$ & \large$10^{-5}$ & \large$10^{-2}$ & \large392 & \large$\times2$ \\
            \midrule
            \multirow{3}{*}{\makecell{Fashion\\MNIST}} & \makecell{7.84k\\\small(IPC=1)} & \large50 & \large2 & \large5 & - & \large$10^{6}$ & \large$10^{-7}$ & \large$10^{-2}$ & \large49 & \large$\times16$ \\
                                                       \cmidrule{2-11}
                                                       & \makecell{78.4k\\\small(IPC=10)} & \large60 & \large2 & \large15 & - & \large$10^{5}$ & \large$10^{-5}$ & \large$10^{-2}$ & \large196 & \large$\times4$ \\
            \midrule
            \multirow{6}{*}{SVHN} & \makecell{30.72k\\\small(IPC=1)}    & \large50 & \large2 & \large5  & - & \large$10^{7}$ & \large$10^{-7}$ & \large$10^{-2}$ & \large64 & \large$\times16$ \\
                                  \cmidrule{2-11}
                                  & \makecell{307.2k\\\small(IPC=10)}  & \large30 & \large2 & \large15 & - & \large$10^{7}$ & \large$10^{-5}$ & \large$10^{-2}$ & \large128 & \large$\times8$ \\
                                  \cmidrule{2-11}
                                  & \makecell{1536k\\\small(IPC=50)} & \large40 & \large2 & \large40 & \large500 & \large$10^{7}$ & \large$10^{-5}$ & \large$10^{-3}$ & \large256 & \large$\times4$ \\
            \midrule
            \multirow{13}{*}{CIFAR-10} & \makecell{30.72k\\\small(IPC=1)}    & \large50 & \large2 & \large5 & - & \large$10^{8}$ & \large$10^{-7}$ & \large$10^{-2}$ & \large64 & \large$\times16$ \\
                                      \cmidrule{2-11}
                                      & \makecell{61.44k\\\small(IPC=2)}    & \large50 & \large2 & \large5 & \large160 & \large$10^{8}$ & \large$10^{-7}$ & \large$10^{-2}$ & \large64 & \large$\times16$ \\
                                      \cmidrule{2-11}
                                      & \makecell{307.2k\\\small(IPC=10)}  & \large40 & \large2 & \large15 & \large320 & \large$10^{7}$ & \large$10^{-5}$ & \large$10^{-2}$ & \large160 & \large$\times6.4$ \\
                                      \cmidrule{2-11}
                                      & \makecell{337.92k\\\small(IPC=11)}  & \large40 & \large2 & \large15 & \large320 & \large$10^{7}$ & \large$10^{-5}$ & \large$10^{-2}$ & \large176 & \large$\times5.82$ \\
                                      \cmidrule{2-11}
                                      & \makecell{1536k\\\small(IPC=50)} & \large30 & \large2 & \large40 & \large500 & \large$10^{7}$ & \large$10^{-5}$ & \large$10^{-3}$ & \large256 & \large$\times4$ \\
                                      \cmidrule{2-11}
                                      & \makecell{1566.72k\\\small(IPC=51)} & \large30 & \large2 & \large40 & \large510 & \large$10^{7}$ & \large$10^{-5}$ & \large$10^{-3}$ & \large256 & \large$\times4$ \\
            \midrule
            \multirow{6}{*}{CIFAR-100} & \makecell{30.72k\\\small(IPC=1)}    & \large50 & \large2 & \large15 & - & \large$10^{8}$ & \large$10^{-5}$ & \large$10^{-2}$ & \large128 & \large$\times8$ \\
                                       \cmidrule{2-11}
                                       & \makecell{307.2k\\\small(IPC=10)}  & \large20 & \large2 & \large40 & \large2048 & \large$5\times10^{6}$ & \large$10^{-5}$ & \large$10^{-2}$ & \large400 & \large$\times2.56$ \\
                                       \cmidrule{2-11}
                                       & \makecell{1536k\\\small(IPC=50)} & \large80 & \large2 & \large40 & \large256 & \large$5\times10^{6}$ & \large$10^{-5}$ & \large$10^{-2}$ & \large400 & \large$\times2.56$ \\
            \midrule
            \multirow{6}{*}{\makecell{Tiny-\\ImageNet}} & \makecell{2457.6k\\\small(IPC=1)}    & \large30 & \large2 & \large30 & \large400 & \large$10^{9}$ & \large$10^{-4}$ & \large$10^{-2}$ & \large512 & \large$\times8$ \\
                                                        \cmidrule{2-11}
                                                        & \makecell{24576k\\\small(IPC=10)} & \large40 & \large2 & \large40 & \large300 & \large$10^{9}$ & \large$10^{-4}$ & \large$10^{-2}$ & \large3840 & \large$\times3.2$ \\
                                                        \cmidrule{2-11}
                                                        & \makecell{122880k\\\small(IPC=50)} & \large40 & \large2 & \large40 & \large250 & \large$10^{9}$ & \large$10^{-4}$ & \large$10^{-2}$ & \large3840 & \large$\times3.2$ \\
            \midrule
            \multirow{6}{*}{\makecell{ImageNet-\\Subset\\($128\times128$)}} & \makecell{491.52k\\\small(IPC=1)} & \large20 & \large2 & \large10 & - & \large$10^{9}$ & \large$10^{-6}$ & \large$10^{-2}$ & \large2048 & \large$\times8$ \\
                                                                 \cmidrule{2-11}
                                                                 & \makecell{983.04k\\\small(IPC=2)}   & \large20 & \large2 & \large10 & \large80 & \large$10^{9}$ & \large$10^{-6}$ & \large$10^{-2}$ & \large2048 & \large$\times8$ \\
                                                                 \cmidrule{2-11}
                                                                 & \makecell{4915.2k\\\small(IPC=10)} & \large20 & \large2 & \large10 & \large80 & \large$10^{9}$ & \large$10^{-6}$ & \large$10^{-2}$ & \large4096 & \large$\times4$ \\
            \bottomrule \bottomrule
          \end{tabular}}
    \end{subtable}
\end{table}

\clearpage
\begin{figure}[t]
    \centering
    \subcaptionbox{Original}{\includegraphics[width=.24\textwidth]{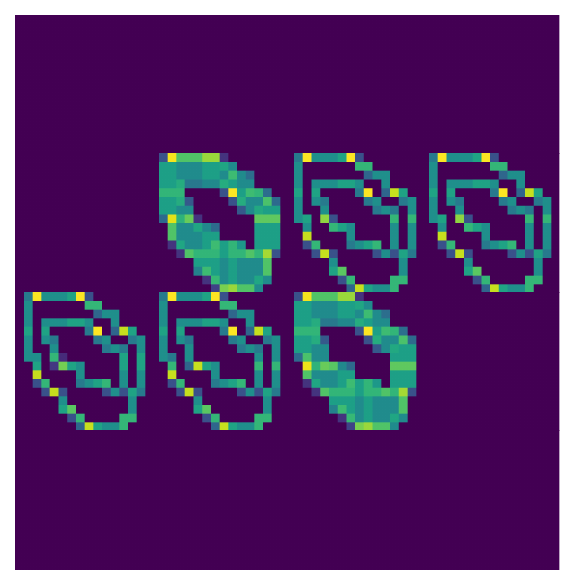}}
    \subcaptionbox{DM}{\includegraphics[width=.24\textwidth]{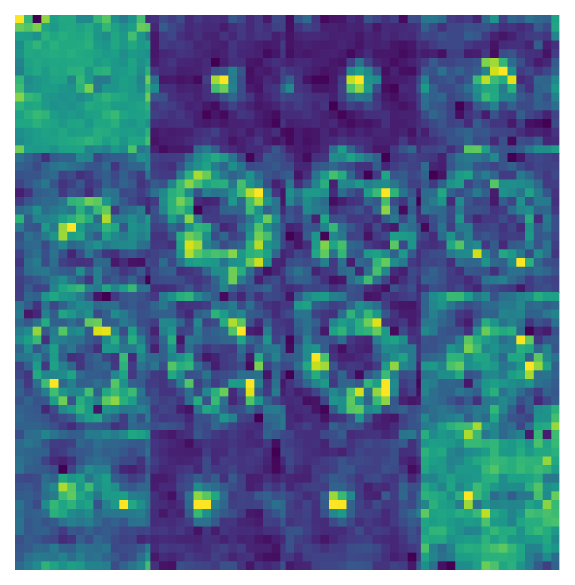}}
    \subcaptionbox{IDC}{\includegraphics[width=.24\textwidth]{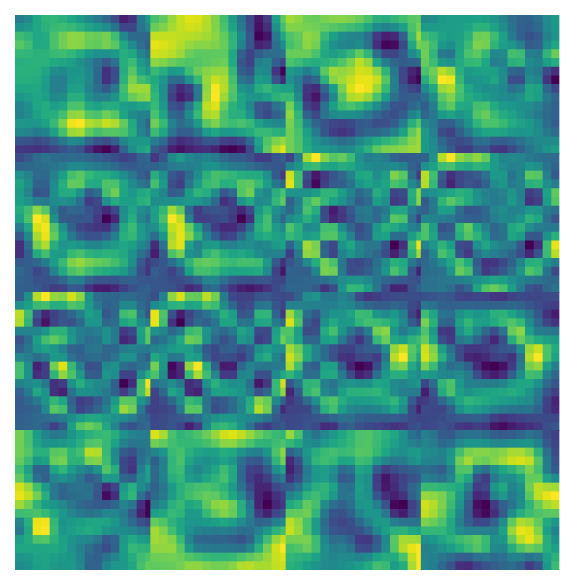}}
    \subcaptionbox{FreD}{\includegraphics[width=.24\textwidth]{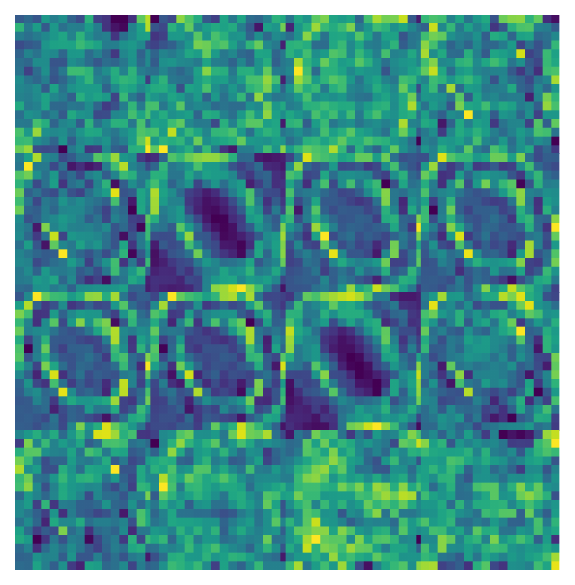}}
    \caption{The cross-section visualizations of class 0 in 3D MNIST. Each top left image represents the frontmost view, while bottom right image corresponds to the rearmost view.}
    \label{appendix:fig:3D MNIST}
\end{figure}

\begin{figure*}[t]
    \centering
    \includegraphics[width=0.9\textwidth]{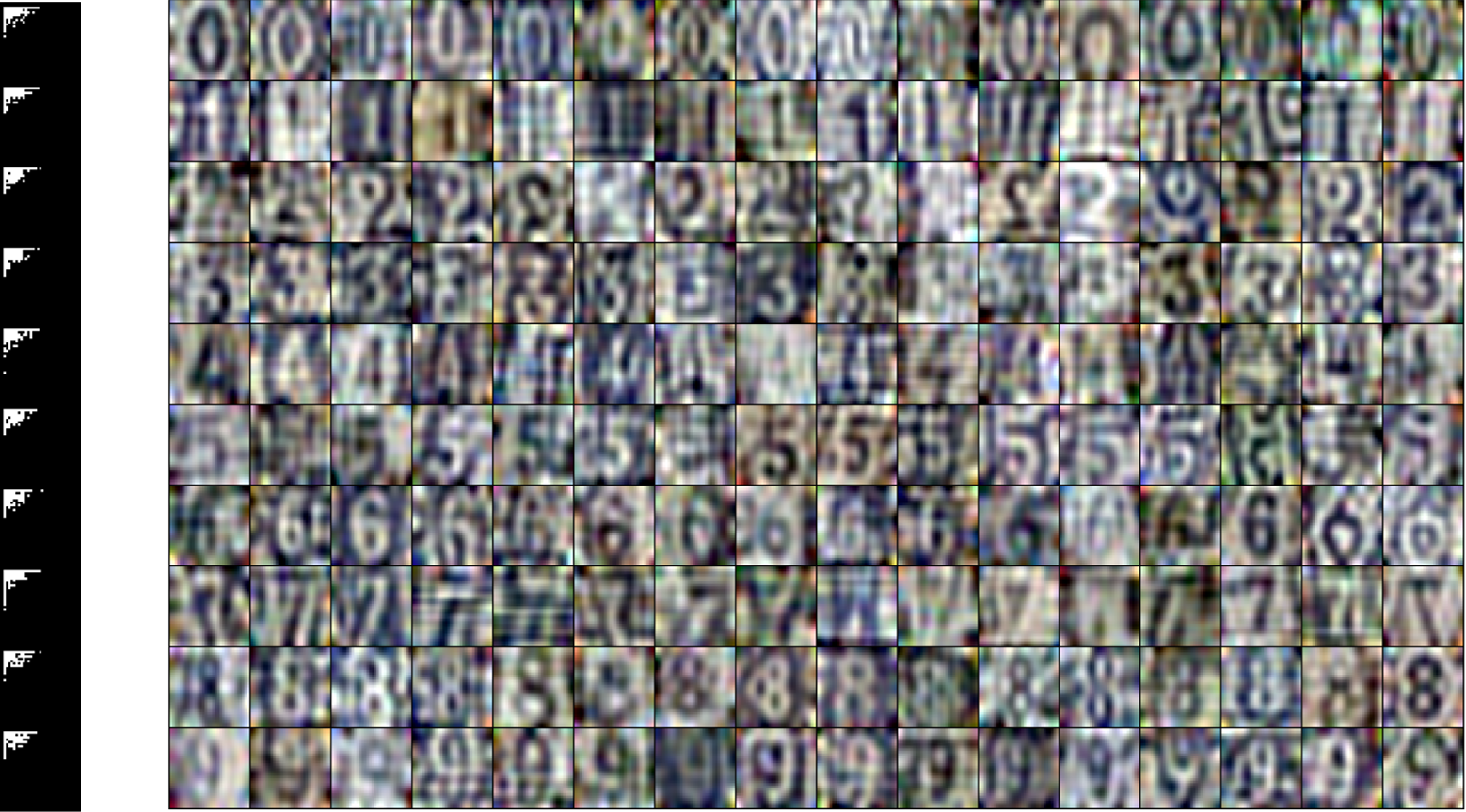}
    \caption{Visualization of the binary mask and the transformed images by FreD on SVHN with IPC=1 (\#Params=30.72k). In this setting, FreD constructs 16 images per class under the same budget.} \label{appendix:fig:qual_svhn}
\end{figure*}

\begin{figure*}[t]
    \centering
    \includegraphics[width=0.9\textwidth]{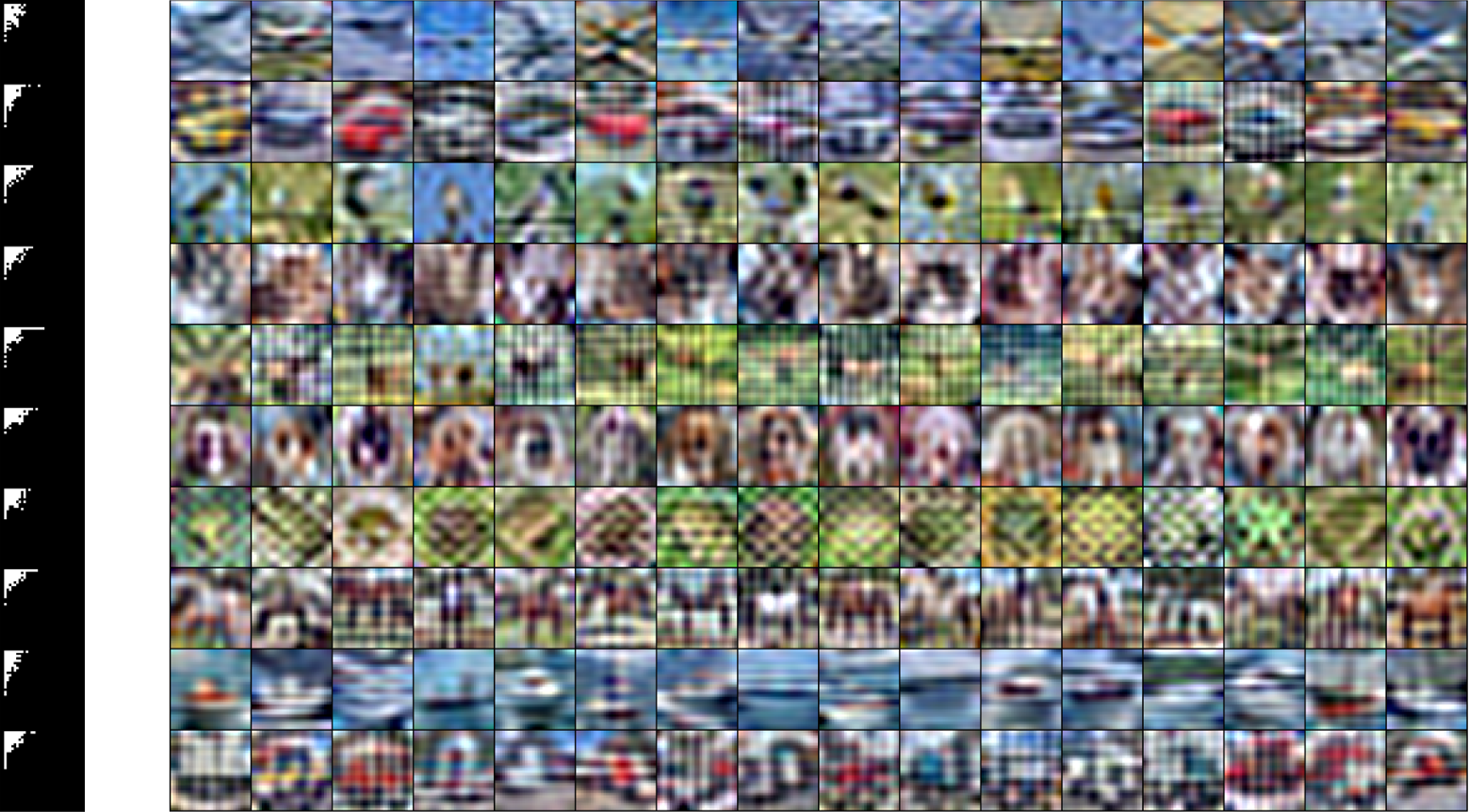}
    \caption{Visualization of the binary mask and the transformed images by FreD on CIFAR-10 with IPC=1 (\#Params=30.72k). In this setting, FreD constructs 16 images per class under the same budget.} \label{appendix:fig:qual_cifar10}
\end{figure*}

\begin{figure*}[h]
\centering
    \begin{subfigure}{0.48\textwidth}
        \centering
        \centerline{\includegraphics[width=\textwidth]{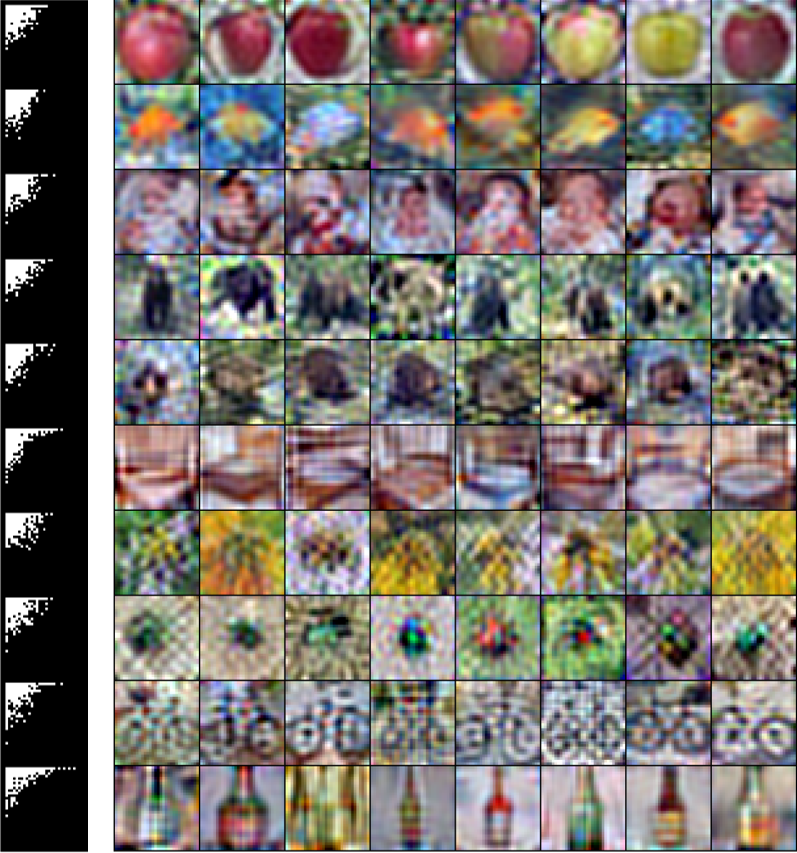}}
        \caption{CIFAR-100} \label{appendix:fig:qual_cifar100}
    \end{subfigure}
    \hfill
    \begin{subfigure}{0.48\textwidth}
        \centering
        \centerline{\includegraphics[width=\textwidth]{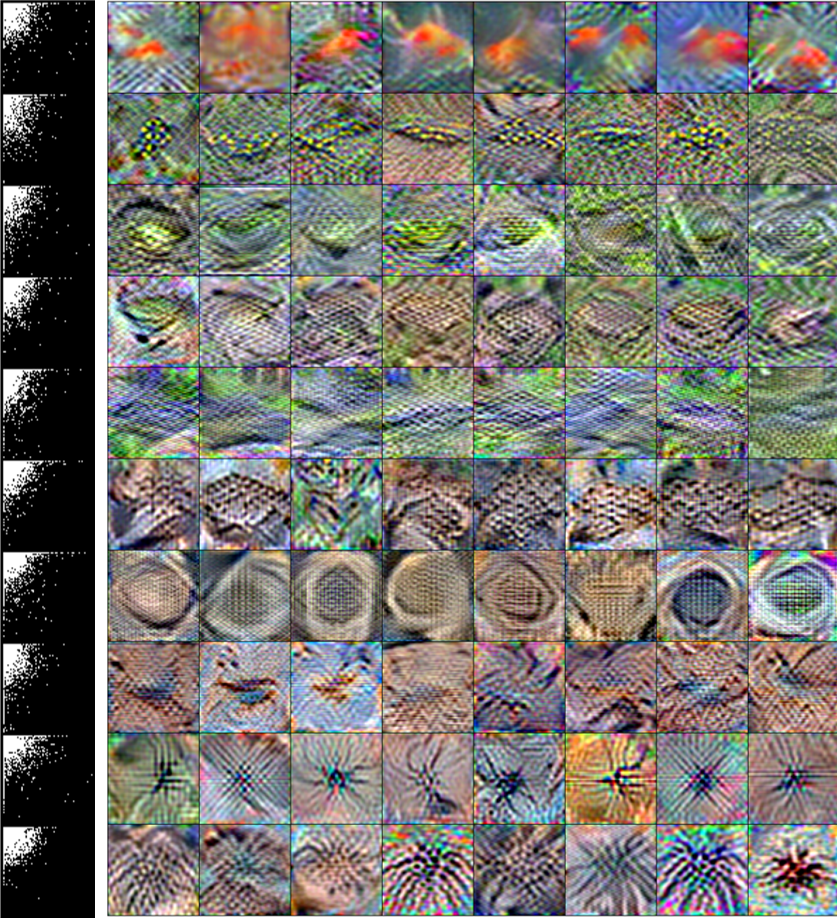}}
        \caption{Tiny-ImageNet} \label{appendix:fig:qual_tiny}
    \end{subfigure}
\caption{Visualization of the binary mask and the transformed images by FreD on CIFAR-100 with IPC=1 (\#Params=30.72k) and Tiny-ImageNet with IPC=1 (\#Params=2457.6k). Due to a lack of space, only the first 10 classes were visualized. In both cases, FreD constructs 8 images per class under the same budget.}
\end{figure*}

\begin{figure*}[h]
\centering
    \begin{subfigure}{0.48\textwidth}
        \centering
        \centerline{\includegraphics[width=\textwidth]{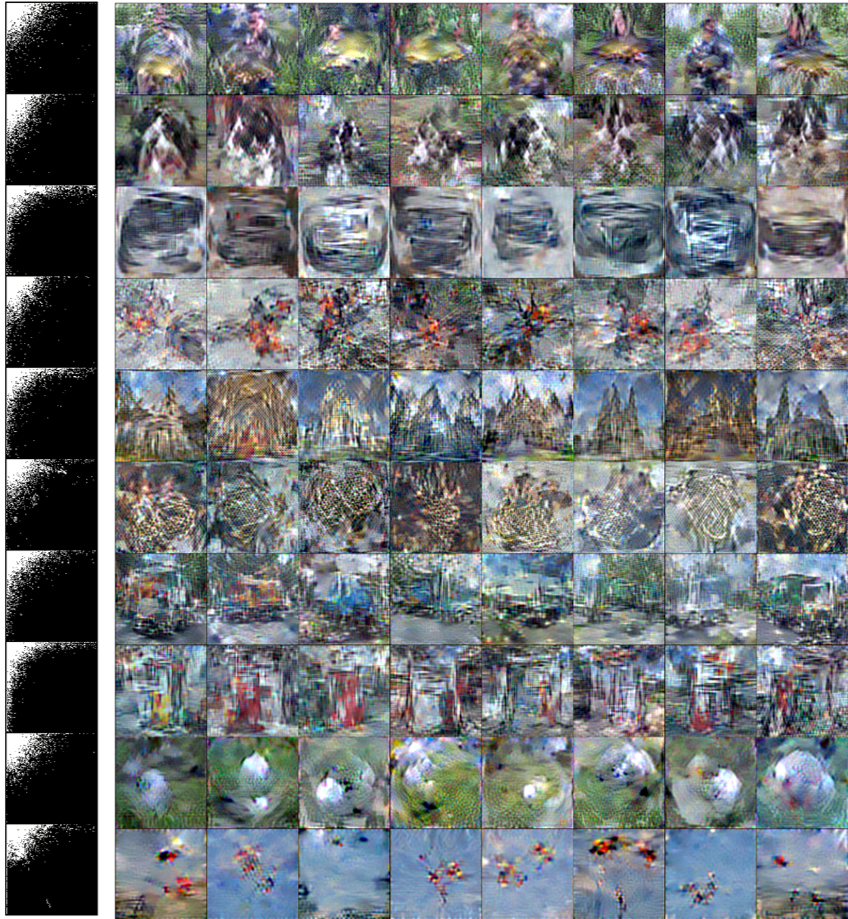}}
        \caption{ImageNette} \label{appendix:fig:qual_imagenette}
    \end{subfigure}
    \hfill
    \begin{subfigure}{0.48\textwidth}
        \centering
        \centerline{\includegraphics[width=\textwidth]{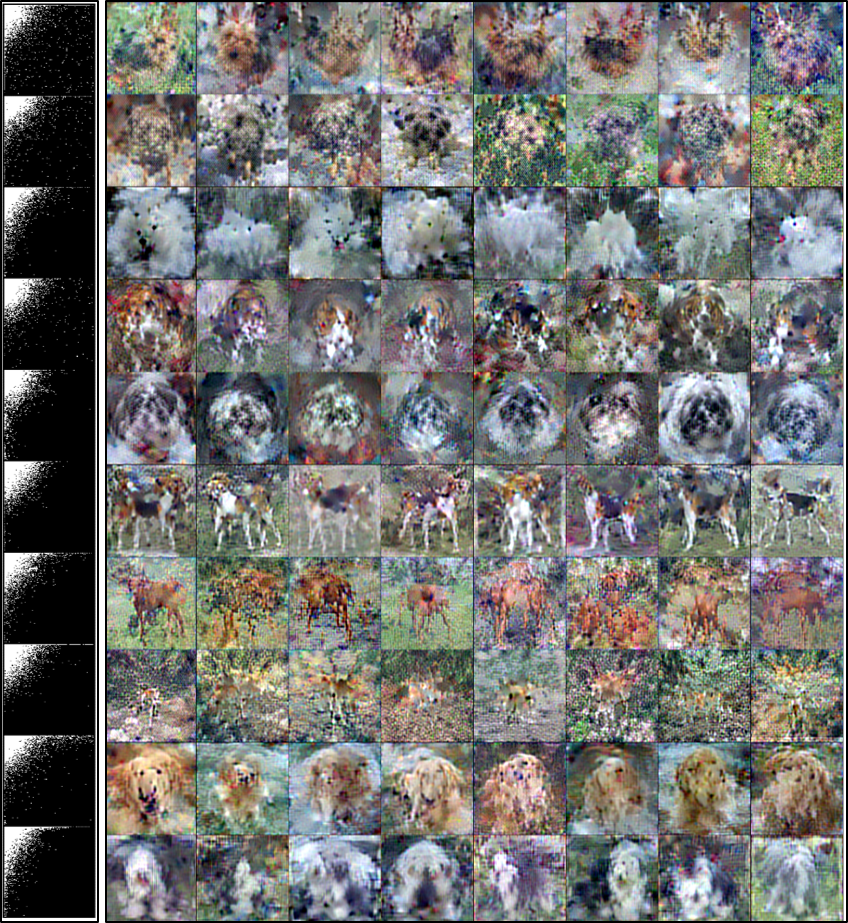}}
        \caption{ImageWoof} \label{appendix:fig:qual_imagewoof}
    \end{subfigure}
\caption{Visualization of the binary mask and the transformed images by FreD on ImageNet-Subset with IPC=1 (\#Params=491.52k). In these cases, FreD constructs 8 images per class under the same budget.}
\label{appendix:fig:qual_imagenet1}
\end{figure*}

\begin{figure*}[h]
\centering
    \begin{subfigure}{0.48\textwidth}
        \centering
        \centerline{\includegraphics[width=\textwidth]{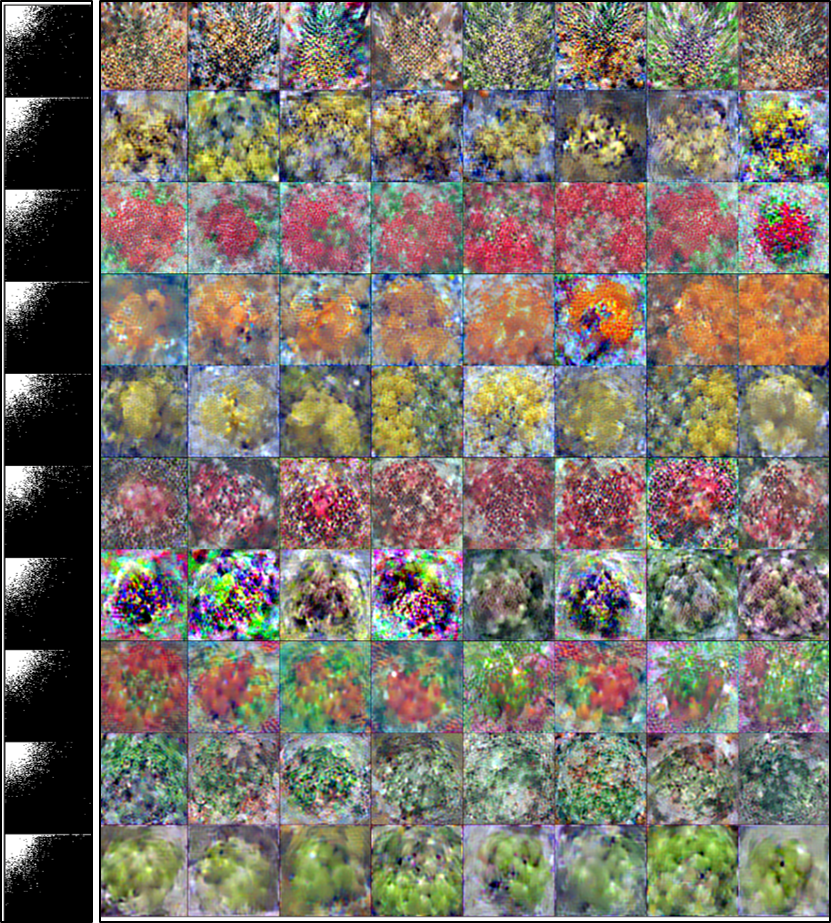}}
        \caption{ImageFruit} \label{appendix:fig:qual_imagefruit}
    \end{subfigure}
    \hfill
    \begin{subfigure}{0.48\textwidth}
        \centering
        \centerline{\includegraphics[width=\textwidth]{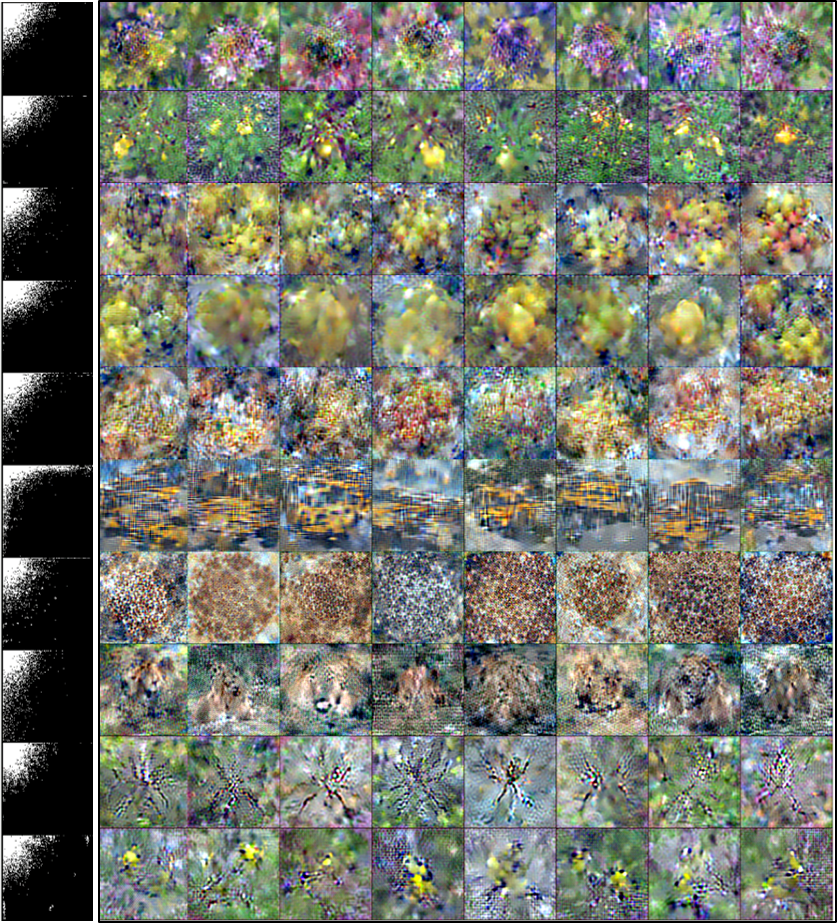}}
        \caption{ImageYellow} \label{appendix:fig:qual_imageyellow}
    \end{subfigure}
\caption{Visualization of the binary mask and the transformed images by FreD on ImageNet-Subset with IPC=1 (\#Params=491.52k). In these cases, FreD constructs 8 images per class under the same budget.}
\label{appendix:fig:qual_imagenet2}
\end{figure*}

\begin{figure*}[h]
\centering
    \begin{subfigure}{0.48\textwidth}
        \centering
        \centerline{\includegraphics[width=\textwidth]{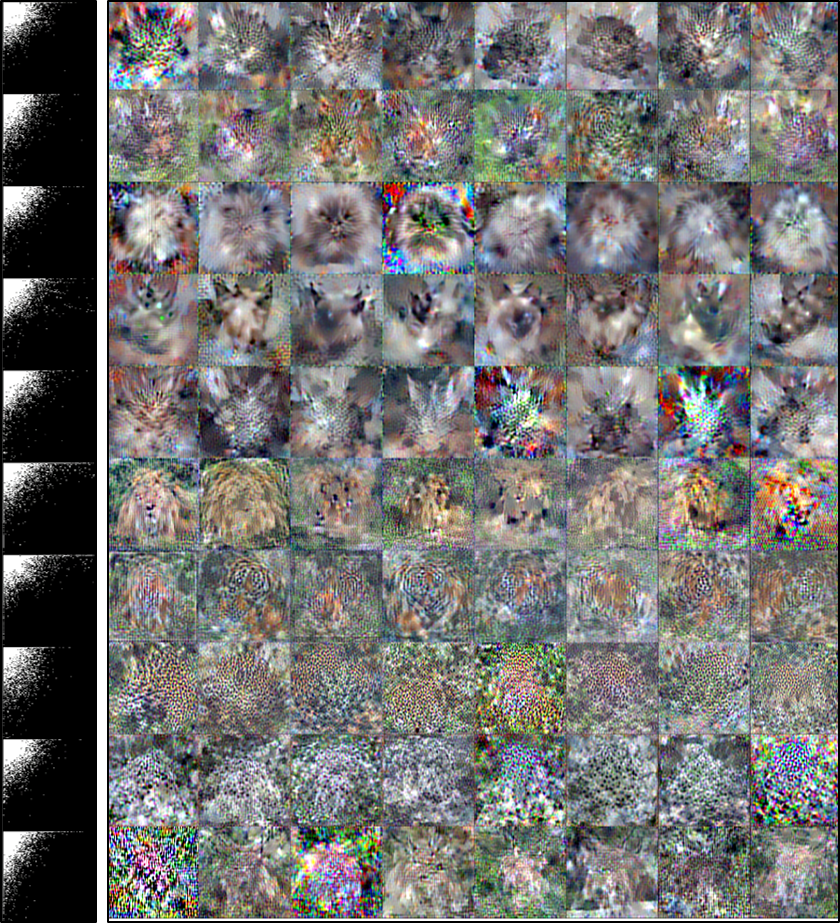}}
        \caption{ImageMeow} \label{appendix:fig:qual_imagemeow}
    \end{subfigure}
    \hfill
    \begin{subfigure}{0.48\textwidth}
        \centering
        \centerline{\includegraphics[width=\textwidth]{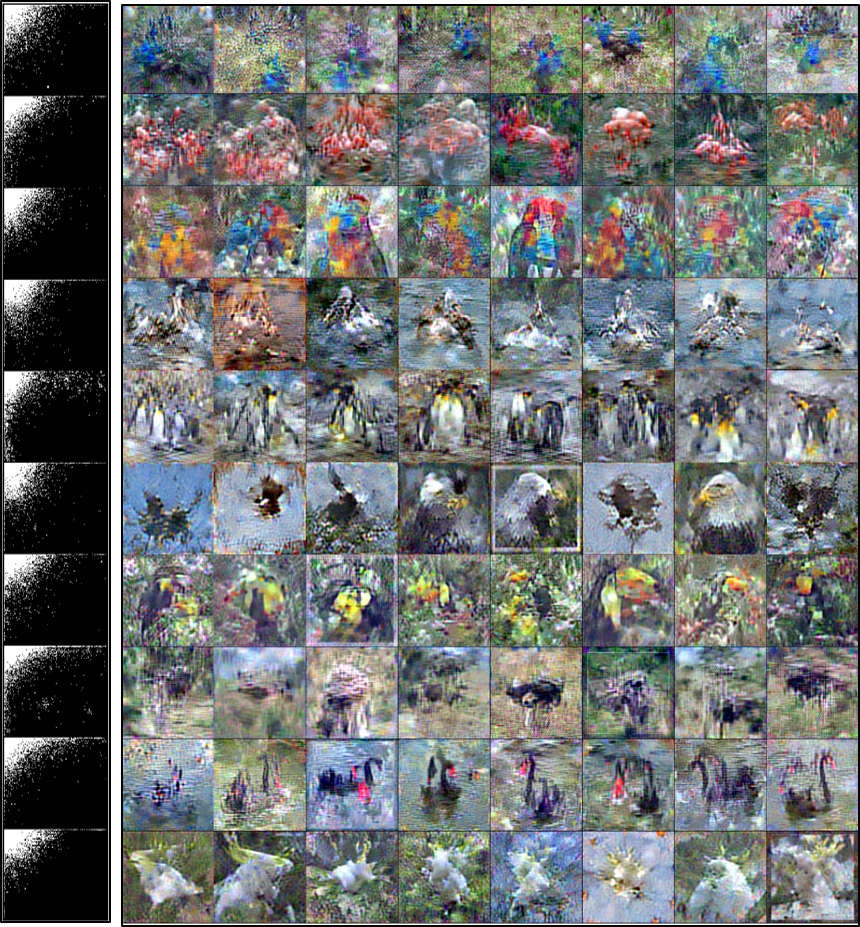}}
        \caption{ImageSquawk} \label{appendix:fig:qual_imagesquawk}
    \end{subfigure}
\caption{Visualization of the binary mask and the transformed images by FreD on ImageNet-Subset with IPC=1 (\#Params=491.52k). In these cases, FreD constructs 8 images per class under the same budget.}
\label{appendix:fig:qual_imagenet3}
\end{figure*}

\end{appendices}


\begin{thebibliography}{10}

\bibitem{ahmed1974discrete}
Nasir Ahmed, T\_ Natarajan, and Kamisetty~R Rao.
\newblock Discrete cosine transform.
\newblock {\em IEEE transactions on Computers}, 100(1):90--93, 1974.

\bibitem{bengio2007scaling}
Yoshua Bengio, Yann LeCun, et~al.
\newblock Scaling learning algorithms towards ai.
\newblock {\em Large-scale kernel machines}, 34(5):1--41, 2007.

\bibitem{spectralbias2}
Yuan Cao, Zhiying Fang, Yue Wu, Ding-Xuan Zhou, and Quanquan Gu.
\newblock Towards understanding the spectral bias of deep learning.
\newblock {\em arXiv preprint arXiv:1912.01198}, 2019.

\bibitem{cazenavette2022dataset}
George Cazenavette, Tongzhou Wang, Antonio Torralba, Alexei~A Efros, and Jun-Yan Zhu.
\newblock Dataset distillation by matching training trajectories.
\newblock In {\em Proceedings of the IEEE/CVF Conference on Computer Vision and Pattern Recognition}, pages 4750--4759, 2022.

\bibitem{cazenavette2023generalizing}
George Cazenavette, Tongzhou Wang, Antonio Torralba, Alexei~A Efros, and Jun-Yan Zhu.
\newblock Generalizing dataset distillation via deep generative prior.
\newblock In {\em Proceedings of the IEEE/CVF Conference on Computer Vision and Pattern Recognition}, pages 3739--3748, 2023.

\bibitem{deng2009imagenet}
Jia Deng, Wei Dong, Richard Socher, Li-Jia Li, Kai Li, and Li~Fei-Fei.
\newblock Imagenet: A large-scale hierarchical image database.
\newblock In {\em 2009 IEEE conference on computer vision and pattern recognition}, pages 248--255. Ieee, 2009.

\bibitem{deng2022remember}
Zhiwei Deng and Olga Russakovsky.
\newblock Remember the past: Distilling datasets into addressable memories for neural networks.
\newblock {\em Advances in Neural Information Processing Systems}, 35:34391--34404, 2022.

\bibitem{dosovitskiy2010image}
Alexey Dosovitskiy, Lucas Beyer, Alexander Kolesnikov, Dirk Weissenborn, Xiaohua Zhai, Thomas Unterthiner, Mostafa Dehghani, Matthias Minderer, Georg Heigold, Sylvain Gelly, et~al.
\newblock An image is worth 16x16 words: Transformers for image recognition at scale. arxiv 2020.
\newblock {\em arXiv preprint arXiv:2010.11929}, 2010.

\bibitem{durall2020watch}
Ricard Durall, Margret Keuper, and Janis Keuper.
\newblock Watch your up-convolution: Cnn based generative deep neural networks are failing to reproduce spectral distributions.
\newblock In {\em Proceedings of the IEEE/CVF conference on computer vision and pattern recognition}, pages 7890--7899, 2020.

\bibitem{guo2018low}
Chuan Guo, Jared~S Frank, and Kilian~Q Weinberger.
\newblock Low frequency adversarial perturbation.
\newblock {\em arXiv preprint arXiv:1809.08758}, 2018.

\bibitem{he2016deep}
Kaiming He, Xiangyu Zhang, Shaoqing Ren, and Jian Sun.
\newblock Deep residual learning for image recognition.
\newblock In {\em Proceedings of the IEEE conference on computer vision and pattern recognition}, pages 770--778, 2016.

\bibitem{hendrycks2019benchmarking}
Dan Hendrycks and Thomas Dietterich.
\newblock Benchmarking neural network robustness to common corruptions and perturbations.
\newblock {\em arXiv preprint arXiv:1903.12261}, 2019.

\bibitem{howard2019smaller}
Jeremy Howard.
\newblock A smaller subset of 10 easily classified classes from imagenet and a little more french.
\newblock {\em URL https://github. com/fastai/imagenette}, 2019.

\bibitem{frequencylearning2}
Liming Jiang, Bo~Dai, Wayne Wu, and Chen~Change Loy.
\newblock Focal frequency loss for image reconstruction and synthesis.
\newblock In {\em Proceedings of the IEEE/CVF International Conference on Computer Vision}, pages 13919--13929, 2021.

\bibitem{kim2022dataset}
Jang-Hyun Kim, Jinuk Kim, Seong~Joon Oh, Sangdoo Yun, Hwanjun Song, Joonhyun Jeong, Jung-Woo Ha, and Hyun~Oh Song.
\newblock Dataset condensation via efficient synthetic-data parameterization.
\newblock In {\em International Conference on Machine Learning}, pages 11102--11118. PMLR, 2022.

\bibitem{krizhevsky2009learning}
Alex Krizhevsky, Geoffrey Hinton, et~al.
\newblock Learning multiple layers of features from tiny images.
\newblock 2009.

\bibitem{krizhevsky2017imagenet}
Alex Krizhevsky, Ilya Sutskever, and Geoffrey~E Hinton.
\newblock Imagenet classification with deep convolutional neural networks.
\newblock {\em Communications of the ACM}, 60(6):84--90, 2017.

\bibitem{le2015tiny}
Ya~Le and Xuan Yang.
\newblock Tiny imagenet visual recognition challenge.
\newblock {\em CS 231N}, 7(7):3, 2015.

\bibitem{lecun1998gradient}
Yann LeCun, L{\'e}on Bottou, Yoshua Bengio, and Patrick Haffner.
\newblock Gradient-based learning applied to document recognition.
\newblock {\em Proceedings of the IEEE}, 86(11):2278--2324, 1998.

\bibitem{liu2022dataset}
Songhua Liu, Kai Wang, Xingyi Yang, Jingwen Ye, and Xinchao Wang.
\newblock Dataset distillation via factorization.
\newblock {\em arXiv preprint arXiv:2210.16774}, 2022.

\bibitem{long2022frequency}
Yuyang Long, Qilong Zhang, Boheng Zeng, Lianli Gao, Xianglong Liu, Jian Zhang, and Jingkuan Song.
\newblock Frequency domain model augmentation for adversarial attack.
\newblock In {\em European Conference on Computer Vision}, pages 549--566. Springer, 2022.

\bibitem{frequencyaugment1}
Yuyang Long, Qilong Zhang, Boheng Zeng, Lianli Gao, Xianglong Liu, Jian Zhang, and Jingkuan Song.
\newblock Frequency domain model augmentation for adversarial attack.
\newblock In {\em Computer Vision--ECCV 2022: 17th European Conference, Tel Aviv, Israel, October 23--27, 2022, Proceedings, Part IV}, pages 549--566. Springer, 2022.

\bibitem{mathieu2013fast}
Michael Mathieu, Mikael Henaff, and Yann LeCun.
\newblock Fast training of convolutional networks through ffts.
\newblock {\em arXiv preprint arXiv:1312.5851}, 2013.

\bibitem{mohan2018mri}
Geethu Mohan and M~Monica Subashini.
\newblock Mri based medical image analysis: Survey on brain tumor grade classification.
\newblock {\em Biomedical Signal Processing and Control}, 39:139--161, 2018.

\bibitem{nagai2018digital}
Yuki Nagai, Yusuke Uchida, Shigeyuki Sakazawa, and Shin’ichi Satoh.
\newblock Digital watermarking for deep neural networks.
\newblock {\em International Journal of Multimedia Information Retrieval}, 7:3--16, 2018.

\bibitem{netzer2011reading}
Yuval Netzer, Tao Wang, Adam Coates, Alessandro Bissacco, Bo~Wu, and Andrew~Y Ng.
\newblock Reading digits in natural images with unsupervised feature learning.
\newblock 2011.

\bibitem{nguyen2021dataset}
Timothy Nguyen, Roman Novak, Lechao Xiao, and Jaehoon Lee.
\newblock Dataset distillation with infinitely wide convolutional networks.
\newblock {\em Advances in Neural Information Processing Systems}, 34:5186--5198, 2021.

\bibitem{pratt2017fcnn}
Harry Pratt, Bryan Williams, Frans Coenen, and Yalin Zheng.
\newblock Fcnn: Fourier convolutional neural networks.
\newblock In {\em Machine Learning and Knowledge Discovery in Databases: European Conference, ECML PKDD 2017, Skopje, Macedonia, September 18--22, 2017, Proceedings, Part I 17}, pages 786--798. Springer, 2017.

\bibitem{spectralbias}
Nasim Rahaman, Aristide Baratin, Devansh Arpit, Felix Draxler, Min Lin, Fred Hamprecht, Yoshua Bengio, and Aaron Courville.
\newblock On the spectral bias of neural networks.
\newblock In {\em International Conference on Machine Learning}, pages 5301--5310. PMLR, 2019.

\bibitem{rahaman2019spectral}
Nasim Rahaman, Aristide Baratin, Devansh Arpit, Felix Draxler, Min Lin, Fred Hamprecht, Yoshua Bengio, and Aaron Courville.
\newblock On the spectral bias of neural networks.
\newblock In {\em International Conference on Machine Learning}, pages 5301--5310. PMLR, 2019.

\bibitem{rebaza2021first}
Jorge Rebaza.
\newblock {\em A first course in applied mathematics}.
\newblock John Wiley \& Sons, 2021.

\bibitem{recht2018cifar}
Benjamin Recht, Rebecca Roelofs, Ludwig Schmidt, and Vaishaal Shankar.
\newblock Do cifar-10 classifiers generalize to cifar-10?
\newblock {\em arXiv preprint arXiv:1806.00451}, 2018.

\bibitem{sachdeva2023data}
Noveen Sachdeva and Julian McAuley.
\newblock Data distillation: A survey.
\newblock {\em arXiv preprint arXiv:2301.04272}, 2023.

\bibitem{shao2023detecting}
Rui Shao, Tianxing Wu, and Ziwei Liu.
\newblock Detecting and grounding multi-modal media manipulation.
\newblock In {\em Proceedings of the IEEE/CVF Conference on Computer Vision and Pattern Recognition}, pages 6904--6913, 2023.

\bibitem{sharma2019effectiveness}
Yash Sharma, Gavin~Weiguang Ding, and Marcus Brubaker.
\newblock On the effectiveness of low frequency perturbations.
\newblock {\em arXiv preprint arXiv:1903.00073}, 2019.

\bibitem{simonyan2014very}
Karen Simonyan and Andrew Zisserman.
\newblock Very deep convolutional networks for large-scale image recognition.
\newblock {\em arXiv preprint arXiv:1409.1556}, 2014.

\bibitem{ulyanov2016instance}
Dmitry Ulyanov, Andrea Vedaldi, and Victor Lempitsky.
\newblock Instance normalization: The missing ingredient for fast stylization.
\newblock {\em arXiv preprint arXiv:1607.08022}, 2016.

\bibitem{wang2022cafe}
Kai Wang, Bo~Zhao, Xiangyu Peng, Zheng Zhu, Shuo Yang, Shuo Wang, Guan Huang, Hakan Bilen, Xinchao Wang, and Yang You.
\newblock Cafe: Learning to condense dataset by aligning features.
\newblock In {\em Proceedings of the IEEE/CVF Conference on Computer Vision and Pattern Recognition}, pages 12196--12205, 2022.

\bibitem{wang2018dataset}
Tongzhou Wang, Jun-Yan Zhu, Antonio Torralba, and Alexei~A Efros.
\newblock Dataset distillation.
\newblock {\em arXiv preprint arXiv:1811.10959}, 2018.

\bibitem{spca}
Wen-Ting Wang and Hsin-Cheng Huang.
\newblock Regularized principal component analysis for spatial data.
\newblock {\em Journal of Computational and Graphical Statistics}, 26(1):14--25, 2017.

\bibitem{wang2000energy}
Ye~Wang, Miikka Vilermo, and Leonid Yaroslavsky.
\newblock Energy compaction property of the mdct in comparison with other transforms.
\newblock In {\em Audio Engineering Society Convention 109}. Audio Engineering Society, 2000.

\bibitem{welling2009herding}
Max Welling.
\newblock Herding dynamical weights to learn.
\newblock In {\em Proceedings of the 26th Annual International Conference on Machine Learning}, pages 1121--1128, 2009.

\bibitem{xiao2017/online}
Han Xiao, Kashif Rasul, and Roland Vollgraf.
\newblock Fashion-mnist: a novel image dataset for benchmarking machine learning algorithms, 2017.

\bibitem{frequencylearning1}
Kai Xu, Minghai Qin, Fei Sun, Yuhao Wang, Yen-Kuang Chen, and Fengbo Ren.
\newblock Learning in the frequency domain.
\newblock In {\em Proceedings of the IEEE/CVF Conference on Computer Vision and Pattern Recognition}, pages 1740--1749, 2020.

\bibitem{spectralbias3}
Zhi-Qin~John Xu, Yaoyu Zhang, and Tao Luo.
\newblock Overview frequency principle/spectral bias in deep learning.
\newblock {\em arXiv preprint arXiv:2201.07395}, 2022.

\bibitem{xu2022overview}
Zhi-Qin~John Xu, Yaoyu Zhang, and Tao Luo.
\newblock Overview frequency principle/spectral bias in deep learning.
\newblock {\em arXiv preprint arXiv:2201.07395}, 2022.

\bibitem{frequencyinput2}
Aitao Yang, Min Li, Zhaoqing Wu, Yujie He, Xiaohua Qiu, Yu~Song, Weidong Du, and Yao Gou.
\newblock Cdf-net: A convolutional neural network fusing frequency domain and spatial domain features.
\newblock {\em IET Computer Vision}, 17(3):319--329, 2023.

\bibitem{yang2020patchattack}
Chenglin Yang, Adam Kortylewski, Cihang Xie, Yinzhi Cao, and Alan Yuille.
\newblock Patchattack: A black-box texture-based attack with reinforcement learning.
\newblock In {\em European Conference on Computer Vision}, pages 681--698. Springer, 2020.

\bibitem{yaroslavsky2015compression}
Leonid~P Yaroslavsky.
\newblock Compression, restoration, resampling,‘compressive sensing’: fast transforms in digital imaging.
\newblock {\em Journal of Optics}, 17(7):073001, 2015.

\bibitem{yu2015lsun}
Fisher Yu, Ari Seff, Yinda Zhang, Shuran Song, Thomas Funkhouser, and Jianxiong Xiao.
\newblock Lsun: Construction of a large-scale image dataset using deep learning with humans in the loop.
\newblock {\em arXiv preprint arXiv:1506.03365}, 2015.

\bibitem{frequencyinput1}
Jingyi Zhang, Jiaxing Huang, Zichen Tian, and Shijian Lu.
\newblock Spectral unsupervised domain adaptation for visual recognition.
\newblock In {\em Proceedings of the IEEE/CVF Conference on Computer Vision and Pattern Recognition}, pages 9829--9840, 2022.

\bibitem{zhao2021dataset}
Bo~Zhao and Hakan Bilen.
\newblock Dataset condensation with differentiable siamese augmentation.
\newblock In {\em International Conference on Machine Learning}, pages 12674--12685. PMLR, 2021.

\bibitem{zhao2023dataset}
Bo~Zhao and Hakan Bilen.
\newblock Dataset condensation with distribution matching.
\newblock In {\em Proceedings of the IEEE/CVF Winter Conference on Applications of Computer Vision}, pages 6514--6523, 2023.

\bibitem{zhao2020dataset}
Bo~Zhao, Konda~Reddy Mopuri, and Hakan Bilen.
\newblock Dataset condensation with gradient matching.
\newblock {\em arXiv preprint arXiv:2006.05929}, 2020.

\bibitem{zhou2022dataset}
Yongchao Zhou, Ehsan Nezhadarya, and Jimmy Ba.
\newblock Dataset distillation using neural feature regression.
\newblock {\em arXiv preprint arXiv:2206.00719}, 2022.

\end{thebibliography}
\end{document}